\documentclass{article}
\usepackage[utf8]{inputenc}
\usepackage{geometry}
\newgeometry{vmargin={1in}, hmargin={1in,1in}} 
\usepackage{amsthm}
\usepackage{amssymb}
\usepackage{amsmath}
\usepackage{amsfonts}
\usepackage{bbm}
\usepackage{color}
\usepackage[colorlinks]{hyperref}
\hypersetup{
	colorlinks=true,
    linkcolor=red,
    citecolor=OliveGreen,
    filecolor=black,
    urlcolor=black,
}
\usepackage[dvipsnames]{xcolor}
\usepackage{url}
\usepackage[T1]{fontenc}
\usepackage{graphicx}
\usepackage{enumitem}
\usepackage{subfigure}
\usepackage{wrapfig}
\usepackage{multirow}
\usepackage{svg}
\usepackage{algorithm}
\usepackage{algpseudocode}
\usepackage{soul}
\usepackage[section]{placeins}
\newcommand{\XX}{\mathcal{X}}
\newcommand{\FF}{\mathcal{F}_{\mu, L, R}}
\newcommand{\WW}{B(0,R)}

 \newcommand{\FFE}{\mathcal{F}_{\mu, L, R}^{ERM}}

\newcommand{\rank}{\mbox{rank}}

\newcommand{\ws}{w^{*}}

\newcommand{\Fl}{F_{\lambda}}
\newcommand{\zl}{z_{\lambda}}
\newcommand{\ld}{\cd + \sqrt{\cd^2 + \varepsilon}}
\newcommand{\GG}{\mathcal{G}_{L, R}}
\newcommand{\GGE}{\mathcal{G}_{L, R}^{ERM}}
\newcommand{\pe}{\frac{d}{\varepsilon}}
\newcommand{\ep}{\frac{\varepsilon}{d}}
\newcommand{\pen}{\frac{d}{\varepsilon n}}
\newcommand{\cd}{c_{\delta}}
\newcommand{\lpen}{\frac{\sqrt{d}\left(\ld\right)}{\varepsilon n}}
\newcommand{\lpe}{\frac{\sqrt{d}\left(\ld\right)}{\varepsilon}}
\newcommand{\sig}{\left(\frac{\cd + \sqrt{\cd^2 + \varepsilon}}{\sqrt{2} \varepsilon}\right) \Delta_{F}}

\newcommand{\wls}{\ws_{\lambda}}
\newcommand{\HH}{\mathcal{H}_{\beta, \mu, L, R}}
\newcommand{\HHE}{\HH^{ERM}}
\newcommand{\JJ} {\mathcal{J}_{\beta, L, R}}
\newcommand{\JJE}{\mathcal{J}_{\beta, L, R}^{ERM}}

\newcommand{\MM}{\mathcal{M}}
\newcommand{\lmu}{\frac{L^2}{\mu}}
\newcommand{\Dl}{\Delta_{\lambda}}
\newcommand{\nep}{\frac{\varepsilon n}{d}}
\newcommand{\lnep}{\frac{\varepsilon n}{\sqrt{d} (\cd + \sqrt{\cd^2 + \varepsilon})}}
\newcommand{\lep}{\frac{\varepsilon}{\sqrt{d} (\cd + \sqrt{\cd^2 + \varepsilon})}}
\newcommand{\lrb}{\frac{L}{R \beta}}
\newcommand{\blrt}{\beta^{1/3} L^{2/3} R^{4/3}}
\newcommand{\ftau}{F_{\tau}}
\newcommand{\hf}{\widehat{F}}
\newcommand{\DD}{\mathcal{D}}
\newcommand{\wpr}{w_{\mathcal{A}}}
\newcommand{\wprl}{w_{\mathcal{A}_{\lambda}}}
\newcommand{\wprpro}{w_{\mathcal{A'}}}
\newcommand{\wla}{w_{\mathcal{A}, \lambda}}
\newcommand{\wlapro}{w_{\mathcal{A'}, \lambda}}
\newcommand{\hw}{\widehat{w}}
\newcommand{\expec}{\mathbb{E}_{X \sim \DD^{n}, \mathcal{A}}}

\newcommand{\btv}{\beta_{v}}

\newcommand{\almu}{\sqrt{\frac{2 \alpha}{\mu}}}

\newcommand{\ER}{\mathbb{E}_{\mathcal{A}} F(\wpr(X), X) - F(\ws(X), X)}
\newcommand{\ERT}{\mathbb{E}_{\mathcal{A}} \ftau(\wpr(X), X) - \ftau(\ws(X), X)}
\newcommand{\ERL}{\mathbb{E}_{\mathcal{A}_{\lambda}} F(\wla(X), X) - F(\ws(X), X)}
\newcommand{\ERP}{\mathbb{E}_{X \sim \DD^{n}, \mathcal{A}}F(\wpr(X), \DD) - F(\ws(\DD), \DD)}
\newcommand{\ERPL}{\mathbb{E}_{X \sim \DD^{n}, \mathcal{A}_{\lambda}}F(\wla(X), \DD) - F(\ws(\DD), \DD)}
\newcommand{\ERD}{\mathbb{E}_{\mathcal{A}} F(\wpr^{\delta}(X), X) - F(\ws(X), X)}
\newcommand{\ERDT}{\mathbb{E}_{\mathcal{A}} \ftau(\wpr^{\delta}(X), X) - \ftau(\ws(X), X)}
\newcommand{\ERDL}{\mathbb{E}_{\mathcal{A}_{\lambda}} F(\wla^{\delta}(X), X) - F(\ws(X), X)}
\newcommand{\ERDP}{\mathbb{E}_{X\sim \DD^{n}, \mathcal{A}}F(\wpr^{\delta}(X), \DD) - F(\ws(\DD), \DD)}
\newcommand{\ERDPL}{\mathbb{E}_{X\sim \DD^{n}, \mathcal{A}_{\lambda}}F(\wla^{\delta}(X), \DD) - F(\ws(\DD), \DD)}
\newcommand{\ERpro}{\mathbb{E}_{\mathcal{A'}} F(\wprpro(X), X) - F(\ws(X), X)}
\newcommand{\ERTpro}{\mathbb{E}_{\mathcal{A'}} \ftau(\wprpro(X), X) - \ftau(\ws(X), X)}
\newcommand{\ERLpro}{\mathbb{E}_{\mathcal{A'}_{\lambda}} F(\wlapro(X), X) - F(\ws(X), X)}
\newcommand{\ERPpro}{\mathbb{E}_{X \sim \DD^{n}, \mathcal{A'}}F(\wprpro(X), \DD) - F(\ws(\DD), \DD)}
\newcommand{\ERPLpro}{\mathbb{E}_{X \sim \DD^{n}, \mathcal{A'}_{\lambda}}F(\wlapro(X), \DD) - F(\ws(\DD), \DD)}
\newcommand{\ERDpro}{\mathbb{E}_{\mathcal{A'}} F(\wprpro^{\delta}(X), X) - F(\ws(X), X)}
\newcommand{\ERDTpro}{\mathbb{E}_{\mathcal{A'}} \ftau(\wprpro^{\delta}(X), X) - \ftau(\ws(X), X)}
\newcommand{\ERDLpro}{\mathbb{E}_{\mathcal{A'}_{\lambda}} F(\wlapro^{\delta}(X), X) - F(\ws(X), X)}
\newcommand{\ERDPpro}{\mathbb{E}_{X\sim \DD^{n}, \mathcal{A'}}F(\wprpro^{\delta}(X), \DD) - F(\ws(\DD), \DD)}
\newcommand{\ERDPLpro}{\mathbb{E}_{X\sim \DD^{n}, \mathcal{A'}_{\lambda}}F(\wlapro^{\delta}(X), \DD) - F(\ws(\DD), \DD)}

\newcommand{\ERA}{\mathbb{E}_{\mathcal{A}} G_D(\wpr) - \min_{w \in \mathbb{R}^d} G_D(w)}
\newcommand{\ERDA}{\mathbb{E}_{\mathcal{A}} G_D(\wpr^{\delta}) - \min_{w \in \mathbb{R}^d} G_D(w)}
\newcommand{\ERAP}{\mathbb{E}_{D \sim \DD^{n}, \mathcal{A}}[G_{\DD}(\wpr(D))] - \min_{w \in \mathbb{R}^d} G_{\DD}(w)}
\newcommand{\MR}{\mathcal{M}_{R}}

\DeclareMathOperator*{\argmax}{argmax}
\DeclareMathOperator*{\argmin}{argmin}
\usepackage{mathtools}
\DeclarePairedDelimiter{\ceil}{\lceil}{\rceil}

\newtheorem{theorem}{Theorem}[section]
\newtheorem{definition}{Definition}
\newtheorem{proposition}{Proposition}[section]
\newtheorem{lemma}{Lemma}[section]
\newtheorem{assumption}{Assumption}
\newtheorem{remark}{Remark}[section]
\newtheorem{corollary}{Corollary}[section]

\usepackage{cleveref}

\crefname{algorithm}{Algorithm}{Algorithms}
\crefname{assumption}{Assumption}{Assumptions}
\crefname{equation}{Equation}{Equations}
\crefname{figure}{Figure}{Figures}
\crefname{table}{Table}{Tables}
\crefname{section}{Section}{Sections}
\crefname{theorem}{Theorem}{Theorems}
\crefname{lemma}{Lemma}{Lemmas}
\crefname{proposition}{Proposition}{Propositions}
\crefname{definition}{Definition}{Definitions}
\crefname{corollary}{Corollary}{Corollaries}
\crefname{remark}{Remark}{Remarks}
\crefname{example}{Example}{Examples}
\crefname{appendix}{Appendix}{Appendices}

\newcommand{\range}{\mbox{range}}

\begin{document}
\title{Output Perturbation for Differentially Private Convex Optimization: Faster and More General 
}
\newcommand*\samethanks[1][\value{footnote}]{\footnotemark[#1]}

\author{
Andrew Lowy \hspace{0.5cm}
Meisam Razaviyayn
}

\date{
\texttt{\{lowya,razaviya\}@usc.edu} \\
\vspace{.2cm}
University of Southern California}
\maketitle

\begin{abstract}
Finding efficient, easily implementable differentially private (DP) algorithms that offer strong excess risk bounds is an important problem in modern machine learning. To date, most work has focused on private empirical risk minimization (ERM) or private stochastic convex optimization (SCO), which corresponds to population loss minimization. However, there are often other objectives---such as fairness, adversarial robustness, or sensitivity to outliers---besides average performance that are not captured in the classical ERM/SCO setups. Further, most recent work in private SCO has focused on $(\varepsilon, \delta)$-differential privacy ($\delta > 0$), whereas proving tight excess risk and runtime bounds for $(\varepsilon, 0)$-differential privacy remains a challenging open problem. Our first contribution is to provide \textit{the tightest known}~$(\varepsilon, 0)$-differentially private expected population loss bounds and \textit{fastest runtimes} for smooth and strongly convex loss functions. In particular, for SCO with well-conditioned smooth and strongly convex loss functions, we provide a \textit{linear-time} algorithm with \textit{optimal} excess risk. For our second contribution, we study DP optimization for a broad class of \textit{tilted} loss functions---which can be used to promote fairness or robustness, and are not necessarily of ERM form. We establish the first known differentially private excess risk and runtime bounds for optimizing this class; under smoothness and strong convexity assumptions, our bounds are near-optimal. For our third contribution, we specialize our theory to \textit{DP adversarial training}. Our results are achieved using perhaps the simplest yet practical differentially private algorithm: output perturbation. Although this method is not novel conceptually, our novel implementation scheme and analysis show that the power of this method to achieve strong privacy, utility, and runtime guarantees has not been fully appreciated in prior works.
\end{abstract}

\newpage
{
  \hypersetup{linkcolor=black}
  \tableofcontents
}

\newpage

\section{Introduction}
\label{sec 1}
\noindent Answering many science and engineering questions relies on computing solutions to  optimization problems (to find the best design parameters). The solution of these optimization problems typically leads to the design of a system or a model for the engineering (or machine learning) task at hand.  By observing the parameters of the system/model, an adversary can gain insights to the data used by the engineer/learner. For example, through observing the behavior of a trained machine learning model, an adversary can learn about the private data used during training procedure; or in smart grids, by observing the price of electricity, an adversary can infer power consumption patterns of the individual users. Other examples include eavesdropping the location of cellphone users through the transmission patterns in wireless networks, or revealing individuals' genomics data by observing the results of data processing procedures on biological data. Finding efficient, easily implementable \textit{differentially private} algorithms that offer strong excess loss/risk bounds is key in all these applications.

\vspace{0.2cm}
\textit{Differential privacy} (DP)~\cite{dwork2006calibrating} provides a rigorous guarantee that no adversary can infer much about any individual's data by observing the output of an algorithm. More precisely, a randomized algorithm $\mathcal{A}:\XX^n \to \mathbb{R}^{d}$ is said to be $(\varepsilon, \delta)$-differentially private if for all measurable subsets $\mathcal{K} \subseteq \range(\mathcal{A})$ and all $n$-element data sets $X, X' \in \XX^n$ which differ by at most one observation (i.e. $|X \Delta X'| \leq 2$), we have  \[
\mathbb{P}(\mathcal{A}(X) \in \mathcal{K}) \leq \mathbb{P}(\mathcal{A}(X') \in \mathcal{K}) e^{\varepsilon} + \delta,
\]
 where the probability is (solely) over the randomness of $\mathcal{A}$ \cite{dwork2014}. We  refer to  data sets differing in only one observation ($|X \Delta X'| \leq 2$) as \textit{adjacent}. If $\delta = 0,$ we may say an algorithm is $\varepsilon$-differentially private. 
 While large values of $\varepsilon$ can still provide some privacy, $\delta \ll 1$ is necessary for meaningful privacy guarantees. In fact, $\delta \ll \frac{1}{n}$ is typically desirable: otherwise, a model may leak the full data of a few individuals and still satisfy the privacy constraint \cite{dwork2014}. 
 To clarify further, \cite{dwork2014} explain: 
 \begin{quote}
``Even when $\delta > 0$ is negligible...there are theoretical distinctions between $(\varepsilon, 0)$- and $(\varepsilon, \delta)$-differential privacy. Chief among these is what amounts to a switch of quantification order. $(\varepsilon, 0)$-differential
privacy ensures that, for every run of the mechanism $\mathcal{A}(X)$, the output observed is (almost) equally likely to be observed on every neighboring database, simultaneously. In contrast $(\varepsilon, \delta)$-differential privacy
says that for every pair of neighboring databases $X$, $X'$, it is extremely
unlikely that, \textit{ex post facto} the observed value $\mathcal{A}(X)$ will be much more
or much less likely to be generated when the database is $X$ than when
the database is $X'$. However, given an output $\xi \sim \mathcal{A}(X)$ it may be possible to find a database $X'$ such that $\xi$ is much more likely to be produced on $X'$ than it is when the database is $X$. That is, the mass of $\xi$ in the distribution $A(X')$ may be substantially larger than its mass in the distribution $\mathcal{A}(x)$.''
 \end{quote}
Whereas a lot of the literature has focused on efficient algorithms for $(\varepsilon, \delta)$-differentially private algorithms, the important case of $\delta = 0$ has been neglected. \ul{The first contribution} of this work is to fill this void. See \cref{table:pop loss delta = 0} in \cref{sec: Appendix A: tables} for a summary of these results. 
In particular, we provide \ul{the tightest known~$(\varepsilon, 0)$-differentially private population loss bounds and fastest runtimes} under strong convexity and smoothness assumptions; we achieve \textit{near-optimal excess risk in nearly linear time}: see Corollary~\ref{rem: smooth sc pop loss runtime katyusha}.

\vspace{0.2cm}
Assume the parameters of a machine learning model are trained via solving the minimization problem: 
\begin{equation} \label{eq:generalMin}
\tilde{w} (X) \approx \arg\min_{w\in \mathbb{R}^d} F(w,X).
\end{equation}
In the case of empirical risk minimization (ERM), where $
    F(w,X) =  \frac{1}{n}\sum_{i=1}^{n} f(w,x_{i}),
$ 
algorithms for maintaining differential privacy while solving \cref{eq:generalMin} are well studied \cite{chaud2011, bst14, zhang2017, wang2017}. Here (and throughout) $X = (x_1 \cdots, x_n)$ is a data set with observations in some set $\XX \subseteq \mathbb{R}^q.$
More recently, several works have also considered private stochastic convex optimization (SCO), where the goal is to minimize the expected population loss $F(w,\mathcal{D}) = \mathbb{E}_{x \sim \DD}[f(w,x)],$ given access to $n$ i.i.d. samples $X = (x_1 \cdots, x_n)$ \cite{bft19, fkt20, arora20}. However, the algorithms in these works are only differentially private for $\delta > 0,$ which, as discussed earlier, provides substantially weaker privacy guarantees.\footnote{For certain function classes, the ERM results of \cite{bst14} and \cite{zhang2017} also extend to expected population loss (SCO) bounds for $\delta = 0$, but we improve or match the loss bounds and improve the runtime bounds for every function class. See \cref{table:pop loss delta = 0} for details.} Therefore, providing efficient, practical algorithms for $(\varepsilon, 0)$-differentially private SCO is an important gap that we fill in the present work.

\vspace{0.2cm}
\ul{Our second main contribution} is differentially private convex optimization for \textit{loss functions that are not of ERM or SCO form}. While ERM and SCO are useful if average performance is the goal, there are situations where another objective besides average performance is desirable. For example, one may want to train a machine learning model that ensures some subsets of the population are treated fairly (see e.g. \cite{datta2015automated}), or one that is robust to corrupted data or adversarial attacks \cite{goodfellow}, or one that has lower variance to allow for potentially better generalization.  One may also want to diminish the effect of outliers or increase sensitivity to outliers. In these cases, it may be more fruitful to consider an alternative loss function that is not of ERM form. For example, the max-loss function
\[
F(w,X) = \max\{f(w, x_{1}), ... f(w, x_{n}) \}
\] 
provides a model that has good ``worst-case'' performance and such $F$ is clearly not of ERM form. The recently proposed \textit{tilted ERM (TERM)} framework \cite{li2020} aims to address these shortcomings of standard ERM and encompasses the max-loss mentioned above. More generally, a benefit of the TERM framework is that it allows for a continuum of solutions between average and max loss, which, for example, can be calibrated to promote a desired level of fairness in the the machine learning model \cite{li2020}. Another interpretation of TERM is that as $\tau$ increases, the variance of the model decreases, while the bias increases. Thus, $\tau$ can also be tuned to improve the generalization of the model via the bias/variance tradeoff. 

\vspace{.2cm}
Existing differentially private utility and runtime results have all been derived specifically for standard ERM or SCO and therefore would not apply to TERM objectives. To address this limitation, we provide \ul{the first excess risk and runtime bounds for differentially private TERM}: see Section~\ref{sec: TERM}. In particular, for smooth strongly convex TERM, we derive excess risk bounds that extend our (near-optimal) differentially private ERM bounds, since the TERM objective encompasses ERM in the limit.

\vspace{0.2cm}
\ul{Our last main contribution} is to specialize  our theory and framework to \ul{DP adversarial training} in Section~\ref{subsection: DP adversarial training}. In adversarial training, the goal is to train a model that has robust predictions against  an adversary's perturbations (with respect to some perturbation set $S$) of the feature data. This problem has gained a lot of attention in recent years, since it was first observed that neural nets can often be fooled by tiny, human-imperceptible perturbations into misclassifying images \cite{goodfellow}. However, the challenging task of ensuring such adversarial training is executed in a differentially private manner has received much less attention by researchers. Indeed, we are not aware of any prior works that have shown how to keep the adversarial training procedure differentially private and provided adversarial risk and runtime bounds. Perhaps the closest step in this direction is the work \cite{phan}, which provides a differentially private algorithm for training a classifier that is ``certifiably robust,'' in the sense that with high probability, the classifier's predicted label is stable under small perturbations. However, our measure of adversarial robustness is different: we look at excess adversarial risk and provide bounds that depend explicitly on the privacy parameters. This allows for an interpretation of the tradeoffs between robustness, privacy, and runtime. Furthermore the noisy stochastic batch gradient descent algorithm of \cite{phan} is complicated to implement and does not come with runtime bounds.

\paragraph{Output Perturbation.}
Our theory is built on the  idea of \textit{output perturbation}.
Conceptually, the output perturbation mechanism outputs \[
 \wpr(X) := \Tilde{w}(X) + b, 
\]
where $\Tilde{w}$ is the output of some non-private algorithm %
and $b \in \mathbb{R}^d$ is some suitably chosen random noise vector. 
Output perturbation has been studied in the differential privacy literature for many years \cite{dwork2006, chaud2011, zhang2017}. In early works, which culminated in \cite{dwork2006}, the method was introduced and proven to be differentially private. %
In \cite{chaud2011}, excess risk bounds for linear classifiers with strongly convex regularizers in the ERM and SCO settings are given via output perturbation. However, no practical implementation is provided. As a first step in the practical direction, \cite{zhang2017} shows how to implement output perturbation with gradient descent in the smooth ERM setting, providing excess empirical risk bounds (which they describe how to extend to excess population loss bounds) and runtime bounds. Their privacy analysis is tied to the particular non-private optimization method they use, which hinders their runtime potential and makes their analysis less transparent. We build on this line of work, by showing \ul{how to implement output perturbation with \textit{any (non-private) optimization method}, transforming it into a differentially private one}. In particular, by using accelerated and stochastic optimization methods, our method facilitates substantial runtime improvements. In addition, our analysis shows that output perturbation extends well beyond smooth ERM losses: e.g., we derive excess risk bounds for TERM and non-smooth losses. 

\vspace{0.2cm}
\noindent\textbf{Notation.}
Recall that a function $h: \cal{W} \to \mathbb{R}$ on some domain $\mathcal{W} \subseteq \mathbb{R}^d$ is \textit{convex} if $h(\lambda w + (1-\lambda)w') \leq \lambda h(w) + (1-\lambda)h(w')$ for all $\lambda \in [0,1]$ and all $w, w' \in \cal{W}.$ We say $h$ is \textit{$\mu$-strongly convex} if $h(w) - \frac{\mu}{2}\|w\|_2^2$ is convex. Also, $h$ is \textit{$L$-Lipschitz} if $\|h(w) - h(w')\| \leq L\|w - w'\|_2$ for all $w, w' \in \cal{W}$ and \textit{$\beta$-smooth} if $h$ is differentiable and $\nabla h(w)$ is $\beta$-Lipschitz. 
The data universe $\XX$ can be be any set, and we often denote $X = (x_1, \cdots, x_n) \in \XX^n.$ For $R \geq 0,$ define \[
\mathcal{M}_{R} := \{F: \mathbb{R}^d \times \XX^n \to \mathbb{R}  \, | \, ~\forall X \in \XX^n ~\exists ~\ws(X) \in \argmin_{w \in \mathbb{R}^d} F(w,X), \|\ws(X)\|_2 \leq R \},
\]
and define $B(0,R)$ to be the closed Euclidean ball of radius $R$ centered at $0.$ 
Throughout this paper, we will work with the following subfamilies of functions in $\mathcal{M}_{R}$:
\[
\GG := \{F \in \mathcal{M}_{R} \, | \, F(\cdot, X) ~\text{is} ~L\text{-Lipschitz on} ~B(0,R) ~\text{and convex for all} ~X \in \XX^n \};
\]
\[
\JJ:= \{F \in \GG \, | \, F(\cdot, X) ~\text{is} ~\beta\text{-smooth for all} ~X \in \XX^{n} \};\]

\[
\FF := \{F \in \GG \, | \, F(\cdot, X) ~\text{is} ~\mu\text{-strongly convex for all} X \in \XX^n \};
\]

and 
\[ \HH:= \{F \in \FF \, | \, F(\cdot, X)~\text{~is} ~\beta\text{-smooth for all} ~X \in \XX^n \};\]

(In fact, all of our results for $\FF$ and $\JJ$ and $\HH$ apply if $F$ is merely $\mu$-strongly convex and/or $\beta$-smooth on $B(0,R),$ rather than on all of $\mathbb{R}^d,$ but some of the discussion and proofs are cleaner when we use the above definitions.) 
For $F \in \HH,$ denote the condition number by $\kappa = \frac{\beta}{\mu}.$ Also, $\FFE, \HHE, ... $ are defined to be the subset of functions in $\FF, \HH, ...$ that are of ERM form, i.e. $F(w,X) = \frac{1}{n} \sum_{i = 1}^n g(w,x_i)$ for all $w \in \mathbb{R}^d, ~X \in \XX^n$ for some convex $g: \mathbb{R}^d \times \XX \to \mathbb{R}.$ 

\vspace{.2cm}
For real-valued functions $a(\theta)$ and $b(\phi)$, where $\theta$ and $\phi$ are vectors of parameters, we may write $a(\theta) \lesssim b(\phi)$ if there exists an absolute constant $C > 0$ such that $a(\theta) \leq C b(\phi)$; alternatively, we may write $a = O(b).$ If the same inequality holds up to a logarithmic function of the parameters instead of an absolute constant $C,$ then we denote this relationship by $a = \widetilde{O}(b).$ Finally, define the following constant that is used in our choice of Gaussian noise vector: $\cd := \sqrt{\log\left(\frac{2}{\sqrt{16\delta + 1} - 1}\right)}$ \cite{zhao2019}. 

\vspace{0.2cm}
\noindent\textbf{Roadmap.}
The remainder of this paper is organized as follows. In \cref{Sec: Conceptual risk bounds}, we describe our conceptual output perturbation algorithm and state excess risk and runtime bounds for four classes of general (i.e. non-ERM) functions, along with corresponding results for ERM. Then in \cref{section 3: implment}, we describe implementation details for our practical output perturbation method and establish runtime bounds. Next, we give population loss and runtime bounds for $\varepsilon$-differentially private SCO on the same four function classes in \cref{Section 4: pop loss}. Finally, in \cref{section 5: applications}, we consider two applications of our results in machine learning: differentially private Tilted ERM (TERM) \cite{li2020} and adversarial training.

\section{Excess Risk Bounds for Conceptual Output Perturbation Method}
\label{Sec: Conceptual risk bounds}
\subsection{Conceptual output perturbation algorithm}
Assume for now that $F(\cdot, X)$ in \cref{eq: conceputal output pert} is strongly convex, so that there is a unique minimizer $\ws(X) = \arg\min_{w \in \mathbb{R}^d} F(w,X)$ for all $X \in \XX^n.$ Further, \textit{we begin by assuming that we can exactly compute $\ws(X)$ without regard to computational issues and present excess loss/risk bounds for the ``conceptual'' output perturbation algorithm}
\begin{equation}
\label{eq: conceputal output pert}
\mathcal{A}(X) = \wpr(X) := \ws(X) + z,    
\end{equation}
where $z$ is random noise vector in $\mathbb{R}^d$ whose distribution will be described below. 
If $F$ is not $\beta$-smooth, then we project $\wpr(X)$ onto $\WW$ after adding noise since $\ws(X) + z$ might not be in $\WW$ and we require $L$-Lipschitzness of $F$ at $\wpr(X)$ for our excess risk analysis (only in the non-smooth case). Denote the output of this latter algorithm by \begin{equation}
\label{eq: conceptual out pert proj}
    \mathcal{A'}(X) = w_{\mathcal{A'}}(X) := \Pi_{\WW}(\ws(X) + z) = \Pi_{\WW}(\wpr(X)),
\end{equation}
where $\Pi_{\WW}(z):= \arg\min_{w \in \WW} \|z - w\|_2 = \frac{z}{\max\left\{1,\frac{\|z\|}{R}\right\}}.$ 

Before describing the exact choice of noise, we introduce a notion that will be essential throughout our analysis. 
\begin{definition}
Define the \textbf{$L_{2}$ sensitivity} of a (strongly convex in $w$) function $F: \mathbb{R}^d \times \XX^n \to \mathbb{R}$ as \[
\Delta_{F} := \sup_{X, X' \in \XX^n, |X \Delta X'| \leq 2}\|w^{*}(X) - w^*(X')\|_{2}, \]
where $\ws(X) = \argmin_{w \in \mathbb{R}^d} F(w,X).$ 
\end{definition} 
That is, the sensitivity of a function is the maximum $L_2$ distance between the minimizers of that function over all pairs of data sets that differ in only one observation $x_{i} \neq x_{i}' ~(|X \Delta X'| \leq 2)$. Note that $\Delta_{F}$ clearly depends on the set $\XX$ (and on $n$ and $d$), but we will hide this dependence in our notation, as we consider $\XX$ (and $n$ and $d$) to be fixed. We sometimes may also just write $\Delta$ if it is clear which function $F$ we are referring to. Note that this definition is different from the notion of ``sensitivity'' that is sometimes used in the differential privacy literature (e.g. \cite{dwork2014}), where ``sensitivity of a function $g$'' means $\sup_{X, X' \in \XX^n, |X \Delta X'| \leq 2}\|g(X) - g(X')\|_{2}.$ We prefer our notation since in our case, the function $g$ of interest will almost always be $\ws,$ but we want to emphasize the dependence of $\ws$ and $\Delta$ on the underlying function $F.$

\subsection{Choice of Noise}
\label[section]{subsection 2.2}
There are many different choices of noise $z$ to use in \cref{eq: conceputal output pert} and \cref{eq: conceptual out pert proj} for achieving differential privacy via output perturbation. Our choices are tailored to the setting of machine learning/optimization with convex Lipschitz (with respect to the Euclidean $L_2$ norm) loss. 

\vspace{0.2cm}
\begin{itemize}

\item To achieve $(\varepsilon, 0)$-privacy, we take $z$ with the probability density function \begin{equation}
\label{eq: noise for delta = 0}
p_{z}(t) \propto \exp\left(- \frac{\varepsilon \|t\|_2}{\Delta_{F}}\right)\end{equation} for $t \in \mathbb{R}^d.$
We have $\|z\|_2 \sim \Gamma(d, \frac{\Delta_{F}}{\varepsilon}).$ \footnote{Here we say a random variable $Y \sim \Gamma(k, \theta)$ if it has density $p_Y(y) = {\displaystyle {\frac {1}{\Gamma (k)\theta ^{k}}}y^{k-1}e^{-{\frac {y}{\theta}}}}$. Note $\mathbb{E}Y = k \theta$ and $Var(Y) = k \theta^2.$ To see that $\|z\|_2$ has this distribution, one can write, for $s \in \mathbb{R}^d$, $p_{\|z\|}(s) \propto \int_{\partial B(s)} \exp(-\|t\| \frac{\varepsilon}{\Delta_F}) \propto e^{-s \frac{\varepsilon}{\Delta_F}} s^{d-1},$ where $\partial B(s)$ denotes the Euclidean sphere of radius $s$ centered at the origin in $\mathbb{R}^d$ and we used polar coordinates/surface measure of $\partial B(s)$ to evaluate the integral.}

\item For $(\varepsilon, \delta)$-privacy with $\delta \in \left(0, \frac{1}{2}\right)$, we take \begin{equation}
\label{eq: noise for delta > 0}
z = z_{\delta} \sim N(0, \sigma^2 I_{d}), ~\sigma^2 = \left(\frac{\cd + \sqrt{\cd^2 + \varepsilon}}{\sqrt{2} \varepsilon}\right)^2 \Delta_{F}^2, \end{equation}
where $\cd := \sqrt{\log\left(\frac{2}{\sqrt{16\delta + 1} - 1}\right)},$ as proposed in \cite{zhao2019}.  To distinguish between the output of the algorithm(s) with these different choices of noise, we shall denote $\wpr^{\delta}(X) = \ws(X) + z_{\delta}$ and $\wprpro^{\delta}(X) = \Pi_{\WW}(\ws(X) + z_{\delta}).$
\end{itemize}
Note that many other works (e.g. \cite{zhang2017}, \cite{bst14}, \cite{wang2017}, \cite{fkt20}) instead use Gaussian noise with variance $\widetilde{\sigma}^2 = \frac{2 \Delta^2 \log(\frac{a}{\delta})}{\varepsilon^2}$, where $a > 1$: this is the classical Gaussian mechanism introduced in \cite{dwork2006} and improved in \cite{dwork2014}. For $\widetilde{z} \sim N(0, \widetilde{\sigma}^2),$ output perturbation is only $(\varepsilon, \delta)$  differentially private for $\varepsilon \leq 1,$ as explained in \cite{zhao2019}. Our choice of $\sigma$, borrowed from \cite{zhao2019}, has the following advantages: $\varepsilon > 1$ is allowed  and in the comparable regime where $\varepsilon < 1$, $\sigma \leq \widetilde{\sigma},$ which leads to tighter excess risk upper bounds. On the other hand, $\delta \ll 1$ is necessary for meaningful privacy guarantees, and usually $\delta \ll \frac{1}{n}$ is desired, as discussed in the Introduction. So requiring that $\delta < 1/2$ is not restrictive indeed. Requiring $\varepsilon \leq 1$ is much more restrictive since $\varepsilon > 1$ can still provide meaningful (albeit weaker) privacy guarantees if $\delta \ll 1$. However, if for some reason $\frac{1}{2} \leq \delta < 1$ and $\varepsilon \leq 1$ is desired, then one may use the alternative form of Gaussian noise in our algorithms instead, while only affecting the excess risk bounds up to a logarithmic factor and not affecting runtime. 
For the rest of this paper, we assume $\varepsilon > 0$, $\delta \in [0, \frac{1}{2}).$ 
We begin by establishing that our conceptual algorithm is differentially private (for \textit{any} $\varepsilon > 0$). 

\begin{proposition}
\label[proposition]{conceptual sc alg is private}
Let $\varepsilon > 0, \delta \in [0, \frac{1}{2}).$ Let $F: \mathbb{R}^d \times \XX^{n} \to \mathbb{R}$ be a
function with $L_2$ sensitivity $\Delta_{F}$ and unique minimizer $\ws(X) = \argmin_{w \in \mathbb{R}^d} F(w,X)$ for all $X \in \XX^n.$ The (conceptual) output perturbation algorithms $\mathcal{A}(X) = \wpr(X) := \ws(X) + z$ and $\mathcal{A'}(X) = \wprpro(X) = \Pi_{\WW}(\ws(X) + z)$ with noise given by density \[p_z(t) \propto \begin{cases} 
\exp\left\{-\frac{\varepsilon \|t\|_{2}}{\Delta_{F}}\right\} &\mbox{if } \delta = 0 \\
\exp\left\{-\frac{\varepsilon^2 \|t\|_{2}^2}{\Delta_{F}^2 (\cd + \sqrt{\cd^2 + \varepsilon})^2}\right\} &\mbox{if } \delta \in (0, \frac{1}{2})
\end{cases} \]
are $(\varepsilon, \delta)$-differentially private. Here $\cd := \sqrt{\log\left(\frac{2}{\sqrt{16\delta + 1} - 1}\right)}.$
\end{proposition}

All proofs are relegated to the Appendix. 

\vspace{0.2cm}
Note that if we do not have access to $\Delta_{F}$ (or do not wish to compute it), then we can replace $\Delta_{F}$ in the probability density function of $z$ by it's upper bound, given in \cref{prop:sc sensitivity}. 
This will still ensure the algorithm is $(\varepsilon, \delta)$-private and all of our excess risk bounds still hold. 

\vspace{0.2cm}
In the remainder of this section, we will study particular classes of functions and aim to understand how well our output perturbation method performs at privately minimizing the (worst case over all functions in the class and data sets $X \in \XX^n$) excess risk on these classes. We begin with strongly convex, Lipschitz functions in \cref{subsection 2.3: sc Lip ER bounds} and then add additional smoothness assumption in \cref{subsection 2.4 convex lip}. 
We cover the non-strongly-convex case in~\cref{app: Sec 2 proofs}.
Our general strategy will be to estimate (upper bound) the sensitivity $\Delta_{F}$ of functions in a given class and then use convexity, Lipschizness, and smoothness to upper bound the excess risk of our algorithm $\ER$, in terms of $\Delta_{F}.$ 

For the remainder of this paper, unless otherwise stated, we fix $n, d \in \mathbb{N}$ and consider \ul{$\XX$ as a fixed arbitrary set} in any space. 

\subsection{Excess risk bounds for strongly convex, Lipschitz functions}
\label{subsection 2.3: sc Lip ER bounds}
In this subsection, assume our convex objective function $F$ in \cref{eq: conceputal output pert} belongs to the following function class: 

\[
\FFE = \{F \in \FF \, | \, \exists F: \mathbb{R}^d \times \XX \to \mathbb{R} \, | \, F(w,X) = \frac{1}{n}\sum_{i=1}^{n } f(w, x_{i}) ~\forall X \in \XX^n\}.
\]
We say such $F$ are ``of ERM form.'' Observe that for $F$ of ERM form, i.e. $F(w,X) = \frac{1}{n}\sum_{i=1}^{n} f(w, x_{i})$, we have that $F(\cdot,X)$ is $\mu$-strongly convex for all data sets $X = (x_1, \cdots, x_n) \in \XX^n$ if and only if $f(\cdot, x)$ is $\mu$-strongly convex for all $x \in \XX$. Moreover, $F(\cdot,X)$ is $L$-Lipschitz for all data sets $X = (x_1, \cdots, x_n) \in \XX^n$ if and only if $f(\cdot, x)$ is $L$-Lipschitz for all $x \in \XX.$ Similarly, $F(\cdot,X)$ is $\beta$-smooth for all data sets $X = (x_1, \cdots, x_n) \in \XX^n$ if and only if $f(\cdot, x)$ is $\beta$-smooth for all $x \in \XX$. \footnote{Proof: The proof for all of the properties is similar. For concreteness, let's consider the $\mu$-strong convexity condition. Assume $f(\cdot, x)$ is $\mu$-strongly convex for all $x \in \XX$ (i.e. $g(w,x):= f(w,x) - \frac{\mu}{2}\|w\|_2^2$ is convex in $w$ for all $x \in \XX$) and let $w, w' \in \WW, ~X \in \XX^n, ~\lambda \in [0,1].$ Then denoting $G(w,X):= \frac{1}{n} \sum_{i=1}^n g(w,x_{i}) =  \frac{1}{n} \sum_{i=1}^n f(w,x_{i}) - \frac{\mu}{2}\|w\|^2,$ we have \begin{align*}
G(\lambda w + (1-\lambda)w', X) &= \frac{1}{n} \sum_{i=1}^n  [f(\lambda w + (1-\lambda) w', x_{i}) - \frac{\mu}{2}\|w\|^2] \\
&=  \frac{1}{n} \sum_{i=1}^n g(\lambda w + (1-\lambda) w', x_{i})\\
&\leq \frac{1}{n} \sum_{i=1}^n [\lambda g(w, x_{i}) + (1 - \lambda) g(w', x_{i})] \\
&= \lambda G(w,X) + (1-\lambda) G(w',X),
\end{align*}
so that $G(\cdot, X)$ is convex for all $X \in \XX^n,$ and hence $F(\cdot, X)$ is $\mu$-strongly convex for all $X \in \XX^n.$ Conversely, suppose $f(\cdot, x)$ is not $\mu$-strongly convex for all $x \in \XX$. Then there exists $x \in \XX, ~w, w' \in \WW, ~\lambda \in [0,1]$ such that $g(\lambda w + (1-\lambda)w', x) > \lambda g(w, x) + (1-\lambda) g(w',x)$. Take $X = (x, \cdots, x) \in \XX^n$ to be $n$ copies of the ``bad'' data point $x.$ Then \begin{align*}
    G(\lambda w + (1-\lambda) w', X) &= \frac{1}{n} \sum_{i=1}^n g(w,x_{i}) \\
    &> \frac{1}{n} \sum_{i=1}^n [\lambda g(w, x) + (1-\lambda) g(w',x)] \\
    &= \lambda G(w,X) + (1-\lambda) G(w',X),
\end{align*}
which implies that $F(\cdot, X)$ is not $\mu$-strongly convex for all $X \in \XX^n.$
A similar argument easily establishes that the other assertions hold. 
} 

\vspace{0.2cm}
Our main result for this subsection is \cref{cor: smooth sc fund ER bounds}, in which we establish \textit{near-optimal} excess risk bounds for the subset of \textit{smooth} functions that are contained in $\FFE.$ 

\vspace{.2cm}
As a first step towards deriving these excess risk upper bounds, we will need to estimate the sensitivity of $F \in \FFE.$ The following result generalizes \cite[Corollary 8]{chaud2011} (beyond smooth linear ERM classifiers with bounded label space).
\begin{proposition}
\label[proposition]{prop:sc sensitivity}
For any $F \in \FFE,\Delta_{F} \leq \frac{2L}{\mu n}.$ 
\end{proposition}

In \cref{app: sens lower bound}, we prove that \textit{the above sensitivity estimate for $\FFE$ is tight}: there exists $F \in \FFE$ such that $\Delta_{F} = \frac{L}{\mu n}$. However, there are some simple functions in $\FFE$ for which these estimates are not tight. For example, consider the following loss function: $f(w,x) = \frac{\beta}{2}\|Mw - x\|_2 ^2$ for a symmetric positive definite $d \times d$ matrix $M$ with 
$\beta\mathbf{I} \succcurlyeq  \beta M^2 \succcurlyeq \mu \mathbf{I}$ with $\XX \subseteq \WW  \subset \mathbb{R}^d$ and $\beta \gg \mu$. Then $f$ is $L := 2\beta R$-Lipschitz, and hence $F(w,X) = \frac{1}{n}\sum_{i=1}^n \frac{\beta}{2}\|Mw - x_{i}\|_2 ^2\in \FFE.$ In this case, by taking the gradient and computing $\ws(X) = M^{-1} \bar{X}$ directly (where $\bar{X} := \frac{1}{n} \sum_{i=1}^{n} x_{i}$), we can obtain the more precise sensitivity estimate $\Delta_{F} \leq \frac{2R \sqrt{\kappa}}{n}$ (where $\kappa = \frac{\beta}{\mu}$): for any $X, X' \in \XX^n$ such that $|X \Delta X'| = 2$ (say WLOG $x_{n} \neq x'_{n}$), we have \begin{align*}
    \| \ws(X) - \ws(X') \|_2 &= \|M^{-1} \bar{X} - M^{-1} \bar{X'}\|_2 \\
    &\leq \|M^{-1}\|_2 \|\bar{X} - \bar{X'}\|_2 \\
    &\leq \sqrt{\kappa} \left\|\frac{1}{n} (x_{n} - x'_{n})\right\|_2 \\
    &\leq \frac{2R \sqrt{\kappa}}{n}. 
\end{align*} 
Note that since $L = 2 \beta R$ for this $F$, we have $\Delta_{F} \leq \frac{L \sqrt{\kappa}}{\beta n} = \frac{L}{\sqrt{\beta \mu} n} \ll \frac{2L}{\mu n}$. 

\vspace{0.2cm}
Next, we establish excess risk upper bounds for our conceptual algorithm that depend on $\Delta_{F}$. 
\begin{proposition}
\label[proposition]{prop: sc ER} Let $F \in \FFE$ and $w_{\mathcal{A'}}(X)$ be the perturbed solution defined in \cref{eq: conceptual out pert proj}. \\
a) If $\delta = 0$ and the noise $z$ has the distribution in \cref{eq: noise for delta = 0}, then \[
\ERpro \leq \frac{L \Delta_{F} d}{\varepsilon}.\] 

b) If $\delta \in \left(0, \frac{1}{2}\right)$ and the noise $z$ has distribution in \cref{eq: noise for delta > 0}, then \[
\ERDpro \leq L \sqrt{d} \sig \leq \frac{L \sqrt{d}\left(2\cd + \sqrt{\varepsilon}\right)}{\sqrt{2}\varepsilon} \Delta_{F},\] where $\cd = \sqrt{\log\left(\frac{2}{\sqrt{16\delta + 1} - 1}\right)}.$
\end{proposition}

Plugging the sensitivity estimates from \cref{prop:sc sensitivity} into \cref{prop: sc ER} yields excess risk upper bounds for the class $\FFE:$

\begin{corollary}
\label[corollary]{cor: sc upper bound}
Let $F \in \FFE.$ \\
a) If $\delta = 0,$ then \[
\ERpro \leq \frac{2 L^2 d}{\mu \varepsilon n}.\] \\
b) If $\delta \in \left(0, \frac{1}{2}\right),$ then \\
\[
\ERDpro \leq \frac{\sqrt{2} L^2}{\mu} \left(\lpen\right) \leq \frac{\sqrt{2} L^2}{\mu} \frac{\sqrt{d}(2\cd + \sqrt{\varepsilon})}{\varepsilon n}. 
\]

Moreover, if $\sqrt{\varepsilon} \lesssim \cd$, then \[
\ERDpro \lesssim \frac{L^2}{\mu} \frac{\sqrt{d} \cd}{\varepsilon n}.\]
\end{corollary}

Notice that in the practical regime $\delta \ll 1$, the condition $\sqrt{\varepsilon} \lesssim \cd$ is only restrictive if $\varepsilon > 1$. In this regime, other algorithms (e.g. gradient descent output perturbation \cite{zhang2017}, noisy SVRG \cite{wang2017}) that use the alternative form of Gaussian noise are not even differentially private in general, as shown in \cite[Theorem 1]{zhao2019}. For $\varepsilon \leq 1$, the condition $\sqrt{\varepsilon} \lesssim \cd$ automatically holds for any $\delta \in \left(0, \frac{1}{2}\right)$: indeed, $\sqrt{\varepsilon} \leq \cd$ if $\varepsilon \leq 1$ and $\delta \leq \frac{1}{9}.$

\begin{remark} 
Consider the trivial algorithm that outputs some fixed $w_{0} \in \WW$ (regardless of the input data or function), which has excess risk $\leq LR.$ Then the above excess risk upper bounds (and indeed all other upper bounds in this paper) can be written as $\min\{LR, ...\}$. Thus, the upper bounds given in \cref{cor: sc upper bound} for $\FFE$ are non-trivial only if $\pen < 1$ (for $\delta = 0$) or $\lpen < 1$ (for $\delta > 0$).

\end{remark}

\begin{remark}[Suboptimality of output perturbation for non-smooth strongly convex ERM]
The private excess empirical risk lower bounds for $\FFE$ are $\Omega\left(LR\min\left\{1, \left(\pen\right)^2\right\}\right)$ for $\delta = 0$ and $\Omega\left(LR\min\left\{1, \frac{d}{(\varepsilon n)^2}\right\}\right)$ for $\delta = o(\frac{1}{n})$ \cite{bst14}. 
\end{remark}

\vspace{0.2cm}

In the ERM setting, \cite{bst14} proposes algorithms that nearly achieve the excess risk lower bounds. For $\delta = 0$, exponential mechanism with localization gives excess empirical risk of $O\left(\lmu \left(\pen\right)^2 \log(n)\right)$, which is loose by a factor of $O\left(\log(n) \frac{L}{\mu R}\right)$ \cite{bst14}. Compared to this upper bound, our ERM upper bound is loose by a factor of $O\left(\frac{d}{\varepsilon n}\right)$. However, our algorithm has some major advantages over their algorithm. First of all, our resulting population loss bounds beat theirs (by a logarithmic factor) (see \cref{Section 4: pop loss}). Second, the  algorithm proposed in \cite{bst14} is impractical due to lack of implementation details for the localization/output perturbation step and large computational cost \footnote{The runtime of the exponential sampling step with input $\mathcal{W}_0$ from the localization step (assuming the localization/output perturbation step is done) is $\widetilde{O}\left(\left(\frac{L}{\mu}\right)^2  \frac{n d^9}{\varepsilon^2} \max\left\{1, \frac{L}{\mu}\right\}\right)$ in the worst case; it improves by a factor of $d^3$ if the convex set resulting from localization is in isotropic position \cite[Theorem 3.4]{bst14}. If one uses our methods to do that step efficiently, say via SGD (\cref{section 3: implment}), then the total runtime of their full algorithm (for general $\WW$) would be $\widetilde{O}\left(\left(\frac{L}{\mu}\right)^2  \frac{d^9}{\varepsilon^2} \max\{1, \frac{L}{\mu}\} + n^2 d\right),$ which is always larger than the runtime of our method.}. On the other hand, we will show in the next section that our algorithm can be implemented essentially as fast as any non-private algorithm. 
For $\delta > 0$, noisy SGD is an efficient practical algorithm, that can nearly (up to poly$\log(n)$ factor) achieve the ERM lower bound in runtime $O(n^2 d)$, but it does not apply for arbitrary $\varepsilon > 0$ due to the suboptimal choice of Gaussian noise \cite[Section 2]{bst14}. In fact, their method requires $\varepsilon \leq 2 \sqrt{\log\left(\frac{1}{\delta}\right)}$, which restricts the level of excess risk that can be achieved. Furthermore, output perturbation has the benefit of simplicity and is easier (and computationally more efficient, as we will see in \cref{section 3: implment}) to implement, while also leading to population loss bounds (\cref{Section 4: pop loss}) that outperform the bounds given in \cite[Appendix F]{bst14}. For these reasons, we consider our algorithm to be a competitive choice for $\FFE$.

\vspace{.2cm}
Next, we show how to improve the above results and achieve near-optimal excess risk by adding a smoothness assumption.  

\subsubsection{Excess risk bounds for smooth, strongly convex, Lipschitz functions.}
\label[section]{subsection 2.5: smooth sc}
While $\beta$-smoothness does not improve the sensitivity estimates in \cref{prop:sc sensitivity} (as the sensitivity estimates are tight by \cref{app: sens lower bound}), it can be used to yield stronger excess risk bounds, which are nearly optimal in the ERM setting. Consider the smooth, strongly convex, Lipschitz function class

\[
\HHE := \{F \in \FFE : F(\cdot, X) \text{~is} ~\beta\text{-smooth} \text{~for all} ~X \in \XX^n  \}.\] For $F \in \HH$, denote the condition number of $F$ by $\kappa = \frac{\beta}{\mu}.$ Our main result in this subsection is \cref{cor: smooth sc fund ER bounds}. 
For $F$ in the ERM class $\HHE$, we bound excess empirical risk by $O\left(\kappa \lmu \left(\pen\right)^2 \right)$ for $\delta = 0$, assuming $d < \varepsilon n.$ Our bound is $\widetilde{O}\left(\kappa \lmu \frac{d}{\varepsilon^2 n^2}\right)$ for $\delta > 0$, assuming $d < \varepsilon^2 n^2$. We obtain these results through the conceptual output perturbation algorithm $\mathcal{A}$ \textit{without projection}, defined in \cref{eq: conceputal output pert}. As a first step, we obtain an analogue of \cref{prop: sc ER} for $\beta$-smooth functions:
 
\begin{proposition}
\label[proposition]{lem: smooth ER}
Suppose $F(\cdot, X)$ is $\beta$-smooth for all $X \in \XX^{n}$ with $L_2$ sensitivity $\Delta_{F}.$ 

a) If $\delta = 0,$ then \[
\ER \leq \frac{\beta}{2} \frac{d(d+1) \Delta_{F}^2}{\varepsilon^2} \leq \frac{\beta d^2 \Delta_{F}^2}{\varepsilon^2}.\]

b)If $\delta \in (0, \frac{1}{2}),$ then 
\[\ERD \leq \frac{\beta}{2} d \sigma^2 = 
\frac{\beta}{4} \left(\frac{\cd + \sqrt{\cd^2 + \varepsilon}}{\varepsilon}\right)^2 \Delta_{F}^2,\]
where $\cd = \sqrt{\log\left(\frac{2}{\sqrt{16\delta + 1} - 1}\right)}.$
\end{proposition}

Combining \cref{lem: smooth ER} with \cref{prop:sc sensitivity} yields the following near-optimal excess risk guarantees:

\begin{corollary} 
\label[corollary]{cor: smooth sc fund ER bounds}
Let $w_{\mathcal{A}}(X)$ be the perturbed solution defined in \cref{eq: conceputal output pert}.

a) Suppose $\delta = 0$ and the noise $z$ has the distribution in \cref{eq: noise for delta = 0}. \\
Then \[
\ER \leq 4 \kappa \frac{L^2}{\mu} \left(\frac{d}{\varepsilon n}\right)^2.\]

b) Suppose $\delta \in (0, \frac{1}{2})$ and the noise $z$ has the distribution in \cref{eq: noise for delta > 0}. Then \[\ERD \leq 4 \kappa \frac{L^2}{\mu} \left(\frac{\sqrt{d}}{\varepsilon n}\right)^2 \left(\cd + \sqrt{\cd^2 + \varepsilon}\right)^2 
\leq 4 \kappa \frac{L^2}{\mu} \left(\frac{\sqrt{d}}{\varepsilon n}\right)^2 \left(\cd + \sqrt{\varepsilon}\right)^2.\]
Furthermore, if $\sqrt{\varepsilon} \lesssim \cd$, then \[\ERD \lesssim \kappa \frac{L^2}{\mu} \left(\frac{\sqrt{d} \cd}{\varepsilon n}\right)^2.\]

\end{corollary}

\begin{remark}
\label[remark]{rem: min H}
Since $\HHE \subseteq \FFE$, by \cref{cor: sc upper bound}, we obtain the following upper bounds on the excess risk of functions in $\HHE$: 
$\ER \lesssim \min \left\{LR,  \frac{L^2}{\mu} \frac{d}{\varepsilon n}, \kappa \frac{L^2}{\mu} \left(\frac{d}{\varepsilon n}\right)^2 \right \}$ for $\delta = 0$, and
$\ERD \lesssim \min \left\{LR, \frac{L^2}{\mu} \frac{\sqrt{d}(\cd + \sqrt{\varepsilon})}{\varepsilon n}, \kappa \frac{L^2}{\mu} \frac{d (\cd + \sqrt{\varepsilon})^2}{(\varepsilon n)^2} \right\}$ for $\delta > 0.$ (Here $\mathcal{A}$ refers to whichever conceptual algorithm yields smaller excess risk; it may or may not involve projection after perturbation.) Furthermore, as explained in earlier remarks, for $\delta \ll 1$ or $\varepsilon \leq 1,$ we have $\sqrt{\varepsilon} \lesssim \cd,$ so that the upper bounds for $\HHE$ can be written more simply as $\widetilde{O}\left(\min\left\{LR, \lmu \frac{\sqrt{d}}{\varepsilon}, \kappa \lmu \frac{d}{\varepsilon^2} \right\}\right)$ and $\widetilde{O}\left(\min\left\{LR, \lmu \frac{\sqrt{d}}{\varepsilon n}, \kappa \lmu \frac{d}{\varepsilon^2 n^2} \right\}\right)$ for $\delta = 0$  and $\delta > 0$, respectively.
\end{remark}

\vspace{0.2cm}
The excess risk lower bounds established in \cite{bst14} for $\FFE$ with $L \approx \mu R$ also hold for $\HHE$ \textit{if} $\beta \approx \mu$ (i.e. $\kappa = O(1)$) since their $L$-Lipschitz $\mu$-strongly convex hard instance is also $\mu$-smooth. Intuitively, it would seem that the class $\HHE$ with $\beta > \mu$ is at least as hard to privately optimize than $\mathcal{H}^{ERM}_{\mu, \mu, L}$, but this is not proved in \cite{bst14}. Our next result confirms that this intuition does indeed hold:

\begin{proposition}
\label[proposition]{prop: lower bound for HHE}
Let $\beta, \mu, R, \varepsilon > 0,$ and $n, d \in \mathbb{N}$. Set $\XX = B(0,R) \subset \mathbb{R}^d.$ For every $\varepsilon$-differentially private algorithm $\mathcal{A}$ with output $w_{\mathcal{A}}$, there exists a data set $X \in \XX^n$ and a function $F: \mathbb{R}^d \times \XX^n \to \mathbb{R}$ such that $F \in \HHE$ for $L = 2 \beta R$ and \[
\ER = \Omega\left(LR \min \left\{1,\left(\pen\right)^2 \right\}\right).
\]
Moreover, if $\delta = o(1/n)$, then for every $(\varepsilon, \delta)$-differentially private algorithm $\mathcal{A},$ there exists a data set $X \in \XX^n$ such that 
\[
\ER = \Omega\left(LR \min \left \{1, \left(\frac{\sqrt{d}}{\varepsilon n}\right)^2 \right\}\right).
\]
\end{proposition}

\vspace{0.2cm}
For $\HHE,$ there are several competing differentially private algorithms. A summary is given in \cref{table:ERM delta = 0} and \cref{table: ERM delta > 0} in \cref{sec: Appendix A: tables} and a more detailed discussion is provided in \cref{section 3: implment}, after we derive runtime bounds. Here we note that we are not aware of any algorithm that is able to exactly achieve the lower bounds for $\HHE$. The upper bounds in \cite{bst14} for $\FFE$ do not benefit from adding the smoothness assumption and therefore are still suboptimal by a factor of $O\left(\log(n) \frac{L}{\mu R}\right)$ for the class $\HHE$ in general. Similarly, for $\delta > 0,$ both \cite{bst14} and \cite{wang2017} obtain excess risk upper bounds of $O\left(\lmu \frac{d}{(\varepsilon n)^2} \log(n)\right)$ for restricted ranges of $\varepsilon.$ Thus, \textit{when $\kappa \lesssim \log(n)$, our upper bounds for $\HHE$ are the tightest we are aware of}. For $\delta > 0$, we also are the only method that is differentially private without restrictions on the $\varepsilon$. Furthermore, we show in \cref{section 3: implment} that \textit{our runtime is also faster} than both of these competitors, and in \cref{Section 4: pop loss}, we show that our method yields tighter population loss bounds.

\vspace{0.2cm}
We end this subsection by returning to the example introduced earlier and using the tight sensitivity bounds obtained earlier together with the results of this section to get tight excess risk bounds. Recall the function of interest is $f(w,x) = \frac{\beta}{2}\|Mw - x\|_2 ^2$ for symmetric positive definite $M$ with $\beta \mathbf{I} \succcurlyeq \beta M^2 \succcurlyeq \mu \mathbf{I}$ with $\WW = \XX = B_{2}(0, R) \subseteq \mathbb{R}^d$. For this function, we showed that $\Delta_{F} \leq \frac{2R \sqrt{\kappa}}{n} = \frac{L}{\sqrt{\beta \mu} n}$. Then applying \cref{lem: smooth ER} gives the excess risk upper bounds $O\left(\frac{L^2}{\mu} \left(\pen\right)^2\right)$ for $\delta = 0$ and $\widetilde{O}\left(\frac{L^2}{\mu} \frac{d}{(\varepsilon n)^2}\right)$ for $\delta > 0$. These nearly optimal (up to $\frac{L}{\mu R}$ factor) upper bounds are tighter than those of any other competitor we are aware of. This example demonstrates how our general framework and analysis can result in tighter bounds than the general upper bounds given when specialized to specific functions or subclasses.

\vspace{0.2cm}
In the next subsection, we conduct a similar analysis under the weakened assumption that $F(\cdot, X)$ is merely convex and $L$-Lipschitz. 

\subsection{Excess risk bounds for convex, Lipschitz functions}
\label[section]{subsection 2.4 convex lip}
Consider the convex, Lipschitz function classes
\[\GG = \{F \in \MR \, | \ F(\cdot, X) \text{~is} ~L\text{-Lipschitz on}~\WW ~\text{and convex for all} ~X \in \XX^n\};\]
and 
\[
\GGE := \{F \in \GG \colon \exists f: \mathbb{R}^d \times \XX \to \mathbb{R} ~\text{s.t.} ~F(w,X) = \frac{1}{n}\sum_{i=1}^{n} f(w, x_{i}) ~\forall X \in \XX^n\}.\]
Observe that this definition implies that $f(\cdot,x)$ must be convex for all $x \in \XX$. \footnote{The proof of this fact is very similar to the proof of the comparable assertion for strong convexity, given in \cref{subsection 2.3: sc Lip ER bounds}.} 

\vspace{0.2cm}
Note that upper bounding the sensitivity in \cref{prop:sc sensitivity} requires strong convexity. So it is not immediately clear how to derive output perturbation excess risk bounds for functions in $\GG.$ To do so, we consider the regularized objective $F_{\lambda}(w,X) := F(w,X) + \frac{\lambda}{2} \|w\|_{2}^2.$ Note that $\Fl(\cdot, X)$ is $(L + \lambda R)$-Lipschitz on $\WW$ and $\lambda$-strongly convex on $\mathbb{R}^d$ for any $\lambda > 0$. Now we can run our conceptual output perturbation algorithm on $\Fl,$ and use our results from the previous section. Our conceptual (regularized) output perturbation algorithm is then defined as \begin{equation}
\label{regularized output pert - conceptual alg}
\mathcal{A'}_{\lambda}(X) = w_{\mathcal{A'}, \lambda}(X) := \Pi_{\WW}(\wls(X) + \zl),
\end{equation}
where $\wls(X) = \arg\min_{w \in \WW} \Fl(w,X)$ and $\zl$ is a random noise vector with density \[ p_{\zl}(t) \propto
\begin{cases} 
\exp\left\{-\frac{\varepsilon \|t\|_{2}}{\Delta_{\lambda}}\right\} &\mbox{if } \delta = 0 \\
\exp\left\{-\frac{\varepsilon^2 \|t\|_{2}^2}{\Delta_{\lambda}^2 (\cd + \sqrt{\cd^2 + \varepsilon})^2}\right\} &\mbox{if } \delta \in (0, \frac{1}{2}).
\end{cases} \]
Here $\Delta_{\lambda} :=  \sup_{X, X' \in  |X \Delta X'| \leq 2}\|w_{\lambda}^{*}(X) - w_{\lambda}^*(X')\|_{2}$ is the $L_2$ sensitivity of $\Fl$ and $\cd = \sqrt{\log\left(\frac{2}{\sqrt{16\delta + 1} - 1}\right)}.$
For $\delta > 0$, note that the noise $\zl \sim N(0, \sigma_{\lambda}^2 I_{d})$, where $\sigma_{\lambda} = \left(\frac{\cd + \sqrt{\cd^2 + \varepsilon}}{\sqrt{2} \varepsilon}\right) \Delta_{\lambda}.$ Observe that by \cref{prop:sc sensitivity}, $\Dl \leq \frac{2(L+ \lambda R)}{\lambda}$ if $F \in \GG$ and $\Dl \leq \frac{2(L+ \lambda R)}{\lambda n}$ if $F \in \GGE.$ If one does not have access to $\Dl$ or does not wish to compute it, for fixed $\lambda > 0$ one can replace $\Dl$ in the algorithm by these upper bounds and the following results in this section all still hold. 

\vspace{0.2cm}
Our main result in this subsection is \cref{prop: convex ER}, in which we establish upper bounds on the excess risk of the regularized output perturbation method for $\GGE,$ of order $O\left(LR \sqrt{\pen}\right)$ for $\delta = 0$ and $\widetilde{O}\left(LR \sqrt{\frac{d^{1/2}}{\varepsilon n}}\right)$ when $\delta \in (0, \frac{1}{2})$.

\vspace{0.2cm}
First, by our choice of noise here in terms of the sensitivity of $\Fl,$ the same argument used to prove \cref{conceptual sc alg is private} establishes that this algorithm is differentially private: 
\begin{proposition}
\label[proposition]{prop: reg output pert is private}
The regularized output perturbation algorithm $\mathcal{A'}_{\lambda}$ described above is $(\varepsilon, \delta)$-differentially private for any $\lambda > 0$. 
\end{proposition}

Next, we establish excess risk upper bounds for the classes $\GG$ and $\GGE$. 
\begin{proposition}
\label[proposition]{prop: convex ER} 
Let  $F \in \GGE$. The regularized output perturbation method defined in \cref{regularized output pert - conceptual alg} with $\lambda$ prescribed below (and outputs denoted by $\wlapro$ and $\wlapro^{\delta}$ for $\delta = 0$ and $\delta \in (0, \frac{1}{2}),$ respectively) achieves each of the following excess risk bounds: \\
Set $\lambda = \frac{L}{R \sqrt{1 + \frac{\varepsilon n}{d}}}.$ \\
a) Assume $\delta = 0$ and $\pen \leq 1.$ Then \[
\ERLpro \leq 8.5 LR \sqrt{\pen}.\]
b) Assume $\delta \in (0, \frac{1}{2})$ and $\lpen \leq 1.$ 
Then \[
\ERDLpro \leq 8.5 LR \left(\lpen\right)^{1/2}.\]
Furthermore, if $\sqrt{\varepsilon} \lesssim \cd,$ then \[
\ERDLpro \lesssim LR \left(\frac{d^{1/4} \sqrt{\cd}}{\sqrt{\varepsilon n}}\right).\]
\end{proposition}

Note that the above assumptions on the parameters are necessary in each case for the excess risk upper bounds to be non-trivial (i.e. $< LR$). Also, as noted in the previous subsection, the condition $\sqrt{\varepsilon} \lesssim \cd$ is only (formally) restrictive (requiring $\delta \ll 1$) if $\varepsilon > 1.$ For practical privacy assurances, $\delta \ll 1$ is necessary, so this condition is easily satisfied. 

\vspace{0.2cm}
The private excess empirical risk lower bounds for $\GGE$ are: $\Omega\left(LR \min \left\{1, \pen \right\}\right)$ for $\delta = 0$, and  \newline $\Omega\left(LR \min \left\{1, \frac{\sqrt{d}}{\varepsilon n}\right\}\right)$ for $\delta = o(\frac{1}{n})$ \cite{bst14}. 
The exponential mechanism nearly achieves the lower bound for $\delta = 0$ \cite{bst14}. Thus, for $\delta = 0$, we see that our above upper bound is loose by a factor of roughly $O\left(\sqrt{\pen}\right)$ in the non-trivial regime ($d < \varepsilon n$). For $\delta > 0,$ the noisy stochastic subgradient descent method of \cite{bst14} nearly (up to a poly$\log(n)$ factor) achieves the lower bound for $\GGE,$ provided $\varepsilon \leq 2 \sqrt{\log\left(\frac{1}{\delta}\right)}.$ Hence for $\delta > 0$, our upper bound is loose by a factor of roughly $\frac{d^{1/4}}{\sqrt{\varepsilon n}}$ in the non-trivial regime $\sqrt{d} < \varepsilon n$. However, we show in \cref{Section 4: pop loss} that our algorithm again yields tighter population loss bounds than \cite{bst14} for convex Lipschitz loss, despite the looser empirical loss bounds. This fact, in conjunction with the simplicity and faster runtime of our output perturbation method (discussed in \cref{section 3: implment}), and the freedom to choose $\varepsilon > 2 \sqrt{\log\left(\frac{1}{\delta}\right)}$ and $\delta > 0~$ makes our method an appealing alternative for both $\GG$ and $\GGE.$

\vspace{0.2cm}
Next, we add the assumption of $\beta$-smoothess in order to derive tighter excess risk bounds.

\subsubsection{Excess risk bounds for smooth, convex, Lipschitz functions.}
\label[section]{subsection 2.6: convex smooth}
The function class we consider now is:
\[\JJE := \{F \in \GGE \, | \ F(\cdot, X) ~\text{is} ~\beta\text{-smooth} ~\text{for all} ~X \in \XX^{n} \}.\]
For such smooth, convex functions, we combine the techniques developed previously to derive excess risk bounds. More specifically, the regularization technique from \cref{subsection 2.4 convex lip} in conjunction with the descent lemma will yield  \cref{prop: convex smooth ER}, which gives excess risk upper bounds of order
$O\left(\blrt \left(\pen\right)^{2/3}\right)$ for $\JJE$ when $\delta = 0$ and $O\left(\blrt \left( \lpen\right)^{2/3}\right)$ when $\delta \in (0, 1/2)$.

\vspace{.2cm}
While our ERM results are not optimal for the convex class\footnote{Recall that the excess risk lower bounds for $\JJE$ are $\Omega\left(\min\left\{LR, LR \pen\right\}\right)$ and $\Omega\left(\min\left\{LR, LR \frac{\sqrt{d}}{\varepsilon n}\right\}\right)$ for $\delta = 0$ and $\delta > 0$, respectively \cite{bst14}.}, the speed and simplicity of our algorithm makes it the most practical choice for high-dimensional/large-scale problems, particularly when $\delta = 0$ and there are no other non-trivial algorithms that can be executed as fast as ours. We also are able to use the results proved in this section to get strong (in fact, the strongest we are aware of for $\delta = 0$) private population loss bounds in \cref{Section 4: pop loss}. 
\vspace{0.2cm}
For any $F \in \JJ$, consider the regularized objective $\Fl(w,X) = F(w,X) + \frac{\lambda}{2} \|w\|_2^2,$ as defined earlier. Note that $\Fl(\cdot, X)$ is $(\beta + \lambda)$-smooth in addition to being $(L + \lambda R)$-Lipschitz (on $\WW$) and $\lambda$-strongly convex. We execute the following conceptual algorithm: 
\begin{equation}
\label{regularized output pert - no proj}
\mathcal{A}_{\lambda}(X) = w_{\mathcal{A}, \lambda}(X) := \wls(X) + \zl.
\end{equation}
That is, we run output perturbation on the regularized objective (\textit{without projection}). We can then apply the results from \cref{subsection 2.3: sc Lip ER bounds} to $\Fl$ to obtain excess risk bounds for $F$. The method in \cref{regularized output pert - no proj} is differentially private by the same reasoning used to prove \cref{conceptual sc alg is private}. Our main results of this subsection are upper bounds on the excess risk of the regularized output perturbation method for $\JJE$ contained in the following proposition: 
\begin{proposition}
\label[proposition]{prop: convex smooth ER}
Let $F \in \JJE.$ The regularized output perturbation method defined in \cref{regularized output pert - no proj} with $\lambda$ prescribed below (and outputs denoted $\wla$ or $\wla
^{\delta}$ for $\delta = 0$ and $\delta \in (0, \frac{1}{2}),$ respectively)
achieves each of the following excess risk bounds:

a) Suppose $\delta = 0, ~\left(\pen\right)^2 \leq \frac{L}{R \beta}.$ Set $~\lambda = \left(\frac{\beta L^2}{R^2}\right)^{1/3} \left(\pen\right)^{2/3}.$
    Then, \[
    \ERL \leq 48.5 \beta^{1/3} L^{2/3} R^{4/3} \left(\pen\right)^{2/3}.\] \\
    
b) Suppose $\delta \in (0, \frac{1}{2}), ~\left(\lpen\right) ^2 \leq \frac{L}{R \beta}.$ Set $~\lambda = \left(\frac{\beta L^2}{R^2}\right)^{1/3} \left(\lpen\right)^{2/3}.$\\
    Then,
\[
\ERDL \leq 48.5 \beta^{1/3} L^{2/3} R^{4/3} \left(\lpen\right)^{2/3}.\]
In particular, if $\sqrt{\varepsilon} \lesssim \cd,$ then \[
\ERDL \lesssim \blrt \left(\frac{\sqrt{d} \cd}{\varepsilon n} \right)^{2/3}.\]
\end{proposition}

The hypotheses on given parameters are necessary to ensure that the upper bounds are non-trivial (i.e. $< LR$). 

\begin{remark}
\label[remark]{rem: conceptual convex smooth min J}
Since $\JJ \subseteq \GG$, we obtain the following upper bounds for $F \in \JJE$, by combining \cref{prop: convex smooth ER} with \cref{prop: convex ER}:  
$\ERL \lesssim \min \left\{LR, LR \left(\pen\right)^{1/2}, \beta^{1/3} L^{2/3} R^{4/3} \left(\pen\right)^{2/3}\right\}$ for $\delta = 0,$ 
and $~\ERDL \lesssim \min\left\{LR, LR \left(\frac{\sqrt{d} \cd}{\varepsilon}\right)^{1/2}, \beta^{1/3} L^{2/3} R^{4/3} \left(\frac{\sqrt{d} \cd}{\varepsilon}\right)^{2/3}\right\}$ for $\delta \in (0, \frac{1}{2})$ and $\varepsilon \lesssim \cd^2.$
\end{remark}

For $\JJE,$ there are a few existing differentially private algorithms (see \cref{table:ERM delta = 0} and \cref{table: ERM delta > 0}) besides the ones for $\GGE$ (e.g. \cite{bst14}, as discussed in \cref{subsection 2.4 convex lip}) which automatically apply. Among these methods that have practical implementation methods and reasonable runtime bounds (i.e. excluding the exponential mechanism) and apply to all functions in $\JJE$, our method gives the tightest excess risk bounds for $\JJE$ when $\delta = 0$. For $\delta = 0$, the objective perturbation mechanism of \cite{chaud2011} is differentially private only for a strict subclass of $\JJE$ and (namely, functions in $\JJE$ satisfying the restrictive rank-one hessian assumption $\rank(\nabla^2 f(w,x)) \leq 1$ everywhere) but gives excess risk $\widetilde{O}\left(LR \left(\pen\right)\right)$ for $\delta = 0,$ nearly matching the lower bound in the non-trivial parameter regime. Their method is purely conceptual, with no practical implementation process or runtime bounds. Excess risk bounds in \cite{chaud2011} are proved only for linear classifiers. \cite{kifer2012} establishes nearly optimal excess risk bounds of $\widetilde{O}\left(LR \left(\pen\right)\right)$ for rank-one hessian GLM functions in $\JJE$ when $\delta > 0$ via objective perturbation. They also do not provide a practical implementation procedure or runtime bounds. For $\delta > 0,$ the noisy SVRG method of \cite{wang2017} results in excess risk of $\widetilde{O}\left(LR \left(\frac{\sqrt{d} \log\left(\frac{1}{\delta}\right)}{\varepsilon n}\right)\right),$ making it the tightest known bound in this setting (and essentially matching the non-smooth risk bound of \cite{bst14}, but achieving it in faster runtime). However, their results require $\beta$ to be sufficiently large \footnote{This is stated explicitly as an assumption in \cite[Thoerem 4.4]{wang2017} and is also needed in order to guarantee their method is differentially private, by Theorem 4.3 in their work and the choices of algorithmic parameters if $\delta \ll 1.$}. Therefore, if $\beta$ is small, our method may be the most practical alternative for $\JJE, ~\delta > 0.$ 

\vspace{0.2cm}
The other non-trivial excess risk bounds worth mentioning are those of \cite{zhang2017}. They give excess empirical risk upper bounds of $O\left(\blrt \left(\pen\right)^{2/3}\right)$ and $O\left(\blrt \left(\frac{\sqrt{d} \log\left(\frac{1}{\delta}\right)}{\varepsilon n}\right)^{2/3}\right)$ for $\delta = 0$ and $\delta > 0$, respectively. Enhancing their upper bound for $\JJE$ by running the trivial algorithm whenever it does better than their algorithm, we match their upper bound for $\delta = 0$ whenever either $d \geq \varepsilon n$ or $\pen \leq \frac{L^2}{\beta^2 R^8}$ (when either the first or third argument in the $\min\{...\}$ in \cref{rem: conceptual convex smooth min J} above is active). In the complementary case, when $1 > \pen >  \frac{L^2}{\beta^2 R^8},$ our non-smooth upper bound (from \cref{prop: convex ER}) results in smaller excess risk than \cite{zhang2017}.
Furthermore, we do not require differentiability of $F$ in order to achieve the upper bound in this regime (as the bound in \cref{prop: convex ER} was obtained by merely assuming Lipschitzness and convexity). Similar remarks apply for $\delta > 0.$
Thus, our excess risk upper bounds dominate \cite{zhang2017} for all function classes where both of our algorithms apply, but we emphasize that our algorithm also applies to a much wider set of functions (non-smooth and non-ERM) and parameters ($\varepsilon > 1$) than theirs. Furthermore, we will see that our algorithm offers significant runtime advantages over \cite{zhang2017} for $\JJE$, potentially by even a bigger margin than for $\FFE$ by using an accelerated stochastic gradient method instead of gradient descent (see \cref{table:ERM delta = 0}, \cref{table: ERM delta > 0}, and \cref{subsection 3.4: smooth convex lip imp} for details).

\vspace{.2cm}
In the next section, we describe how to efficiently implement our method in practice.

\section{Efficient Implementation of Output Perturbation}
\label{section 3: implment}

In this section, we show that our conceptual output perturbation algorithm can be treated as a black box method: it can be efficiently implemented using any non-private optimization algorithm the user desires, transforming that method into a differentially private one without sacrificing runtime/computational cost. In particular, by choosing fast gradient/subgradient methods, we can achieve the conceptual upper bounds on excess risk from Section 2 in shorter runtime than any competing $\varepsilon$-differentially private algorithms that we are aware of. The idea and techniques we use for practical implementation of outptut perturbation are very similar in spirit to the conceptual algorithm of Section 2. At a high level, our method consists simply of finding an \textit{approximate} minimizer of $F(\cdot, X)$ (or $\Fl(\cdot, X)$ in the non-strongly convex case) instead of an exact minimizer, and then adding noise, calibrated according to the sensitivity of the last iterate $w_{T}$ (instead of $\ws$) to ensure privacy. Our \ul{main result} in this section is Corollary~\ref{cor: katyusha sc smooth}, which provides \textit{near-optimal excess risk in near-linear-time}. We begin by studying strongly convex, Lipschitz loss functions.

\subsection{Strongly convex, Lipschitz functions}
\label{subsection: 3.1}
Our general implementation method is described in \cref{alg: black box sc}. Note that it accepts as input any non-private optimization algorithm and transforms it into a differentially private one, by adding noise to the approximate minimizer $w_{T}.$ As a result, we are able to obtain runtimes as fast as non-private optimization methods while also achieving the conceptual excess risk bounds from Section 2 (by choosing $\alpha$ appropriately). 
\begin{algorithm}[ht!]
\caption{Black Box Output Perturbation Algorithm for $\FF$ and $\HH$}
\label[algorithm]{alg: black box sc}
\begin{algorithmic}[1]
\Require~ Number of data points $n \in \mathbb{N},$ dimension $d \in \mathbb{N}$ of data, non-private (possibly randomized) optimization method $\mathcal{M}$, privacy parameters $\varepsilon > 0, \delta \geq 0$, data universe $\XX,$ data set $X \in \XX^{n}$, function $F \in \FF$ with $L_{2}$-sensitivity $\Delta_{F} \geq 0$, accuracy parameter $\alpha > 0$ with corresponding iteration number $T = T(\alpha) \in \mathbb{N}$ (such that $\mathbb{E}[F(w_{T}(X), X) - F(\ws(X), X)] \leq \alpha$).
 \State  Run $\MM$ for $T = T(\alpha)$ iterations to ensure $\mathbb{E}F(w_{T}(X), X) - F(\ws, X) \leq \alpha$. 
 \State Add noise to ensure privacy: $\wpr:= w_{T} + \hat{z},$ where the probability density function $p(\hat{z})$ of $\hat{z}$ is proportional to 
$\begin{cases} 
\exp\left(-\frac{\varepsilon \|\hat{z}\|_{2}}{\Delta_{F} + 2\sqrt{\frac{2 \alpha}{\mu}}}\right) &\mbox{if } \delta = 0 \\
\exp\left(-\frac{\varepsilon^2 \|\hat{z}\|_{2}^2}{\left(\Delta_{F} + 2\sqrt{\frac{2 \alpha}{\mu}}\right)^2 \left(\cd + \sqrt{\cd^2 + \varepsilon}\right)^2} \right) &\mbox{if } \delta > 0
\end{cases}.$ \\
\Return $\wpr.$
\end{algorithmic}
\end{algorithm}

First, we establish the privacy guarantee:
\begin{proposition}
\label[proposition]{prop: sc imp priv}
\cref{alg: black box sc} is $(\varepsilon, \delta)$-differentially private. Moreover, the algorithm $\mathcal{A'}(X) = w_{\mathcal{A}'}(X) := \Pi_{\WW}(\wpr(X))$ is $(\varepsilon, \delta)$-differentially private, where $\wpr(X)$ denotes the output of \cref{alg: black box sc}.
\end{proposition}

Next, we establish excess risk upper bounds for the Black Box Algorithm on functions $F \in \FF,$ which depend on the inputted accuracy parameter, $\alpha.$

\begin{theorem}
\label{thm: black box sc lip}
Suppose $F \in \FFE$. Run \cref{alg: black box sc} with arbitrary inputs and denote $\mathcal{A}'(X) = w_{\mathcal{A}'}(X) := \Pi_{\WW}(\wpr(X)),$ where $\wpr(X)$ is the output of \cref{alg: black box sc}. 

a) Let $\delta = 0, ~\pen \leq 1$. 
Then 
\[
\ER \leq 2\sqrt{2}\left(\lmu \left(\frac{1}{n}\right) + L \almu\right)\left(\pe\right) + \alpha.\]
In particular, setting $\alpha = \frac{L^2}{\mu n} \min \{\frac{1}{n}, \pe \}$ gives \[
\ERpro \leq 9 \lmu \left(\pen\right).\]

b) Let $\delta > 0, ~\lpen \leq 1.$ Then \[
\ERpro \leq 2 \left(\frac{L^2}{\mu n} + L\almu\right)\left(\lpe\right) + \alpha.\]
In particular, setting $\alpha = \frac{L^2}{\mu n} \min \left\{\frac{1}{n}, \lpe\right\}$ gives \[
\ERDpro \leq 6 \lmu \left(\lpen\right).\]
\end{theorem}

The assumptions on parameters in each case (e.g. $\pe \leq 1$ or $\lpe \leq 1$ for non-ERM) are necessary to ensure that the resulting upper bound is non-trivial. Note that with the choices of $\alpha$ prescribed, the corresponding conceptual excess risk upper bounds from \cref{cor: sc upper bound} are achieved. Furthermore, notice that it is not necessary to have a tight estimate of $\Delta_{F}$ in order to achieve the excess risk bounds in \cref{thm: black box sc lip}. Indeed, in \cref{alg: black box sc}, one can replace $\Delta_{F}$ by $\frac{2L}{\mu}$ if $F \in \FF$ or $\frac{2L}{\mu n}$ if $F \in \FFE$ (c.f. \cref{prop:sc sensitivity}) and the excess risk guarantees of \cref{thm: black box sc lip} still hold. This is clear from the proof of \cref{th: smooth sc implement} (see \cref{app: proofs of sec 3.1}). 
\vspace{0.2cm}

\vspace{0.2cm}
Notice that \cref{alg: black box sc} relies on an iterative algorithm for obtaining $w_T.$ Next, we specialize our result to the case where a stochastic subgradient method (\cref{alg: stochastic subgrad}) is used to obtain $w_T$ in \cref{alg: black box sc}. Running the stochastic subgradient method with $\eta_{t} = \frac{2}{\mu (t+1)}$ produces a point $\widehat{w}_{T}$ such that $\mathbb{E}F(\widehat{w}_{T}, X) - F(\ws(X), X) \leq \alpha$ in $T = \frac{2 L^2}{\mu \alpha}$ stochastic gradient evaluations \cite[Theorem 6.2]{bubeck}. Since each iteration amounts to just one gradient evaluation of $f(w, x_{i})$ (with runtime $d$), the resulting runtime of the full method is $O(dT).$

\begin{algorithm}[h!]
\caption{Stochastic Subgradient Method (for $\FFE$)}
\label{alg: stochastic subgrad}
\begin{algorithmic}[1]
\Require~$F(w,X) = \frac{1}{n}\sum_{i=1}^n f(w,x_{i}) \in \FFE$, data universe $\XX,$ dataset $X = \{x_{i}\}_{i=1}^n \in \XX^n$, iteration number $T \in \mathbb{N}$, step sizes $\{\eta_{t}\} ~(t \in [T])$.   
\State Choose any point $w_0 \in \WW.$
\For {$t \in [T]$} 
 \State Choose $x_{i} \in X$ uniformly at random, with replacement.
 \State $w_{t} = \Pi_{\WW} [w_{t-1} - \eta_{t-1} g(w_{t-1}, x_{i})],$ where $g(w, x_{i}) \in \partial f(w, x_{i})$ is a subgradient of $f$. 
\EndFor \\
\Return $\widehat{w}_{T} = \sum_{t = 1}^{T} \frac{2t}{T(T+1)}w_{t}.$
\end{algorithmic}
\end{algorithm}

\begin{corollary}
\label[corollary]{cor: sc lip SGD}
Let $F \in \FFE$. Run \cref{alg: black box sc} with Stochastic Subgradient Method (\cref{alg: stochastic subgrad}) with step size $\eta_{t} = \frac{2}{\mu (t+1)}$ as the non-private $\MM$. 

a) Let $\delta = 0, \pen \leq 1.$ Then setting $\alpha = \frac{L^2}{\mu n} \min \left\{\frac{1}{n}, \frac{d}{\varepsilon}\right\}$ gives \[
\ERpro \leq 9 \frac{L^2}{\mu} \pen
\]
in $T = 2n \max \left\{n, \frac{\varepsilon}{d} \right\}$ stochastic gradient evaluations. The resulting runtime is $O(\max \{dn^2, \varepsilon n\}).$\\
b) Let $\delta > 0, \lpen \leq 1.$ Then setting $\alpha = \frac{L^2}{\mu n} \min \left \{\frac{1}{n}, \lpe\right\}$ gives \[
\ERDpro \leq 6 \lmu \lpen\]
in $T = 2n \max \left\{n, \frac{\varepsilon}{\sqrt{d}(2\cd + \sqrt{\varepsilon})}\right\}$ stochastic gradient evaluations. The resulting runtime is $O\left(\max \left \{dn^2, n \sqrt{d \varepsilon} \right\}\right).$
\end{corollary}

In the ERM setting, the only competing algorithms we are aware of for non-smooth objective ($\FFE$) are the exponential sampling + localization and noisy SGD methods of \cite{bst14}. Consider the ERM case with $\delta = 0$: in \cite{bst14}, they show that exponential sampling + localization achieves excess risk bounds of $O\left(\frac{L^2}{\mu} \left(\pen\right)^2 \log(n) \right)$. While these bounds are tighter than ours in most parameter regimes, their runtime is generally much slower. An implementation method for the output perturbation/localization step of their algorithm is not provided, but the second step of their algorithm alone (consisting of the exponential mechanism initialized with a smaller set $\mathcal{W}_0 \subset \WW$ obtained from localization) has runtime $\widetilde{O}\left(\frac{L^2}{\mu^2} \frac{ n d^6}{\varepsilon^2} \max 
\left\{1, \frac{L}{\mu} \right\}\right),$ \textit{provided $\mathcal{W}_0$ is in isotropic position}; in the worst case (if $\mathcal{W}_0$ is non-isotropic), then runtime would slow by an additional factor of $O(d^3)$ \cite{bst14}. Implementing the first step (output perturbation/localization) through methods similar to our own would increase the runtime in \cite{bst14} by an additional additive term comparable to our total runtime. On the other hand, our output perturbation method has much smaller runtime of $O(d n^2)$ when stochastic subgradient method is used.
Furthermore, our resulting population loss bounds match (in fact, beat by a logarithmic factor) the population loss bounds of \cite{bst14} for this class (see \cref{Section 4: pop loss}). Thus, our method is a more practical choice in most parameter regimes. 

\vspace{0.2cm}
For $\delta > 0,$ the noisy stochastic subgradient algorithm of \cite{bst14} achieves runtime of $O(n^2 d),$ while achieving excess empirical risk of $\widetilde{O}\left(\frac{L^2}{\mu} \frac{d}{(\varepsilon n)^2}\right),$ which is better for most parameter regimes than ours. However, their method requires $\varepsilon \leq 2 \log\left(\frac{1}{\delta}\right)$ in order to be differentially private, whereas our method is private for all $\varepsilon > 0,$ giving more flexibility to further reduce excess risk (while sacrificing some privacy) by choosing larger $\varepsilon.$ Furthermore, we will show in \cref{Section 4: pop loss} that we can obtain tighter \textit{population loss} bounds than those established in \cite{bst14}, suggesting that our algorithm is a competitive alternative, even in the $\delta > 0$ setting. 

\vspace{0.2cm}
Next we analyze the Black Box \cref{alg: black box sc} applied to the subclass of $\beta$-smooth functions $\HH \subset \FF$. 
\paragraph{Smooth, strongly convex Lipschitz functions.}
We begin by stating an analogue of \cref{thm: black box sc lip} for $\beta$-smooth loss functions.

\begin{theorem}
\label{th: smooth sc implement}
Suppose $F \in \HHE.$ Run \cref{alg: black box sc} on $F \in \HH$ with arbitrary inputs. \\
a) Let $\delta = 0$ and assume $\pen \leq 1.$ \\
Then \[
\ER \leq 4 \beta \left(\frac{L}{\mu n} + \almu\right)^2\left(\pe\right)^2 + \alpha.\]
In particular, setting $\alpha = \frac{L^2}{\mu n^2} \min \left\{\kappa \left(\pe\right)^2, 1 \right\}$ gives \[
\ER \leq 26 \frac{L^2}{\mu} \kappa \left(\pen\right)^2.\]
b) Let $\delta\in (0,\frac{1}{2})$ and assume $\lpen \leq 1$. Then \[
\ERD \leq  2 \beta \left(\frac{L}{\mu n} + \almu\right)^2 \left(\lpe\right)^2 + \alpha.\]
In particular, setting $\alpha = \frac{L^2}{\mu n^2} \min \left\{\kappa \left(\lpe\right)^2, 1 \right\}$ gives \[
\ERD \leq 13.5 \lmu \kappa \left(\lpen\right)^2.\]
\end{theorem}

Note that the prescribed choices of $\alpha$ ensure the excess risk upper bounds from \cref{cor: smooth sc fund ER bounds} are attained (up to constant factors). 

\vspace{0.2cm}
With $\alpha$ and $\mathcal{M}$ chosen appropriately in \cref{alg: black box sc}, we can use \cref{th: smooth sc implement} to ensure that the excess risk bounds for $\HH$ and $\HHE$ (\cref{cor: smooth sc fund ER bounds}) are attained in as little runtime as any non-private optimization method. In particular, plugging Nesterov's Accelerated Gradient Descent (AGD) method \cite{nesty} or, in the ERM setting, an accelerated stochastic method such as Catalyst \cite{lmh15} or Katyusha \cite{az16} into \cref{alg: black box sc} yields \ul{the fastest differentially private algorithm we are aware of}, while also achieving \textit{near-optimal} (up to $\kappa$ factor) excess empirical risk for $\HHE$. 

\begin{algorithm}[h!]
\caption{Katyusha}
\label{alg: Katyusha}
\begin{algorithmic}[1]
\Require~$F(w,X) = \frac{1}{n}\sum_{i=1}^n g(w, x_{i}) + \psi(w) := G(w, X) + \psi(w),$ where each $g(\cdot, x_{i})$ is convex, $\beta$-smooth and $\psi$ is $\mu$-strongly convex, data universe $\XX,$ dataset $X = \{x_{i}\}_{i=1}^n \in \XX^n$, iteration number $T \in \mathbb{N}$ (corresponding to accuracy parameter $\alpha$).
\State $m := 2n.$
\State $\tau_2 := \frac{1}{2}, \tau_1 := \min\{\frac{\sqrt{m \mu}}{\sqrt{3L}}, \frac{1}{2} \}, \gamma := \frac{1}{3 \tau_1 \beta}.$
\State $y_0, z_0, \widetilde{w}_0 := w_0.$
\For {$t = 0, 1, ... T-1$}
    \State $\mu_{t} = \nabla G(\widetilde{w}_{t}, X).$ 
    \For {$j = 0, 1, ... m-1$}
        \State $k:= (tm) + j.$
        \State $w_{k+1} = \tau_1 z_{k} + \tau_2 \widetilde{w}_{t} + (1 - \tau_1 - \tau_2)y_{k}.$
        \State $\widetilde{\nabla}_{k+1} = \mu_{t} + \nabla g(w_{k+1}, x_{i}) - \nabla g(\widetilde{w_{s}}, x_{i})$ for $i$ chosen uniformly from $[n].$
        \State $z_{k+1} := \argmin_{z \in \mathbb{R}^d}[\frac{1}{2 \gamma}\|z - z_{k}\|_2^2 + \langle \widetilde{\nabla}_{k+1}, z\rangle + \psi(z)].$
        \State $y_{k+1} = \argmin_{y \in \mathbb{R}^d}[\frac{3 \beta}{2}\|y - w_{k+1}\|_2^2 + \widetilde{\nabla}_{k+1}, y\rangle + \psi(y)].$
        \EndFor
\State $\widetilde{w}_{t+1} = (\sum_{j=0}^{m-1}(1 + \gamma \mu)^{j})^{-1} (\sum_{j=0}^{m-1} (1 + \gamma \mu)^{j} y_{tm + j + 1}).$
\EndFor
\newline
\Return $\widetilde{w}_{T}.$
\end{algorithmic}
\end{algorithm} 

\vspace{0.2cm}
For $F \in \HHE,$ we input an accelerated stochastic method as $\MM$ in \cref{alg: black box sc}. In particular, the Katyusha method of \cite{az16}, outlined  in \cref{alg: Katyusha}, provides the following guarantee: 

\begin{theorem} \cite[Theorem 2.1 simplified]{az16}
\label{thm 2.1: katyusha az}
Let $F: \mathbb{R}^d \to \mathbb{R}, ~F(w) = \frac{1}{n}\sum_{i=1}^n f_{i}(w) = \frac{1}{n}\sum_{i=1}^n g_{i}(w) + \psi(w),$ where each $g_{i}$ is convex and $\beta$-smooth, $\psi$ is $\mu$-strongly convex, and $\kappa := \frac{\beta}{\mu}.$ Then running \cref{alg: Katyusha} for $T = O\left(\left(n + \sqrt{n \kappa}\right)\log\left(\frac{F(w_0) - F(\ws)}{\alpha}\right) \right)$ stochastic gradient iterations returns a point $\widetilde{w}_{T}\in \mathbb{R}^d$ such that $F(\widetilde{w}_{T}) - F(\ws) \leq \alpha.$
\end{theorem}

Clearly, given any $F \in \HHE$ and any $X \in \XX^n$, we can write $F(w, X) = \frac{1}{n}\sum_{i=1}^n f(w,x_{i}) = \frac{1}{n}\sum_{i=1}^n g(w,x_{i}) + \psi(w),$ where $\psi(w) := \frac{\mu}{2} \|w \|_2^2$ 
is $\mu$-strongly convex, $g(\cdot, x_{i}) = f(\cdot, x_{i}) - \psi(\cdot)$ is convex and $(\beta - \mu)$-smooth (hence $\beta$-smooth) for all $x_{i} \in X.$ 
Combining \cref{thm 2.1: katyusha az} with \cref{th: smooth sc implement}, we obtain the following: 

\begin{corollary}
\label[corollary]{cor: katyusha sc smooth}
Let $F \in \HHE.$ Take $\MM$ to be Katyusha and $T = O\left(\left(n + \sqrt{n \kappa}\right)\log\left(\frac{LR}{\alpha}\right) \right)$ as inputs in \cref{alg: black box sc}. \\
a) Let $\delta = 0, \pen \leq 1.$ Then setting $\alpha = \frac{L^2}{\mu n^2} \min \left\{\kappa \left(\pe\right)^2, 1 \right\}$ gives \[
~\ER \leq 26 \frac{L^2}{\mu} \kappa \left(\pen\right)^2\]
in $T = \widetilde{O}(n + \sqrt{n \kappa})$ stochastic gradient evaluations. The resulting runtime is \[
\widetilde{O}(d(n + \sqrt{n \kappa})).\] \\
b) Let $\delta > 0, \lpen \leq 1.$ Then setting $\alpha = \frac{L^2}{\mu n^2} \min \left\{\kappa \left(\lpe\right)^2, 1 \right\}$ gives \[
\ERD \leq 13.5 \lmu \kappa \left(\lpen\right)^2\]
in $T = \widetilde{O}(n + \sqrt{n \kappa})$ stochastic gradient evaluations. The resulting runtime is \[
\widetilde{O}(d(n + \sqrt{n \kappa})).\] 
\end{corollary}

\begin{remark}
\label[remark]{rem: min smooth sc ERM}
Since $\HHE \subseteq \FFE$, by \cref{cor: sc upper bound}, we obtain the following upper bounds on the excess risk of functions in $\HHE$: 
$\ER \lesssim \min \{LR,  \frac{L^2}{\mu} \pen, \kappa \frac{L^2}{\mu} \left(\pen\right)^2 \}$ for $\delta = 0$; and $\ERD \lesssim \min \{LR, \frac{L^2}{\mu} \frac{\sqrt{d}(c + \sqrt{\varepsilon})}{\varepsilon n}, \kappa \frac{L^2}{\mu} \frac{p (c + \sqrt{\varepsilon})^2}{(\varepsilon n)^2} \}$ for $\delta > 0$ for $\delta > 0.$ Furthermore, the runtime for these enhanced upper bounds are (up to log terms) as stated above in \cref{cor: katyusha sc smooth}, since
obtaining the excess risk bounds of \cref{cor: sc lip SGD} via Katyusha only affects the choice of $\alpha,$ which is inside of a $\log(\cdot)$ in the expression for the iteration complexity of Katyusha. 
\end{remark}

As noted in \cref{subsection 2.5: smooth sc} (and as can be seen in \cref{table:ERM delta = 0} and \cref{table: ERM delta > 0}), there are several existing differentially private methods for $\HHE.$ However, output perturbation, when implemented with Katyusha (or another accelerated stochastic method such as Catalyst) achieves the \textit{fastest runtime of all of these}.

\vspace{.2cm}
In addition to faster runtime, our method is more widely applicable than many other methods (e.g. for $\delta > 0,$ our method is the only one that applies for arbitrary $\varepsilon > 0$, as can be seen in \cref{table: ERM delta > 0}). Let us discuss these differences in a bit more detail. Consider $\delta = 0.$ First, the objective perturbation method achieves excess empirical risk of $\widetilde{O}\left(\lmu \left(\pen\right)^2\right),$ but requires twice differentiability and (crucially) $\rank(\nabla_{w}^2 f(w,x)) \leq 1$ everywhere \cite{chaud2011}. Formally, their excess risk bounds are only proved for linear classifiers, but it does not seem that this restriction is necessary. In addition, it is a purely conceptual method: no practical implementation procedure or runtime estimate is provided. On the other hand, our output perturbation method does not have any of these restrictions and can be easily implemented in very fast runtime $\widetilde{O}\left(d(n + \sqrt{n \kappa})\right)$ via Katyusha. 
Recall that the excess empirical risk upper bound for $\FFE$ in \cite{bst14} is $O\left(\lmu \log(n) \left(\pen\right)^2\right)$ for $\delta = 0,$ achieved by the exponential mechanism with localization. Smoothness does not improve this bound to the best of our knowledge. Assuming $\pen < 1,$ so that our bound and that of \cite{bst14} are both non-trivial, our upper bound for $\delta = 0$ (\cref{rem: min H}) is tighter when $\kappa <  \log(n)$ or when $1 \lesssim \pen \log(n)$. In addition, our method is much simpler and runtime is much faster. 
The last method we mention for $\delta = 0$ is the gradient descent-based output perturbation method of \cite{zhang2017}. They give excess risk bounds for $\HHE$
that match those given in \cref{cor: smooth sc fund ER bounds}. In light of \cref{rem: min H}, however, there are some parameter regimes where our excess risk bounds are smaller than theirs. First, ignoring constants, we see that when $d > \varepsilon n$ our upper bound (\cref{rem: min H}) is the trivial upper bound $LR$. When $d < \varepsilon n$ so that $LR > \min\left\{\lmu \pen ,\kappa \frac{L^2}{\mu} (\frac{d}{\varepsilon n})^2\right\}$, there are two subcases: if $\frac{d}{\varepsilon n} \leq \frac{1}{\kappa}$, then the resulting upper bound is $\kappa \frac{L^2}{\mu} (\frac{d}{\varepsilon n})^2.$ Finally, when $\kappa$ is very large so that $\frac{1}{\kappa} < \frac{d}{\varepsilon n} \leq 1$, then our upper bound becomes $\frac{d}{\varepsilon n}\frac{L^2}{\mu}$. On the other hand, \cite{zhang2017} has excess risk $O\left(\kappa (\frac{d}{\varepsilon n})^2 \frac{L^2}{\mu}\right)$ for $\delta = 0.$ If we trivially enhance their algorithm to instead perform the trivial algorithm whenever it yields smaller excess risk ($LR$), then our upper bound matches theirs in the first two cases mentioned above and is superior in the last case. Furthermore, the last case (very large $\kappa$) is also where we dominate in terms of runtime by allowing for stochastic acceleration. By comparison, their runtime is $\widetilde{O}(nd \kappa).$

\vspace{0.1cm}
Similar remarks apply when $\delta > 0$. If $\varepsilon > 1,$ then the method of \cite{zhang2017} is not even differentially private, so our excess risk dominates theirs trivially. Assume $\cd \lesssim \sqrt{\varepsilon}$ (this is guaranteed automatically when $\varepsilon \leq 1$ and $\delta \in (0, \frac{1}{2})$). In this case, our upper bound is non-trivial iff $\frac{\cd \sqrt{d}}{\varepsilon n} < 1,$ in which case our upper bound reduces to $O\left(\frac{L^2}{\mu} \left(\frac{\sqrt{d} \cd }{\varepsilon n}\right) \min \left\{1, \kappa \left(\frac{\sqrt{d} \cd }{\varepsilon n}\right)\right\}\right).$ On the other hand, the upper bound of \cite{zhang2017}, which is differentially private only when $\varepsilon \leq 1$, is $O\left(\kappa \frac{L^2}{\mu} \frac{d \log\left(\frac{1}{\delta}\right)}{\varepsilon^2 n^2}\right).$ By \cite{zhao2019}, our logarithmic factor $\cd \leq \sqrt{\log\left(\frac{1}{\delta}\right)}$, so our bound is at least as good as  \cite{zhang2017} when $\kappa \frac{\cd \sqrt{d}}{\varepsilon n} < 1.$ On the other hand, when $1 \leq \kappa \frac{\cd \sqrt{d}}{\varepsilon n}$, our upper bound is also tighter since then $\frac{L^2}{\mu} (\frac{\cd \sqrt{d}}{\varepsilon n}) < \kappa \frac{L^2}{\mu} \frac{d \log\left(\frac{1}{\delta}\right)}{\varepsilon^2 n^2}.$
 For $\delta > 0$, the noisy SGD algorithm of \cite{bst14} yields excess risk $O(\lmu \frac{d}{(\varepsilon n)^2} \log\left(\frac{1}{\delta}\right) \log(\frac{n}{\delta})^2).$ Assuming $\lpen \lesssim 1$ so that both of our bounds are non-trivial, our upper bound is tighter when $\kappa \lesssim \log(\frac{n}{\delta})^2.$ Furthermore, our algorithm allows for arbitrary $\varepsilon > 0,$ whereas their results only apply for $\varepsilon \leq 2 \sqrt{\log\left(\frac{1}{\delta}\right)}.$ The noisy SVRG method of \cite{wang2017} is faster than the method of \cite{bst14},  with runtime of $\widetilde{O}\left(d(n + \sqrt{\kappa}n)\right),$ while achieving comparable excess risk. However, their method requires small $\varepsilon = O\left(\frac{\kappa}{n^2}\right)$ to be differentially private, which is quite restrictive. By contrast, our method provides the flexibility for arbitrarily large $\varepsilon,$ providing a full continuum of accuracy-privacy tradeoff options. 
 Finally, objective perturbation also applies for $\delta > 0$ \cite{kifer2012}, but again comes with additional restrictions. In addition to rank-one hessian and GLM assumption, they also require $n \geq d^2$ and $\frac{\sqrt{d}}{\varepsilon n} = \widetilde{O}(1).$ Moreover, no implementation procedure or runtime bounds are given. 
 
 \vspace{0.2cm}
 In Appendix~\ref{subsection 3.3}, we use regularization to efficiently implement our conceptual regularized output perturbation method when the objective function is merely convex.

\section{From Empirical Loss to Population Loss}
\label{Section 4: pop loss}

Recent works \cite{bst14,bft19,fkt20} have established tight bounds on the expected population loss of $(\varepsilon, \delta)$-private stochastic optimization algorithms for $\delta > 0$. However, the important case of $\delta = 0,$ which provides the strongest privacy guarantee, has largely been overlooked, except for \cite{bst14} and (in the smooth, strongly convex case) \cite{zhang2017}.
\ul{We show that our simple output perturbation algorithm results in expected excess population loss bounds that improve the expected excess population loss bounds in \mbox{\cite{bst14}} and can be executed in substantially less runtime.} Moreover, for the smooth, strongly convex function class where \cite{zhang2017} gives expected excess population loss bounds, our bounds are tighter (or matching, in certain parameter regimes) and are achieved in less runtime. Our method also works for $\delta > 0,$ where it serves as a simple, flexible, efficient alternative which, unlike the other algorithms proposed so far, also allows for arbitrary $\varepsilon > 0.$ The \ul{main result} of this section is stated in Corollary~\ref{rem: smooth sc pop loss runtime katyusha}.

\vspace{0.2cm}
\noindent \textbf{Change of Notations.} In this section, we change some notation for further clarity and introduce some new definitions. Let the data have some (unknown) distribution $x \sim \mathcal{D}.$ We are given a sample (i.i.d.) data set $X \sim \DD^{n}.$
We will denote the normalized empirical loss $\hf(w,X):= \frac{1}{n} \sum_{i=1}^n f(w,x_{i})$ and the \textbf{population loss} $F(w, \DD) := \mathbb{E}_{x \sim \DD} f(w,x).$ 

\vspace{.1cm}
We want to understand how output perturbation performs in privately minimizing the \textbf{expected excess population loss}, defined at a point $w \in \mathbb{R}^d$ by 
\begin{equation}
\label{excess pop loss}
    F(w, \DD) - \min_{w \in \mathbb{R}^d}F(w,\DD) = F(w, \DD) - F(\ws(\DD), \DD),
\end{equation}
where we denote a parameter that minimizes the population loss by $\ws(\DD).$ That is, \[
\ws(\DD) \in \argmin_{w \in \mathbb{R}^d} F(w, \DD) 
\]
To avoid any ambiguity, in this section, we will denote the minimizer of the \textit{empirical} loss for a given data set $X$ by $\widehat{w}(X).$  Denote the output of the differentially private output perturbation algorithm by $\wpr(X) = \widehat{w}(X) + z$ where z is the same noise we used before. Our approach is to extend our ERM results to population loss bounds using the algorithmic stability-based techniques introduced in \cite{be02} and applied by \cite{hardt15} and \cite{bft19}.

\vspace{0.2cm}
We begin with a definition and lemma from \cite{be02} that will be very useful for us.
\begin{definition}
(Uniform stability) Let $\alpha > 0.$ A (possibly randomized) algorithm $\mathcal{A}: \XX^n \to \WW$ is $\alpha$-uniformly stable (w.r.t loss function $f$) if for any pair of data sets $X, X' \in \XX^n$ differing by at most one point (i.e. $|X \Delta X'| \leq 2$), we have \[
\sup_{x \in \XX} \mathbb{E}_{\mathcal{A}}[f(\mathcal{A}(X), x) - f(\mathcal{A}(X'), x)] \leq \alpha.
\]
\end{definition}
Informally, the definition says that a uniformly stable algorithm is one that for similar datasets, produces outputs with similar loss values on average.  

\begin{lemma} (\cite{be02})
\label[lemma]{lem: stability}
Let $\mathcal{A}: \XX^n \to \mathbb{R}^d$ be an $\alpha$-uniformly stable algorithm w.r.t. loss function $f$ and let $X \sim \DD^n.$ Then \[
\mathbb{E}_{X \sim \DD^n, \mathcal{A}}[F(\mathcal{A}(X), \DD) - \hf(\mathcal{A}(X), X)] \leq \alpha. 
\]
\end{lemma}

Combining \cref{lem: stability} with our excess empirical risk bounds from Section 2 for each ERM function class will enable us to upper bound the population loss when running output perturbation on the empirical loss $\widehat{F}$. We begin with the strongly convex, Lipschitz class. 

\subsection{Strongly convex, Lipschitz functions}
\label{subsection 4.1}

\begin{proposition}
\label[proposition]{prop: sc lip pop loss}
Let $f(\cdot, x)$ be $L$-Lipschitz on $\WW$ and $\mu$-strongly convex for all $x \in \XX,$ let $\varepsilon > 0$ and $\delta \in [0, \frac{1}{2}).$ Let $X$ be a data set of size $n$ drawn i.i.d. according to $\mathcal{D}.$ Run the conceptual output perturbation algorithm (\cref{eq: conceptual out pert proj}) on the empirical loss $\widehat{F}(w, X)$ to obtain $\wprpro(X).$\\
a) If $\delta = 0$ and $\left(\pen\right) \leq 1$, then  \[
\ERPpro \leq \frac{2 L^2}{\mu}\left(\frac{1}{n} + \pen\right).\]\\
b) If $\delta \in \left(0, \frac{1}{2}\right)$ and $\left(\lpen\right)\leq 1,$ then \[
\ERDPpro \leq \frac{2 L^2}{\mu}\left(\frac{1}{n} + \lpen\right).\]
\end{proposition}

In order to prove \cref{prop: sc lip pop loss}, we will need to estimate the uniform stability of the output perturbation algorithm:
\begin{lemma} 
\label[lemma]{lem: stability of sc lip f}
Let $f(\cdot, x)$ be $L$-Lipschitz on $\WW$ and $\mu$-strongly convex for all $x \in \XX$. Then the conceptual output perturbation algorithms $\mathcal{A}$ (\cref{eq: conceputal output pert}) and $\mathcal{A'}$ (\cref{eq: conceptual out pert proj}) are $\alpha$-uniformly stable for $\alpha = \frac{2L^2}{\mu n}.$
\end{lemma}
This lemma is proved in \cref{app: Sec 4 proofs}. Combining this with \cref{lem: stability} enables us to bridge the gap between our empirical loss results and the population loss bound claimed above. See \cref{app: Sec 4 proofs} for details.

\vspace{0.2cm}
By using stochastic subgradient descent (SGD, \cref{alg: stochastic subgrad}) to approximately minimize the empirical objective $\widehat{F},$ as described in \cref{subsection: 3.1}, we can attain the expected excess population loss bounds stated in \cref{prop: sc lip pop loss} efficiently:

\begin{corollary}
\label[corollary]{rem: sc lip pop loss runtime}
Let $f(\cdot, x)$ be $L$-Lipschitz on $\WW$ and $\mu$-strongly convex for all $x \in \XX,$ let $\varepsilon > 0$ and $\delta \in [0, \frac{1}{2}).$ Let $X$ be a data set of size $n$ drawn i.i.d. according to $\mathcal{D}.$ Run $\mathcal{A}=$~\cref{alg: black box sc} on $\hf$ with $\mathcal{M}$ as the stochastic subgradient method (\cref{alg: stochastic subgrad}) with step sizes $\eta_{t} = \frac{2}{\mu (t + 1)}$ and $~T, \alpha$ as prescribed below. Denote the output of this algorithm by $\wprpro(X).$\\
a) Suppose $\delta = 0, ~\pen \leq 1.$ Then setting $\alpha = \frac{L^2}{\mu n}\min\{\frac{1}{n}, \frac{d}{\varepsilon}\}$ and $T = 2\max\left\{n^2, \frac{n \varepsilon}{d}\right\}$ implies \[
\ERPpro \leq  \frac{L^2}{\mu}\left(\frac{5}{n} + 9\pen\right).\]
The runtime of this method is $O(\max\{dn^2, \varepsilon n\}).$\\
b) Suppose $\delta \in (0, \frac{1}{2}), ~\lpen \leq 1.$ Then setting $\alpha = \frac{L^2}{\mu n} \min \left \{\frac{1}{n}, \lpe\right\}$ and $T = 2n \max \left\{n, \frac{\varepsilon}{\sqrt{d}(2\cd + \sqrt{\varepsilon})}\right\}$ implies \[\ERDPpro \leq \frac{L^2}{\mu}\left(\frac{5}{n} + 6\lpen\right).\] The resulting runtime is $O\left(\max \left \{dn^2, n \sqrt{d \varepsilon} \right\} \right).$
\end{corollary}

The optimal non-private expected excess population loss for $\mu$-strongly convex, $L$-Lipschitz $f$ is $O(\frac{L^2}{\mu n})$ \cite{hk14}. So if $\delta = 0$ and $d \lesssim \varepsilon,$ then we match the optimal non-private rate and in effect get ``privacy for free.'' Likewise, for $\delta > 0$, if $\sqrt{d} \cd + \sqrt{\varepsilon} \lesssim \varepsilon,$ then we get privacy for free. 

\vspace{0.2cm}
As noted earlier, neither \cite{bft19}, nor \cite{fkt20} give expected excess population loss bounds for $\delta = 0$. The only private population loss bound we are of for with $\delta = 0$ is $O\left(\left(\frac{L^2}{\mu}\right)\left(\pen \sqrt{\log(n)} + \frac{1}{n}\right)\right)$ in \cite[Theorem F.2]{bst14}, obtained by exponential sampling + localization. 
Thus, we beat their loss bound by a logarithmic factor. More importantly, however, is our runtime advantage. As noted above, our runtime is $O\left( \max\left\{n^2 d, \varepsilon n \right\} \right)$ for $\delta = 0.$
By contrast, the runtime of the exponential sampling + localization method of \cite{bst14} is generally much larger: $\widetilde{O}\left(\frac{L^2}{\mu^2} d^9 \frac{n}{\varepsilon^2} \max \{1, \frac{L}{\mu}\}\right)$ in the (non-isotropic) worst case \cite[Theorem 3.4, Lemma 4.2]{bst14}. 

\vspace{0.2cm}
For $\delta > 0,$ \cite[Algorithm 3, Sec. 5.1]{fkt20} gives an efficient phased ERM algorithm that can nearly (up to a $O\left(\log\left(\frac{1}{\delta}\right)\right)$ factor) achieve the optimal (by \cite{bft19}) private rate of $O\left(\frac{L^2}{\mu}\left(\frac{1}{n} + \frac{d}{(\varepsilon n)^2}\right)\right)$ in the same runtime as our method: $\widetilde{O}(n^2 d)$.  However, this algorithm requires $\varepsilon \leq 2 \log(\frac{1}{\delta})$ to be differentially private since they use the suboptimal form of Gaussian noise at each iteration, whereas we provide the flexibility of choosing any $\varepsilon > 0.$ Additionally, our method is arguably much simpler to implement in practice.

\vspace{.1cm}
The noisy SGD method of \cite{bst14} also results in a population loss bound of $O\left(\lmu\left(\frac{1}{n} + \frac{\sqrt{d}}{\varepsilon n} \log^{2}\left(\frac{n}{\delta}\right) \log\left(\frac{1}{\delta}\right)\right)\right)$ for $\delta > 0$, which is looser than our bound by a factor of $\log^{2}\left(\frac{n}{\delta}\right).$ This method requires $\varepsilon \leq 2 \log\left(\frac{1}{\delta}\right)$ to be differentially private and has the same runtime as our algorithm.

\paragraph{Smooth, strongly convex, Lipschitz functions.}
Adding the assumption of $\beta$-smoothness results in tighter, \textit{near-optimal} private population loss bounds. We begin with loss bounds for the conceptual output perturbation algorithm \cref{eq: conceputal output pert}.

\begin{proposition}
\label[proposition]{prop: sc smooth pop loss}
Let $f(\cdot, x)$ be $\mu$-strongly convex, $L$-lipschitz, and $\beta$-smooth for all $x \in \XX$ with condition number $\kappa = \beta/\mu,$ let $\varepsilon > 0$ and $\delta \in [0, \frac{1}{2}).$ Let $X$ be a data set of size $n$ drawn i.i.d. according to $\mathcal{D}.$ Run \cref{eq: conceputal output pert} on the empirical loss $\hf(w,X).$\\
a) If $\delta = 0$ and $\pen \leq 1,$ then \[
\ERP \leq  \frac{L^2}{\mu}\left(\frac{2}{n} + 4\kappa \left(\pen\right)^2\right).\] \\
b) If $\delta \in \left(0, \frac{1}{2}\right)$ and $\lpen \leq 1,$ then \[\ERP \leq \frac{L^2}{\mu}\left(\frac{2}{n} + 4\kappa \left(\lpen\right)^2\right).\]
\end{proposition}

With the smoothness assumption, we can also use Katyusha instead of SGD for implementation and obtain faster runtimes:  

\begin{corollary}
\label[corollary]{rem: smooth sc pop loss runtime katyusha}
Let $f(\cdot, x)$ be $L$-Lipschitz on $\WW$, $\mu$-strongly convex, and $\beta$-smooth for all $x \in \XX,$ let $\varepsilon > 0.$ Let $X$ be a data set of size $n$ drawn i.i.d. according to $\mathcal{D}.$ Run $\mathcal{A}=$~\cref{alg: black box sc} on $\hf(w,X)$ with $\mathcal{M}$ as Katyusha with $T = O\left((n + \sqrt{n \kappa}) \log\left(\frac{LR}{\alpha}\right)\right)$ and $\alpha$ as prescribed below. \\
a) Suppose $\delta = 0$ and $\pen \leq 1.$ Then setting $\alpha = \frac{L^2}{\mu n^2} \min \left\{\kappa \left(\pe\right)^2, 1 \right\}$ results in a point $\wpr(X)$ such that \[
\ERP \leq \frac{L^2}{\mu}\left(\frac{5}{n} + 26\kappa\left(\pen\right)^2\right)\] in $T = \widetilde{O}(n + \sqrt{n \kappa})$ stochastic gradient evaluations. The resulting runtime is $\widetilde{O}(d(n + \sqrt{n \kappa})).$ \\
b) Suppose $\delta \in (0, \frac{1}{2})$ and $\lpen \leq 1.$ Then setting \\
$\alpha = \frac{L^2}{\mu n^2} \min \left\{\kappa \left(\lpe\right)^2, 1 \right\}$ results in a point $\wpr(X)$ such that \[\ERP \leq \lmu \left(\frac{5}{n} + 13.5\kappa \left(\lpen\right)^2\right)\] in 
$T = \widetilde{O}(n + \sqrt{n \kappa})$ stochastic gradient evaluations.
The resulting runtime is $\widetilde{O}(d(n + \sqrt{n \kappa})).$  
\end{corollary}

\begin{remark}
\label[remark]{rem: min smooth sc pop}
Since the smooth, strongly convex, Lipschitz loss function class is a subset of the strongly convex, Lipschitz loss function class, we obtain by \cref{prop: sc lip pop loss} the following upper bounds on the expected excess population loss of smooth, strongly convex, Lipschitz functions: $\ERP \lesssim \frac{L^2}{\mu} \left(\frac{1}{n}\right) + \min \left\{LR,  \frac{L^2}{\mu} \pen, \kappa \frac{L^2}{\mu} \left(\pen\right)^2 \right\}$ for $\delta = 0$; and $\ERDP \lesssim \frac{L^2}{\mu} \left(\frac{1}{n}\right) + \\
\min \left\{LR,  \frac{L^2}{\mu} \lpen, \kappa \frac{L^2}{\mu} \left(\lpen\right)^2 \right\}$ for $\delta > 0.$ Furthermore, the runtime for these enhanced upper bounds are (up to log terms) as stated above: $\widetilde{O}\left(d(n + \sqrt{\kappa n})\right).$ This is because obtaining the expected excess population loss bounds of \cref{prop: sc lip pop loss} via Katyusha (on the empirical loss) only affects the choice of $\alpha$ (in \cref{section 3: implment}), which is inside of a $\log(\cdot)$ in the expression for the iteration complexity of Katyusha. 
\end{remark}

The non-private statistically optimal population loss is $O(\frac{L^2}{\mu n}),$ as for non-smooth. Hence, if $\delta = 0,$ we get ``privacy for free'' whenever $\kappa \left(\pe\right)^2 \lesssim n$ or $d \lesssim \varepsilon$. For $\delta > 0$, we get ``privacy for free'' whenever $\frac{\kappa d \cd^2}{\varepsilon} \lesssim n$ or $\sqrt{d}\cd + \sqrt{\varepsilon} \lesssim \varepsilon.$ 

\vspace{.2cm}
For $\delta = 0$, \cite{zhang2017} presents an expected excess population loss bound of $O\left(\kappa \frac{L^2}{\mu} \left(\pen\right)^2 + \frac{L^2}{\mu}\left(\frac{1}{n}\right)\right)$ (via gradient descent with output perturbation), which is achieved in runtime $\widetilde{O}(dn\kappa).$ Thus, their loss bound matches ours whenever $\kappa \left(\pen \right) < 1$ and is otherwise inferior by \cref{rem: min smooth sc pop}; their runtime is also slower than ours, especially when $\kappa$ is large. Besides \cite{zhang2017}, the only other private population loss bound we are aware of for the smooth, strongly convex, Lipschitz class with $\delta = 0$ is the one from exponential sampling + localization in \cite{bst14}, which does not benefit additionally from smoothness, to the best of our knowledge. Therefore, our method provides the tightest $(\varepsilon, 0)$-differentially private excess population loss bounds that we are aware of and runs in less time than any competing algorithm. 

\vspace{0.2cm}
For $\delta > 0, ~\varepsilon \leq 2 \log\left(\frac{1}{\delta}\right),$ the recent work of  \cite[Section 5]{fkt20} attains the optimal bound of $O\left(\frac{L^2}{\mu}\left(\frac{1}{n} + \frac{\sqrt{d \log({\frac{1}{\delta})}}}{\varepsilon n}\right)\right)$ with
an efficient algorithm, provided that $\kappa \kappa \leq \frac{n}{\log(n)}$. 
By contrast, \textit{our algorithm does not require any restrictions on $\kappa$ or $\varepsilon,$ which make it an attractive alternative}, particularly for ill-conditioned problems or problems where only a modest amount of privacy is desired, but strong (perhaps non-private) expected excess population loss is needed.

\vspace{.2cm}
In Appendix~\ref{subsection 4.3}, we extend our algorithm to convex loss functions via regularization.

\section{Applications to Private Tilted ERM and Private Adversarial Training}
\label{section 5: applications}
In this section, we specialize our developed theory to two specific applications in machine learning: differentially private tilted ERM (TERM) and adversarial training. As discussed in the Introduction (\cref{sec 1}), both of these problems are of great practical importance in modern machine learning and yet there is no known efficient differentially private algorithm for solving either of them while providing excess risk guarantees. Using our general methodology and results from Sections 2-4, we will address these gaps and obtain excess loss and runtime bounds for privately optimizing TERM and adversarial objectives via output perturbation.

\subsection{Differentially Private Tilted Empirical Risk Minimization}
\label{sec: TERM}
Consider the tilted ERM (TERM) problem
\[
\min_{w \in \mathbb{R}^d} \left[F_{\tau}(w,X) := \frac{1}{\tau} \log\left(\frac{1}{n} \sum_{i=1}^{n} e^{\tau f(w, x_{i})}\right)\right],
\]
where $\tau > 0$ is a given constant (see e.g. \cite{kort72, pee11, cohen14, cohen16, katharop17, li2020} and the references therein for the applications of this loss). It is easy to show that 
as $\tau \xrightarrow{} 0, \ftau(w,x) \xrightarrow{} \frac{1}{n}\sum_{i=1}^{n} f(w, x_{i})$, so this extends the classical ERM framework in the limit. It also encompasses the max loss ($F_{max}(w,X) = \max \{f(w, x_1), ... , f(w, x_{n} \}$) for instance, by letting $\tau \xrightarrow{} \infty.$ More generally, a benefit of the TERM framework is that it allows for a continuum of solutions between average and max loss, which, for example, can be calibrated to promote a desired level of fairness in the the machine learning model \cite{li2020}. Another interpretation of TERM is that as $\tau$ increases, the variance of the model decreases, while the bias increases. Thus, $\tau$ can also be tuned to improve the generalization of the model via the bias/variance tradeoff. In what follows, we specialize our developed theory to the TERM objective function. 

\vspace{.2cm}
Our \ul{main result} in this section is~\cref{thm: smooth TERM implementation}, which gives \textit{near-optimal excess risk} for smooth, strongly convex TERM with an efficient algorithm. 

\subsubsection{Strongly convex, Lipschitz, twice differentiable $f(\cdot, x)$}
First, we show that if $f$ is ``nice enough,'' then $\ftau \in \FF,$ so that the excess risk bounds proved earlier hold; we will also refine these results to show how excess risk depends on the parameter $\tau > 0.$ Assume $f( \cdot, x)$ is a $\mu$-strongly convex, twice differentiable, $L$-Lipschitz loss function for all $x \in \XX.$ Then we compute:

\begin{equation} 
\label[equation]{tilted ERM lemma 1}
\nabla_{w} \ftau(w,X) = \sum_{i = 1}^{n} v_{i}(\tau, w) \nabla_{w} f(w, x_{i}), 
\end{equation}
where $~v_{i}(\tau, w) := \frac{e^{\tau f(w, x_{i})}}{\sum_{j=1}^n e^{\tau f(w, x_{j})}}.$
It follows easily from \cref{tilted ERM lemma 1} that (for $\tau > 0$) if $f(\cdot, x)$ is $L$-Lipschitz for all $x_{i} \in X$ then $\ftau(\cdot, X)$ is $L$-Lipschitz. Furthermore, $\ftau(\cdot, X)$ is $\mu$-strongly convex:

\begin{lemma} \cite[Lemma 3]{li2020}
\label[lemma]{tilted ERM lemma 3}
Let $X \in \XX^n$, $\tau > 0$. Assume $f(\cdot, x_{i})$ is twice-differentiable and $\mu$-strongly convex for all $x_{i} \in X.$ Then $\ftau(\cdot, X)$ is twice-differentiable and $\mu$-strongly convex. 
\end{lemma}

It follows from \cite[Lemma 3]{li2020} and \cref{tilted ERM lemma 1} that $\ftau \in \FF.$ Then combining \cref{tilted ERM lemma 1} and \cref{tilted ERM lemma 3} with \cref{prop:sc sensitivity} implies that the $L_2$ sensitivity of $\ftau$ is bounded as $\Delta_{\ftau} \leq \frac{2L}{\mu}$ for all $\tau > 0$ and even as $\tau \to \infty$. If, in addition, $f$ is bounded on $\WW \times \XX$, then we also have the following estimate of the sensitivity of $\ftau$:
\begin{lemma}
\label[lemma]{lem: sensitivity of ftau}
Let $\tau > 0.$ Suppose $f(\cdot, x)$ is $L$-Lipschitz on $B(0,R)$ (where $\|\ws(X)\| \leq R$) and $\mu$-strongly convex for all $x \in \XX.$ Moreover, assume $a_R \leq f(w,x) \leq A_R$ for all $w \in \WW$ and all $x \in \XX.$ Then \[
\Delta_{\ftau} \leq \frac{2L}{\mu} \min\left\{1, \frac{C_{\tau}}{n} \right\}
\]
where $C_{\tau} := e^{\tau (A_R - a_R)}.$ 
\end{lemma}
Observe that requiring boundedness of $f$ on $B(0,R) \times \XX$ is not very restrictive: indeed, it is automatic if $f$ is continuous and $\XX$ is compact (e.g. if input data is normalized), by the extreme value theorem. Also, note that as $\tau \to 0,$ we recover the sensitivity estimate $\frac{2L}{\mu n}$ of \cref{prop:sc sensitivity} for $F \in \FFE.$ On the other hand, as $\tau$ increases, $\Delta_{\ftau}$ increases. This is intuitive since when $\tau \to \infty$, $\ftau \to F_{max},$ so the behavior of the loss is determined by only the worst case data point, which increases the sensitivity. 

\vspace{0.2cm}
\cref{lem: sensitivity of ftau} implies that running our conceptual output perturbation algorithm with projection $\mathcal{A'}$ (\cref{eq: conceptual out pert proj}) on $\ftau$ ensures $(\varepsilon, \delta)$-privacy and also gives \[ \ERTpro \leq \frac{2 L^2 C_{\tau}}{\mu} \left(\pen\right) ~\text{if} ~\delta = 0, 
\]
and \[
\ERDTpro \leq \frac{\sqrt{2}L^2 C_{\tau}}{\mu} \left(\lpen\right) ~\text{if} ~\delta \in (0, \frac{1}{2}), 
\]
by \cref{prop: sc ER} if $f$ is bounded on $\WW \times \XX.$ (Note: throughout this section, we denote $\ws(X) \in \argmin \ftau(w,X).$)

Moreover, we can achieve these risk bounds efficiently:

\begin{theorem}
\label{thm: TERM implementation}
Assume $f(\cdot, x)$ is $L$-Lipschitz on $\WW$ and $\mu$-strongly convex and twice differentiable for all $x \in \XX.$ Assume further that $a_R \leq f(w,x) \leq A_R$ for all $w \in \WW$ and all $x \in \XX.$ Then for any $\tau > 0,$ running \cref{alg: black box sc} on $\ftau$ with $\mathcal{M}$ as the subgradient method yields 
\[\ERTpro \leq 9 \frac{L^2 C_{\tau}}{\mu} \pen\]
in $T = \frac{2}{C_{\tau}} \max\{\frac{n^2}{C_{\tau}}, \frac{\varepsilon n}{d}\}$ gradient evaluations of $\ftau$ and runtime $O\left(\frac{n^2 d}{C_{\tau}} \max\left\{\frac{\varepsilon}{d}, \frac{n}{C_{\tau}}\right\}\right)$ for $\delta = 0.$ For $\delta \in (0, \frac{1}{2}),$ we get \[
\ERDTpro \leq 9 \frac{L^2 C_{\tau}}{\mu} \lpen\]
in runtime $O\left(\frac{nd}{C_{\tau}} \max\left\{\frac{n^2}{C_{\tau}}, \lnep \right\} \right).$ Here $C_{\tau} = e^{\tau(A_R - a_R)}$.
\end{theorem}
These results follow by plugging the sensitivity estimate from \cref{lem: sensitivity of ftau} into the proof of \cref{thm: black box sc lip} with $\alpha = \frac{L^2 C_{\tau}}{\mu n} \min\left\{\frac{C_{\tau}}{n}, \pe \right\}$ for $\delta = 0$ and $\alpha = \frac{L^2 C_{\tau}}{\mu n} \min\left\{\frac{C_{\tau}}{n}, \lpe \right\}$ for $\delta \in (0, \frac{1}{2})$ and then recalling the iteration complexity $T = \left\lceil \frac{2L^2}{\mu \alpha} \right\rceil$ of the subgradient method. Alternatively, one could use the stochastic TERM algorithm described in \cite[Appendix C]{li2020} as the non-private input $\mathcal{M}$ in \cref{alg: black box sc}. However, no iteration complexity/runtime bounds are given for that method in \cite{li2020}. 

\subsubsection{Smooth, strongly convex, Lipschitz, twice differentiable $f(\cdot, x)$}
If we assume additionally that $f(\cdot, x_i)$ is $\beta$-smooth for all $x_i, i \in [n]$, then $\ftau(w, X)$ is $\beta_{\tau}$-smooth with $\beta_{\tau}$ given in the following:
\begin{lemma}
\label[lemma]{lem: TERM beta smoothness}
Assume $f(\cdot, x_i)$ is $\beta$-smooth and $L$-Lipschitz for all $i \in [n].$ Then for any $\tau > 0,$ the TERM objective $\ftau(\cdot,X)$ is $\beta_{\tau}$-smooth for $X = (x_1, \cdots, x_n)$, where $\beta_{\tau}:= \beta + L^2 \tau.$
\end{lemma}
See Appendix~\ref{sec: TERM proofs} for the proof. It follows by combining the above lemmas that if $f$ is twice-differentiable, $\beta$-smooth, $\mu$-strongly convex, and $L$-Lipschitz (i.e. $f \in \HH$), then $\ftau \in \mathcal{H}_{\left(\beta + 2L^2 \tau\right), \mu, L, R}.$ Hence by \cref{cor: smooth sc fund ER bounds}, upon running the conceptual output perturbation algorithm \cref{eq: conceputal output pert}, we obtain: \[
\ERT \leq 4\kappa_{\tau} \frac{L^2}{\mu} \left(\pe\right)^2 ~\text{if} ~\delta = 0,
\] and
\[
\ERDT \leq 4 \kappa_{\tau} \frac{L^2}{\mu} \left(\lpe\right)^2 ~\text{if} ~\delta \in (0, \frac{1}{2}),
\]
where $\kappa_{\tau}:= \frac{\beta_{\tau}}{\mu} = \frac{\beta + 2L^2 \tau}{\mu}$ is the condition number of $\ftau.$ These bounds hold even for unbounded $f$ and as $\tau \to \infty;$ however, the non-smooth bounds from the previous subsection, which also apply here, would be tighter as $\tau \to \infty,$ since then $\kappa_{\tau} \to \infty.$ 
If, in addition, $f$ is bounded on $\WW \times \XX,$ then plugging the estimate for $\Delta_{\ftau}$ from \cref{lem: sensitivity of ftau} into \cref{lem: smooth ER} yields:
\[
\ERT \leq 4\kappa_{\tau} \frac{L^2 C_{\tau}}{\mu} \left(\pen\right)^2 ~\text{if} ~\delta = 0,
\] and
\[
\ERDT \leq 4 \kappa_{\tau} \frac{L^2 C_{\tau}}{\mu} \left(\lpen\right)^2 ~\text{if} ~\delta \in (0, \frac{1}{2}),
\]
With $\beta$-smoothness, we can also use Nesterov's Accelerated Gradient Descent (AGD, \cref{alg: AGD}) as $\mathcal{M}$ in \cref{alg: black box sc} to speed up runtime and obtain:

\begin{algorithm}[h!]
\caption{Nesterov's (Projected) Accelerated Gradient Descent (AGD)}
\label{alg: AGD}
\begin{algorithmic}[1]
\Require~$F(w,X) \in \HH$, data universe $\XX,$ dataset $X = \{x_{i}\}_{i=1}^n \in \XX^n$, iteration number $T \in \mathbb{N}$.   
\State Choose any point $w_0 \in \WW.$
\For {$t \in [T]$} 
 \State $y_{t} = \Pi_{\WW}[w_{t-1} - \frac{1}{\beta} \nabla F(w_{t-1}, X)].$
 \State $w_{t} = \left(1 + \frac{\sqrt{\kappa} - 1}{\sqrt{\kappa} + 1}\right)y_{t} - \frac{\sqrt{\kappa} - 1}{\sqrt{\kappa} + 1}y_{t-1}.$\\
\EndFor \\
\Return $\widehat{w}_{T} = y_{T}.$
\end{algorithmic}
\end{algorithm}

\begin{theorem}
\label{thm: smooth TERM implementation}
Assume $f(\cdot, x)$ is $L$-Lipschitz on $\WW$, $\beta$-smooth, $\mu$-strongly convex and twice differentiable for all $x \in \XX.$ Assume further that $a_R \leq f(w,x) \leq A_R$ for all $w \in \WW$ and all $x \in \XX.$ Then for any $\tau > 0,$ running \cref{alg: black box sc} with $\mathcal{M}$ as AGD on $\ftau$ yields
\[
\ERT \leq 26 \kappa_{\tau} \frac{L^2 C_{\tau}}{\mu} \left(\pen\right)^2
\]
if $\delta = 0,$ and \[
\ERDT \leq 13.5 \kappa_{\tau} \frac{L^2 C_{\tau}}{\mu} \left(\lpen\right)^2
\] if $\delta \in (0, \frac{1}{2}).$
Here $C_{\tau} = e^{\tau(A_R - a_R)}$ and $\kappa_{\tau} = \frac{\tau L^2 + \beta}{\mu}.$ Both of these excess risk bounds are realized in runtime $\widetilde{O}(d n\sqrt{\kappa_{\tau}}).$
\end{theorem}
To derive \cref{thm: smooth TERM implementation}, we simply plug the estimate of $\Delta_{\ftau}$ from \cref{lem: sensitivity of ftau} into the proof of \cref{th: smooth sc implement} and set \[
\alpha = \frac{L^2}{\mu n^2}
\begin{cases}
\min\left\{\kappa_{\tau}\left(\pe\right)^2, C_{\tau}^2 \right\} &\mbox{if} ~\delta = 0\\
\min\left\{\kappa_{\tau}\left(\lpe\right)^2, C_{\tau}^2 \right\} &\mbox{if} ~\delta \in (0, \frac{1}{2}).
\end{cases}\] Then recall the iteration complexity $T = \widetilde{O}(\sqrt{\kappa_{\tau}})$ of AGD (where we hide a logarithmic factor that depends on $\alpha$) for obtaining $w_T$ such that $\ftau(w_T, X) - \ftau(\ws(X), X) \leq \alpha.$

\subsection{Differentially Private Adversarially Robust Learning}
\label{subsection: DP adversarial training}
In recent years, adversarial attacks on neural networks have raised significant concerns on reliability of these methods in critical applications. Adversarial attacks are input examples to a machine learning model crafted by making small perturbations to legitimate inputs to mislead the network. These adversarial examples lead to highly confident, but incorrect outputs; see e.g. \cite{szegedy14, goodfellow, papernot16, eykholt, deepfool} and the references therein. A natural approach to training a model that is robust to such adversarial attacks is to solve the following adversarial training problem \cite{madry, zhangjordan-adversarial, maher}: 
\begin{equation}
\label{eq: adversarial training}
    \min_{w \in \mathbb{R}^d} \max_{\mathbf{v} \in S^n} F(w, X + \mathbf{v}, \mathbf{y}).
\end{equation}
Here $X = (x_1, \cdots, x_n) \in \XX^n$ contains the feature data, $\mathbf{y} \in \mathcal{Y}^n$ is the corresponding label/target vector (e.g. $y_{i} \in \{0, 1\}$ for a binary classification task), and $S$ is a set of permissible adversarial perturbations. As discussed in \cref{sec 1}, solving \cref{eq: adversarial training} with a practical differentially private algorithm to ensure privacy, robustness, and computational speed simultaneously is an important open problem. Indeed, we are not aware of any works that provide privacy, robustness (``excess adversarial risk''), and runtime guarantees for solving \cref{eq: adversarial training}. In this section, we illustrate how our methods and results developed so far can be easily applied to establish excess adversarial risk bounds (robustness guarantees) and runtime bounds via output perturbation.

\subsubsection{Notation and preliminaries}

Let $D = (X, \mathbf{y}) = ((x_1, y_1), \cdots, (x_n, y_n)) \in \left(\XX \times \cal{Y}\right)^{n}$ be a given training data set.  Assume that the adversarial perturbation set \textit{$S$ is convex and compact with $L_2$ diameter $\rho$ }and that the adversary chooses $\mathbf{v} = (v_1, \cdots, v_n) \in S^n$ corresponding to the training examples $((x_1, y_1), \cdots, (x_n, y_n)).$ The following definition is for notational convenience: 
\begin{definition}
\label[definition]{def: H adversarial}
For a loss function $F: \mathbb{R}^d \times (\XX + S)^{n} \times \mathcal{Y}^n,$ and dataset $D = (X, \mathbf{y}) \in (\XX \times \mathcal{Y})^{n},$ denote the function of the weights and perturbations corresponding to $F$ by
\[H_D: \mathbb{R}^d \times S^n \to \mathbb{R}, ~H_{D}(w, \mathbf{v}):= F(w, X + \mathbf{v}, \mathbf{y}).\]
\end{definition}
We make the following additional assumptions: 
\begin{assumption}
\label[assumption]{assump: strongly convex adversarial}
$H_D(\cdot, \mathbf{v}) \in \FF$ for all $\mathbf{v} \in S^n$  and all $D \in \left(\XX \times \cal{Y}\right)^{n}.$ 
\end{assumption}

\begin{assumption}
\label[assumption]{assump: concave adversarial}
$H_D(w, \cdot)$ is continuous and concave for all $w \in \WW$ and all $D \in \left(\XX \times \cal{Y}\right)^{n}.$ 
\end{assumption}
Note that \cref{assump: concave adversarial} is a standard assumption in the min-max literature \cite{maher, lin20}. In addition, problems with a discrete adversary set $S$ can be transformed into problems that satisfy \cref{assump: concave adversarial} (see \cite{raz2020}). Together with \cref{assump: strongly convex adversarial}, it ensures the existence of a saddle point (see definition below) and enables us to (approximately) find the saddle point and implement our output perturbation method efficiently (see \cref{adversarial implementation}). However, it is not strictly necessary for the results for conceptual output perturbation that we present in \cref{conceptual adversarial training}. Next, we recall a basic notion from min-max optimization: 
\begin{definition}
For $\alpha \geq 0,$ say a point $(\widehat{w}, \widehat{\mathbf{v}}) \in \mathbb{R}^d \times S^n$ is an \textbf{$\alpha$-saddle point} of a convex (in $w$)-concave (in $\mathbf{v}$) function $H(w,\mathbf{v})$ if \[
\max_{\mathbf{v} \in S^n} H(\widehat{w}, \mathbf{v} ) - \min_{w \in \mathbb{R}^d} H(w, \widehat{\mathbf{v}}) \leq \alpha.
\]
\end{definition}

Observe that by \cref{assump: strongly convex adversarial}, \cref{assump: concave adversarial}, and convexity and compactness of $\WW$ and $S$, there exists at least one $\alpha$-saddle point point $(\widehat{w}, \widehat{\mathbf{v}}) \in \WW \times S^n$ of $H_D$ for any $\alpha \geq 0.$  Also, if we denote \[
G_D(w):= \max_{\mathbf{v} \in S^n} H_D(w,\mathbf{v}) = \max_{\mathbf{v} \in S^n} F(w, X + \mathbf{v}, \mathbf{y}),\] 
which is in $\FF$ (by \cref{assump: strongly convex adversarial} and compactness of $S^n$), and \[
\ws(D):= \argmin_{w \in \mathbb{R}^d} G_D (w) = \argmin_{w \in \mathbb{R}^d} \max_{\mathbf{v} \in S} H_D(w,\mathbf{v}) \in \WW,\]
then for any $\alpha$-saddle point $(\widehat{w}, \widehat{\mathbf{v}})$ of $H_D,$ we have: \begin{align*}
G_D(\widehat{w}) - G_D(\ws(D)) &=
   \max_{\mathbf{v} \in S^n} H_D(\hw,\mathbf{v}) - \min_{w \in \WW} H_D(w, \widehat{\mathbf{v}})\\ 
   &+ \min_{w \in \WW} H_D(w, \widehat{\mathbf{v}}) - \max_{v \in S^n} H_D(\ws(D),\mathbf{v})\\
    &\leq \alpha + 0,
\end{align*}
because \begin{align*}
\min_{w \in \WW} H_D(w, \widehat{\mathbf{v}}) - \max_{\mathbf{v} \in S^n} H_D(\ws(D),\mathbf{v}) &\leq 
\max_{\mathbf{v} \in S^n} \min_{w \in \WW} H_D(w, \mathbf{v}) \\
&- \min_{w \in \WW} \max_{\mathbf{v} \in S^n} H_D(w,\mathbf{v}) \\
&= 0
,
\end{align*}
by Sion's minimax theorem \cite{sion58}. 

\vspace{0.2cm}
For a model $w_{\mathcal{A}}$ trained on loss function $F$ (by some randomized algorithm $\mathcal{A}$), the measure of adversarial robustness that we will consider is:
\begin{definition}
Let $\wpr$ be the output of a randomized algorithm $\mathcal{A}$ for solving \cref{eq: adversarial training}. Define the \textbf{excess adversarial risk} of $\mathcal{A}$ by \[
\mathbb{E}_{\mathcal{A}} \max_{\mathbf{v} \in S^n} F(\wpr, X + \mathbf{v}, y) - \min_{w \in \mathbb{R}^d} \max_{\mathbf{v} \in S^n} F(w, X + \mathbf{v}, \mathbf{y}) = \mathbb{E}_{\mathcal{A}} G_D(\wpr) - G_D(\ws(D)).
\]
\end{definition}

\vspace{0.2cm}
In what follows, we aim to quantify the tradeoffs between excess adversarial risk, privacy, and runtime for output perturbation. We begin by considering the conceptual output perturbation algorithm, in which we assume that we can compute an exact ($\alpha = 0$) saddle point of $H_D(w,\mathbf{v}).$

\subsubsection{Conceptual output perturbation for differentially private adversarial training}
\label[section]{conceptual adversarial training}
The key step towards applying our theory is to notice that by \cref{assump: strongly convex adversarial}, the $L_2$ sensitivity $\Delta_{G_{D}}$ of $G_D(w)$ is bounded according to \cref{prop:sc sensitivity}: 
$\Delta_{G_{D}} \leq \frac{2L}{\mu}$ for general $F$ and $\Delta_{G_{D}} \leq \frac{2L}{\mu n}$ if $F$ (and hence $G_D$) is of ERM form. To clarify further, note that if $F(w, X, y) = \frac{1}{n} \sum_{i=1}^n f(w, x_i, y_i)$ is of ERM form, then $f(\cdot, x_i + v_i, y_i)$ is $\mu$-strongly convex and $L$-Lipschitz for all $(x_i, y_i) \in \XX \times \mathcal{Y}$ and all $v_i \in S$ by \cref{assump: strongly convex adversarial} and the discussion in the first footnote of \cref{subsection 2.3: sc Lip ER bounds}. It follows that $G_{D}(w) = \max_{\mathbf{v} \in S^n} \frac{1}{n} \sum_{i=1}^n f(w, x_i + v_i, y_i) = \frac{1}{n} \sum_{i=1}^n g_i(w),$ where $g_i(w) := \max_{v_i \in S} f(w, x_i + v_i, y_i)$ is $\mu$-strongly convex and $L$-Lipschitz by the following:
\begin{lemma}
\label[lemma]{lem: max is strongly convex and lipschitz}
Let $h: \mathbb{R}^d \times S$ be such that $h(\cdot, v)$ is $L$-Lipschitz on $\WW \subset \mathbb{R}^d$ and $\mu$-strongly convex in $w$ for all $v \in S.$ Assume further that $S$ is compact and that $h(w, \cdot)$ is continuous for all $w \in \WW.$ Then $g(w):= \max_{v \in S} h(w,v)$ is $L$-Lipschitz on $\WW$ and $\mu$-strongly convex.
\end{lemma}

Our conceptual output perturbation algorithm is described in \cref{alg: conceptual adv robust}. The fact that it is differentially private follows from applying \cref{conceptual sc alg is private} to $G_{D}$ and the sensitivity estimate for $G_{D}$ stated above. 

\begin{algorithm}[ht!]
\caption{Conceptual Output Perturbation for DP Adversarial Training}
\label{alg: conceptual adv robust}
\begin{algorithmic}[1]
\Require~ Number of data points $n \in \mathbb{N},$ dimension $d \in \mathbb{N}$ of feature data, non-private (possibly randomized) optimization method $\mathcal{M}$, privacy parameters $\varepsilon > 0, \delta \geq 0$, feature data universe $\XX \subseteq \mathbb{R}^d$ and label universe $\mathcal{Y},$ data set $D = (X, \mathbf{y}) \in (\XX \times \mathcal{Y})^{n}$, convex compact perturbation set $S \subset \mathbb{R}^d$ of diameter $\rho \geq 0,$ loss function $F(w,X,y)$ that is $L$-Lipschitz on $B(0,R)$ and $\mu$-strongly convex in $w$ and concave in $\mathbf{v} \in S^n.$
\State Solve $\min_{w \in \WW} \max_{\mathbf{v} \in S^n} F(w, X + \mathbf{v}, \mathbf{y})$ to obtain a saddle point $(\ws(D), v^{*}(D)).$ 
\State Add noise to ensure differential privacy: 
$\wpr(D):= \Pi_{\WW}(\ws(D) + z)$, where the density $p_{z}(t)$ of $z$ is given as: 
$p_{z}(t) \propto 
\begin{cases}
 \exp\left(- \frac{\varepsilon\|t\|_2}{\Delta}\right) &\mbox{if } \delta = 0 \\
 \exp\left(- \frac{\varepsilon^2 \|t\|_2^2}{\Delta^2 \left(\cd + \sqrt{\cd^2 + \varepsilon}\right)^2}\right) &\mbox{if } \delta \in \left(0, \frac{1}{2}\right), \\
\end{cases}$ 
and $\Delta := \frac{2L}{\mu n}$
\end{algorithmic}
\end{algorithm}

\vspace{0.2cm}
By the results in \cref{subsection 2.3: sc Lip ER bounds} and \cref{subsection 2.5: smooth sc}, we can obtain excess adversarial risk bounds for strongly convex-concave, Lipschitz $F$ with or without the $\beta$-smoothness (in $w$) assumption. Here we state the results for $\beta$-smooth $F,$ which is a simple consequence of \cref{cor: smooth sc fund ER bounds}:
\begin{proposition} 
\label[proposition]{prop: fund EAR bounds for adv training smooth sc}
Let $D = (X, \mathbf{y}) \in (\mathcal{X} \times \mathcal{Y})^{n}$ be a dataset
and let $\varepsilon > 0, \delta \in [0, \frac{1}{2}).$ Assume that $S$ is a convex compact set in $\mathbb{R}^d$ and that the ERM loss function $F$ is such that \cref{assump: strongly convex adversarial} and \cref{assump: concave adversarial} hold (see \cref{def: H adversarial}). Furthermore, assume that $H_D(\cdot, \mathbf{v})$ is $\beta$-smooth for all $\mathbf{v} \in S^n$.\\
a) If $\delta = 0,$ then \[
\ERA \leq 4\kappa \frac{L^2}{\mu} \left(\pen\right)^2.\] 
b) If $\delta \in (0, \frac{1}{2}),$ then \[
\ERDA \leq 4\kappa \frac{L^2}{\mu} \left(\lpen\right)^2.\]\\
\end{proposition}

Observe that the stated assumptions on $F$ and $H$ above are stronger than what is required to obtain \cref{prop: fund EAR bounds for adv training smooth sc}. Indeed, instead of assuming strong convexity, Lipschitzness, and smoothness of $H_D(w,v),$ it is enough to assume strong convexity, Lipschitzness, and smoothness of $G_D(w) = \max_{\mathbf{v} \in S^n} F(w, X + \mathbf{v}, \mathbf{y})$ for all $D = (X, \mathbf{y}) \in (\mathcal{X} \times \mathcal{Y})^{n}$. Moreover, concavity is not necessary: all we need is the existence of a saddle point of $H_{D}$ for any $D \in (\XX \times \mathcal{Y})^n.$ 

\vspace{.2cm}
\cref{prop: fund EAR bounds for adv training smooth sc} describes the tradeoffs between excess adversarial risk, privacy, $n,$ and $d$: as the privacy level increases (i.e. as $\varepsilon$ and $\delta$ shrink) and the dimension $d$ increases, the guaranteed excess adversarial risk (robustness) bound increases; and vice-versa. Furthermore, as the number of samples $n$ increase, it is possible to obtain more favorable adversarial excess risk and privacy guarantees for any fixed $d$. Note also that similar results could be written down in the non-smooth case as consequences of \cref{cor: sc upper bound}. Likewise, one can obtain excess adversarial risk bounds in the non-strongly convex case by using the regularization technique and invoking \cref{prop: convex ER} or \cref{prop: convex smooth ER}. We omit the details here.  

\vspace{0.2cm}
In the next subsection, we present a practical, efficient algorithm for implementing differentially private adversarial training with runtime bounds for attaining the above robustness guarantees. 

\subsubsection{Efficiently Implementing Output Perturbation for Private Adversarial Training}
\label[section]{adversarial implementation}
In order to practically implement the output perturbation mechanism, we make the following additional assumption: 
\begin{assumption}
\label[assumption]{assump: adversarial smoothness in v}
$H_D(w, \cdot)$ is $\btv$-smooth as a function of $\mathbf{v} \in S^n$ for all $w \in \WW$ and all $D = (X, \mathbf{y}) \in (\XX \times \mathcal{Y})^{n}.$
\end{assumption}
Then \cref{eq: adversarial training} is a smooth (in $w$ and $v$), strongly convex-concave min-max problem and there are efficient non-private algorithms for solving such problems \cite{nesterov07, alkousa20, lin20}. In \cref{alg: black box adv training imp}, we give a differentially private Black Box Algorithm tailored to the special min-max structure of the adversarial training objective.  We then instantiate the Algorithm with the near-optimal (in terms of gradient complexity) non-private algorithms of \cite{lin20} and provide upper bounds on the runtime for privately optimizing the adversarially robust model to match the levels of excess risk obtained in \cref{prop: fund EAR bounds for adv training smooth sc} (and/or the corresponding non-smooth version, which matches the bounds in \cref{cor: sc upper bound}.)  

\begin{algorithm}[hbt!]
\caption{Black Box Output Perturbation Algorithm for Implementing DP Adversarial Training}
\label{alg: black box adv training imp}
\begin{algorithmic}[1]
\Require~Number of data points $n \in \mathbb{N},$ dimension $d \in \mathbb{N}$ of feature data, non-private (possibly randomized) optimization method $\mathcal{M}$, privacy parameters $\varepsilon > 0, \delta \geq 0$, feature data universe $\XX \subseteq \mathbb{R}^d$ and label universe $\mathcal{Y},$ data set $D = (X, \mathbf{y}) \in (\XX \times \mathcal{Y})^{n}$, convex compact perturbation set $S \subset \mathbb{R}^d$ of diameter $\rho \geq 0,$ loss function $F(w,X,\mathbf{y})$ that is $L$-Lipschitz and $\mu$-strongly convex in $w$ on $B(0,R)$ and concave in $\mathbf{v} \in S^n$, accuracy parameter $\alpha > 0$ with corresponding iteration number $T = T(\alpha) \in \mathbb{N}$ such that $(w_T, v_T)$ is an $\alpha$-saddle point of $H_D(w,\mathbf{v}).$
\State Run $\MM$ for $T = T(\alpha)$ iterations to obtain an $\alpha$-saddle point $(w_{T}, v_{T})$ of $H_D(w,\mathbf{v}) = F(w, X + \mathbf{v}, \mathbf{y}).$
\State Add noise to ensure differential privacy: 
$\wpr(D):= \Pi_{\WW}(w_T + \widehat{z})$, where the density of $\widehat{z}$ is given as: 
$p_{\widehat{z}}(t) \propto 
\begin{cases}
 \exp\left(- \frac{\varepsilon\|t\|_2}{\Delta + 2 \sqrt{\frac{2\alpha}{\mu}}}\right) &\mbox{if } \delta = 0 \\
 \exp\left(- \frac{\varepsilon^2 \|t\|_2^2}{\left(\Delta + \frac{2\alpha}{\mu}\right)^2 \left(\cd + \sqrt{\cd^2 + \varepsilon}\right)^2}\right) &\mbox{if } \delta \in \left(0, \frac{1}{2}\right), \\
\end{cases}$  
and $\Delta := \frac{2L}{\mu n}.$ \\
\Return $\wpr(D).$
\end{algorithmic}
\end{algorithm}

\vspace{0.2cm}
The exact same arguments used in Section 3 show that \cref{alg: black box adv training imp} is $(\varepsilon, \delta)$-differentially private and that the following excess adversarial risk guarantees hold:

\begin{theorem}
\label{thm: robust smooth implementation}
Let $D = (X, \mathbf{y}) \in (\mathcal{X} \times \mathcal{Y})^{n}$ be a dataset
and let $\varepsilon > 0, \delta \in [0, \frac{1}{2}).$ Assume that $S$ is a convex compact set in $\mathbb{R}^d$ and that the ERM loss function $F$ is such that \cref{assump: strongly convex adversarial}, \cref{assump: concave adversarial}, and \cref{assump: adversarial smoothness in v} hold (see \cref{def: H adversarial}). Furthermore, assume that $H_D(\cdot, v)$ is $\beta$-smooth for all $\mathbf{v} \in S^n$.
Run $\mathcal{A} =$~\cref{alg: black box adv training imp} on $F$ with arbitrary inputs. \\
a) Let $\delta = 0$ and $\pen \leq 1.$ Then \[
\ERA \leq 4\beta \left(\frac{L}{\mu n} + \sqrt{\frac{\alpha}{\mu}}\right)^2 \left(\pe\right)^2 + \alpha.\] In particular, setting $\alpha = \frac{L^2}{\mu n^2} \min \{\kappa \left(\pe\right)^2, 1 \}$ gives \[
\ERA \leq 26 \frac{L^2}{\mu} \kappa \left(\pen\right)^2.\] 
b) Let $\delta\in (0, \frac{1}{2})$ and $\lpen \leq 1$. Then \[\ERDA \leq 2\beta \left(\frac{L}{\mu n} + \sqrt{\frac{\alpha}{\mu}}\right)^2 \left(\lpe\right)^2 + \alpha.\]
In particular, setting $\alpha = \frac{L^2}{\mu n^2} \min \left\{\kappa \left(\lpe\right)^2, 1 \right\}$ gives \[
\ERDA \leq 13.5 \lmu \kappa \left(\lpen\right)^2.\]
\end{theorem}

The theorem follows immediately from \cref{th: smooth sc implement}. Note that there is no explicit dependence on the size $\rho$ of the adversary's perturbation in the above adversarial risk bounds. However, the parameters $L, \beta, \btv,$ and $\mu$ (which we assumed to be fixed) may depend on $\rho.$ ~Furthermore, the iteration complexity and runtime depends on $\rho$ as we shall see below. There is also no dependence on $\btv$ in the above bounds, but this parameter does influence runtime as we will also see. Finally, note that a result similar to \cref{thm: robust smooth implementation} for \textit{non-smooth} (in $w$) $F$ can easily be written down as a direct consequence of \cref{thm: black box sc lip}.

\vspace{0.2cm}
Next, we instantiate our Black Box Algorithm with $\mathcal{M}$ = ``Minimax-APPA'' \cite[Algorithm 4]{lin20}), which combines an inexact accelerated proximal point method to minimize $\max_{\mathbf{v}} F(w, X + \mathbf{v}, \mathbf{y})$ with Nesterov's accelerated gradient method for solving the outer max problem. We refer the interested reader to \cite{lin20} for further details about the algorithm. We restate two results from \cite{lin20}, translated into our notation. 

\begin{proposition} (\cite[Thm 5.1/Cor 5.2]{lin20})
\label{lin}
Assume $H(\cdot, v)$ is $\mu$-strongly convex, $\beta$-smooth, (with condition number $\kappa = \beta / \mu$) for all $v \in S,$ where $S$ is a closed convex set with $L_2$ diameter $\rho.$ Assume $H(w, \cdot)$ is $\btv$-smooth and concave as a function of $v \in S$ for all $w \in \mathbb{R}^d.$ Then \\ 
1. Minimax-APPA returns an $\alpha$-saddle point of $H$ in at most $T =  \widetilde{O}\left(\sqrt{\frac{\kappa \btv}{\alpha}} \rho\right)$ total gradient evaluations. \\
2. If, in addition, $H(w, \cdot)$ is $\mu_{v}$-strongly concave, then Minimax-APPA returns an $\alpha$-saddle point in at most $T = \widetilde{O}(\sqrt{\kappa \kappa_{v}})$ gradient evaluations, where $\kappa_{v} = \frac{\beta_{v}}{\mu_{v}}$ is the condition number. 
\end{proposition}
Using the above result, together with \cref{thm: robust smooth implementation}, we obtain the following: 
\begin{corollary}
\label[corollary]{cor: smooth sc adv training imp aipp}
Let $D = (X, \mathbf{y}) \in (\mathcal{X} \times \mathcal{Y})^{n}$ be a dataset
and let $\varepsilon > 0, \delta \in [0, \frac{1}{2}).$ Assume that $S$ is a convex compact set in $\mathbb{R}^d$  of $L_2$ diameter $\rho$ and that the loss function $F$ is such that \cref{assump: strongly convex adversarial}, \cref{assump: concave adversarial}, and \cref{assump: adversarial smoothness in v} hold (see \cref{def: H adversarial}). Furthermore, assume that $H_D(\cdot, v)$ is $\beta$-smooth for all $\mathbf{v} \in S^n$ with condition number $\kappa = \beta/\mu$. Run \cref{alg: black box adv training imp} with  $\mathcal{M} = $ Minimax-AIPP. 
Suppose $F$ is of ERM form. \\
a)  Let $\delta = 0$ and $\pen \leq 1.$ Setting $\alpha = \lmu \frac{1}{n^2} \min \{\kappa \left(\pe\right)^2, 1 \}$ yields 
\[
\ERA \leq 26 \kappa \lmu \left(\pen\right)^2\]
in 
$T = \widetilde{O}\left(n^{3/2} \sqrt{\frac{\beta \btv}{L^2 \min\left\{\kappa \left(\pe\right)^2, 1\right\}}} \rho\right)$ gradient evaluations and runtime 
\[
\widetilde{O}\left(n^{5/2} d \sqrt{\frac{\beta \btv}{L^2 \min\left\{\kappa \left(\pe\right)^2, 1\right\}}}\rho\right).\]\\
b) Let $\delta \in \left(0, \frac{1}{2}\right)$ and $\lpen \leq 1.$ Setting $\alpha = \lmu \frac{1}{n^2} \min \left\{ \kappa \left(\lpe\right)^2, 1\right\}$ gives \[~\ERDA \leq 13.5 \kappa \lmu \left(\lpen\right)^2\] 
in 
$T = \widetilde{O}\left(n^{3/2} \sqrt{\frac{\beta \btv}{L^2 \min \left\{\kappa\left(\lpe\right)^2, 1 \right\}}} \rho \right)$ gradient evaluations and runtime
\[
\widetilde{O}\left(n^{5/2} d \sqrt{\frac{\beta \btv}{L^2 \min \left\{\kappa\left(\lpe\right)^2, 1 \right\}}} \rho \right).\]

If, in addition, $H_{D}(w, \cdot)$ is \textit{$\mu_{v}$-strongly concave} in $v$ for all $w \in \WW,$ then the above bounds are all attained and the gradient complexity improves to $T = \widetilde{O}\left(\sqrt{\kappa \kappa_{v}}\right)$ and runtime improves to $\widetilde{O}\left(nd \sqrt{\kappa \kappa_{v}}\right).$
\end{corollary}

Notice that the $L_2$ diameter of the cartesian product $S^n$ scales as $\sqrt{n}\rho$: this explains the presence of the extra $\sqrt{n}$ factor in the iteration complexity and runtime bounds (for non-strongly concave $H_D(w, \cdot)$) in \cref{cor: smooth sc adv training imp aipp} (since the adversary chooses $n$ perturbations $v_1, \ldots, v_n \in S$) as compared to \cref{lin}. Also, note that if $F(\cdot, X, y)$ is merely (non-strongly) convex, $L$-Lipschitz, and $\beta$-smooth, then applying the regularization technique (see \cref{subsection 2.4 convex lip}) can be used to derive similar results (with adversarial risk bounds resembling those in \cref{thm: black box convex smooth}). We omit the details here. 

\vspace{.2cm}
In the next and final subsection, we show how the above results for ERM $F$ can be used to yield excess adversarial population loss bounds via the results in \cref{Section 4: pop loss}. 

\subsubsection{From Adversarial Empirical Risk to Adversarial Population Loss}
Let $\DD$ be a distribution on $\cal{X} \times \cal{Y}$ and $F: \mathbb{R}^d \times (\XX + S) \times \cal{Y} \to \mathbb{R}$ be a loss function. In this subsection, we aim to understand how well \cref{alg: black box adv training imp} (run on the empirical loss with $n$ training samples $D = ((x_1, y_1, \cdots, x_n, y_n))$ drawn i.i.d from $\DD$) performs in minimizing the expected excess adversarial \textit{population} loss of $f$ on unseen data drawn from $\DD$. To clarify the presentation of our adversarial population loss results, we change some notation, similar to \cref{Section 4: pop loss}. 

\vspace{0.2cm}
\noindent \textbf{Change of Notations.} Denote the \textbf{expected adversarial population loss} by 
\[G_{\DD}(w) := \mathbb{E}_{(x,y) \sim \DD}[\max_{v \in S} f(w,x + v, y)].\] 
In other words, the adversarial population loss measures how well our model performs on unseen data against an adversary. To emphasize the distinction between this function and the \textit{empirical} adversarial loss with respect to a data set $D = ((x_1, y_1, \cdots, x_n, y_n)),$ we denote the latter by \[\widehat{G}_{D}(w) = \max_{v_{i} \in S^n, i \in [n]} \frac{1}{n} \sum_{i=1}^n f(w, x_{i} + v_{i}, y_{i}).\] 

\vspace{0.2cm}
Using the the results and methods from the preceding subsection and \cref{Section 4: pop loss}, we can obtain population adversarial loss bounds in comparable runtimes to those for empirical adversarial loss. Here we state bounds under the smoothness and strong convexity assumptions on $f(\cdot, x + v, y)$ (so that $\widehat{G}_{D} \in \HHE$). Bounds for the other three function classes studied in this paper can also be easily obtained, as consequences of our results from \cref{section 3: implment} and \cref{Section 4: pop loss}. To begin, we state the  excess adversarial population loss guarantee for \cref{alg: conceptual adv robust}, which is a direct consequence of \cref{prop: sc smooth pop loss}: 

\begin{proposition}
\label[proposition]{prop: sc smooth adversarial pop loss}
Assume $f(\cdot, x + v, y)$ is $\mu$-strongly convex, $L$-Lipschitz, and $\beta$-smooth (with condition number $\kappa = \beta / \mu$) for all $x + v \in \XX + S$ and $y \in \mathcal{Y}.$ Also assume $f(w, x + \cdot, y)$ is concave for all $w \in \WW, x \in \XX, y \in \mathcal{Y}.$ Let $\varepsilon > 0$ and $\delta \in [0, \frac{1}{2}).$ Let $D = (X, \mathbf{y})$ be a data set of size $n$ drawn i.i.d. according to $\mathcal{D}.$ Run \cref{alg: conceptual adv robust} on the empirical loss $\widehat{G}_D(w)$ to obtain $\wpr(D).$\\
a) If $\delta = 0$ and $\pen \leq 1,$ then 
\[
\ERAP \leq  \frac{L^2}{\mu}\left(\frac{2}{n} + 4\kappa \left(\pen\right)^2\right).\]
\\
b) If $\delta \in \left(0, \frac{1}{2}\right)$ and $\lpen \leq 1,$ then 
\[
\ERAP \leq \frac{L^2}{\mu}\left(\frac{2}{n} + 4\kappa \left(\lpen\right)^2\right).\]
\end{proposition}
Note that the hypotheses in \cref{prop: sc smooth adversarial pop loss} imply that $\widehat{G}_{D}(w)$ is $\mu$-strongly convex, $L$-Lipschitz, and $\beta$-smooth for all $D,$ which is why the result follows immediately from \cref{prop: sc smooth pop loss}. As in \cref{prop: fund EAR bounds for adv training smooth sc}, the assumption that $f(w, x + \cdot, y)$ is concave can be replaced with the weaker assumption that a saddle point of $\widehat{H}_{D}(w, \mathbf{v}):= \frac{1}{n} \sum_{i=1}^n f(w, x_i + v_i, y_i)$ exists. If we assume additionally that $f(w, x + \cdot, y)$ is $\btv$-smooth on $S$ for all $w \in \WW$ and all $(x,y) \in \XX \times \mathcal{Y}$ (c.f. \cref{assump: adversarial smoothness in v}), we obtain, via \cref{cor: smooth sc adv training imp aipp}, the following iteration complexity bounds: 
\begin{corollary}
Assume $f(\cdot, x + v, y)$ is $\mu$-strongly convex, $L$-Lipschitz, and $\beta$-smooth (with condition number $\kappa = \beta / \mu$) for all $x + v \in \XX + S$ and $y \in \mathcal{Y}.$ Assume also that $f(w, x + \cdot, y)$ is $\btv$-smooth and concave for all $w \in \WW, x \in \XX, y \in \mathcal{Y}$ (so that $\widehat{G}_{D}$ is also $\btv$-smooth). Let $\varepsilon > 0$ and $\delta \in [0, \frac{1}{2}).$ Let $D = (X, \mathbf{y})$ be a data set of size $n$ drawn i.i.d. according to $\mathcal{D}.$ Run \cref{alg: black box adv training imp} on $\widehat{G}_{D}$ with $\mathcal{M} = $ Minimax-AIPP to obtain $\wpr(D).$\\
a) Let $\delta = 0$ and $\pen \leq 1.$ Setting $\alpha = \lmu \frac{1}{n^2} \min \left\{\kappa \left(\pe\right)^2, 1 \right\}$ yields 
\[\ERAP \leq \lmu\left(\frac{5}{n} + 26 \kappa \left(\pen\right)^2\right)\]
in $T = \widetilde{O}\left(n^{3/2} \sqrt{\frac{\beta \btv}{L^2 \min\left\{\kappa \left(\pe\right)^2, 1\right\}}} \rho\right)$ gradient evaluations and runtime \[\widetilde{O}\left(n^{5/2} d \sqrt{\frac{\beta \btv}{L^2 \min\left\{\kappa \left(\pe\right)^2, 1\right\}}} \rho\right).\] 

b) Let $\delta \in \left(0, \frac{1}{2}\right)$ and $\lpen \leq 1.$ Setting $\alpha = \lmu \frac{1}{n^2} \min \left\{ \kappa \left(\lpe\right)^2, 1\right\}$ yields \[
\ERAP \leq \lmu\left(\frac{5}{n} + 13.5 \kappa \lmu \left(\lpen\right)^2\right)\] 
in $T = \widetilde{O}\left(n^{3/2} \sqrt{\frac{\beta \btv}{L^2 \min \left\{\kappa\left(\lpe\right)^2, 1 \right\}}} \rho \right)$ gradient evaluations \\
and runtime \[\widetilde{O}\left(n^{5/2} d \sqrt{\frac{\beta \btv}{L^2 \min \left\{\kappa\left(\lpe\right)^2, 1 \right\}}} \rho \right).\]

If, in addition, $f(w, x + \cdot, y)$ is \textit{$\mu_{v}$-strongly concave} for all $w \in \WW$ and $(x, y) \in \XX \times \cal{Y},$ then the above adversarial loss bounds are all attained and the gradient complexity improves to $T = \widetilde{O}\left(\sqrt{\kappa \kappa_{v}}\right)$ and runtime improves to $\widetilde{O}\left(nd \sqrt{\kappa \kappa_{v}}\right).$
\end{corollary}

\section*{Acknowledgements}
\addtocontents{toc}{\protect\setcounter{tocdepth}{0}}
We would like to thank Zeyuan Allen-Zhu, Larry Goldstein, and Adam Smith for helpful comments.

\newpage

\vskip 0.2in
\bibliographystyle{alpha}
\bibliography{references}
\newpage
\appendix
\addtocontents{toc}{\protect\setcounter{tocdepth}{1}}
\section{Summary Tables of Main Results (Excess Risk and Runtime Bounds)}
\label{sec: Appendix A: tables}

\begin{table}[h]
\resizebox{\textwidth}{!}{
\begin{tabular}{||c | c | c | c | c||}
 \hline 
 Function Class & Reference & Excess Risk & Runtime & Assumptions
 \\ [0.5ex]
 \hline\hline
TERM & \cref{thm: TERM implementation} & $O\left(\frac{L^2 C_{\tau}}{\mu} \pen \right)$ & $O\left(\frac{n^2 d}{C_{\tau}} \max\left\{\pe, \frac{n}{C_{\tau}}\right\}\right)$ & $f$ bounded on $\WW \times \XX$, $\pen \leq 1$
\\
\hline
Smooth TERM & \cref{thm: smooth TERM implementation} & $O\left(\frac{L^2 C_{\tau}}{\mu}\pen \min\left\{ 1, \kappa_{\tau}\pen\right\}\right)$ & $\widetilde{O}\left(nd \kappa_{\tau}\right)$ & $f$ bounded on $\WW \times \XX$, $\pen \leq 1$
 \\ [1ex] 
 \hline
\end{tabular}
}
\caption{$\delta = 0$. Tilted ERM (TERM) = $\ftau(w,X) = \frac{1}{\tau} \log\left(\frac{1}{n} \sum_{i=1}^{n} e^{\tau f(w, x_{i})}\right)$, where $f(\cdot, x)$ is $\mu$-strongly convex and $L$-Lipschitz. Smooth TERM assumes $f$ is $\beta$-smooth, $\mu$-strongly convex and $L$-Lipschitz. $C_{\tau} = e^{\tau(A_R - a_R)},$ where $a_R \leq f(w,x) \leq A_R$ for all $w,x \in \WW \times \XX.$ $\kappa_{\tau} = \frac{\tau L^2 + \beta}{\mu}.$
Note: All excess risk bounds should be read as $\min\{..., LR\}$ by taking the trivial algorithm that outputs $w_0 \in \WW$ for all $X \in \XX^n$. The assumptions (besides boundedness of $f$) are needed to ensure the excess risk bounds shown are non-trivial. 
}
\end{table}

\begin{table}[!htbp]
\resizebox{\textwidth}{!}{\begin{tabular}{||c | c | c | c | c||}
 \hline 
 Function Class & Reference & Excess Risk & Runtime & Assumptions \\ [0.5ex]
 \hline\hline
TERM & \cref{thm: TERM implementation} & $O\left(\frac{L^2 C_{\tau}}{\mu} \frac{\sqrt{d}}{\varepsilon n} \right)$ & $O\left(\frac{nd}{C_{\tau}} \max\left\{\frac{n^2}{C_{\tau}}, \lnep \right\} \right).$ & $f$ bounded on $\WW \times \XX$; $\widetilde{O}\left(\frac{\sqrt{d}}{\varepsilon}\right) \leq 1$
\\
\hline
Smooth TERM & \cref{thm: smooth TERM implementation} & $O\left(\frac{L^2 C_{\tau}}{\mu}\frac{\sqrt{d}}{\varepsilon n} \min\left\{ 1, \kappa_{\tau}\frac{\sqrt{d}}{\varepsilon n}\right\}\right)$ & $\widetilde{O}\left(nd \kappa_{\tau}\right)$ & $f$ bounded on $\WW \times \XX$; $\widetilde{O}\left(\frac{\sqrt{d}}{\varepsilon}\right) \leq 1$
 \\ [1ex] 
 \hline
\end{tabular}}
\caption{$\delta > 0$. Tilted ERM (TERM) = $\ftau(w,X) = \frac{1}{\tau} \log\left(\frac{1}{n} \sum_{i=1}^{n} e^{\tau f(w, x_{i})}\right) $, where $f(\cdot, x)$ is $\mu$-strongly convex and $L$-Lipschitz. Smooth TERM = $\ftau(w,X)$ for $\beta$-smooth, $\mu$-strongly convex and $L$-Lipschitz $f$. $C_{\tau} = e^{\tau(A_R - a_R)},$ where $a_R \leq f(w,x) \leq A_R$ for all $w,x \in \WW \times \XX.$ $\kappa_{\tau} = \frac{\tau L^2 + \beta}{\mu}.$ Notes: All excess risk bounds ignore the logarithmic factor $\cd + \sqrt{\cd^2 + \varepsilon},$ where $\cd = \sqrt{\log\left(\frac{2}{\sqrt{16\delta + 1} - 1}\right)}.$ In addition, all excess risk bounds should be read as $\min \{LR, ... \}$ by taking the trivial algorithm. The assumptions (besides boundedness of $f$ for TERM) are needed to ensure the excess risk bounds shown are non-trivial. 
\label{table: non-ERM, delta > 0} 
}
\end{table}

\begin{table}[t]
 \resizebox{\textwidth}{!}{\begin{tabular}{||c | c | c | c | c||}
 \hline 
 Function Class & Reference & Expected Excess Population Loss & Runtime & Assumptions \\ [0.5ex] 
 \hline\hline
 \multirow{2}{*}{$\FF$}
 & \cite{bst14} & $\widetilde{O}\left(\frac{L^2}{\mu}( \pen + \frac{1}{n})\right)$ & $O\left(\frac{L^2}{\mu^2} \frac{n d^9}{\varepsilon^2} \max \left\{1, \frac{L}{\mu} \right\} + \max \left\{dn^2, \varepsilon n \right\}\right)$ & $\pen \leq 1$
 \\ 
 \cline{2-5}
 & \cref{prop: sc lip pop loss}, \cref{rem: sc lip pop loss runtime}& $O\left(\frac{L^2}{\mu}\left(\pen + \frac{1}{n}\right)\right)$ & $O\left(nd \max \left\{n, \frac{\varepsilon}{d}\right\}\right)$ & $\pen \leq 1$
 \\ 
 \hline
 \multirow{3}{*}{$\HH$}
 & \cite{chaud2011} & $\widetilde{O}\left(LR (\frac{1}{\sqrt{n}} + \pen) + \beta R^2 \left(\pen\right)\right)$ & N/A & $\rank(\nabla^2 f(w,x)) \leq 1$ for all $w, x$, $\pen \leq 1$
 \\
 \cline{2-5}
& \cite{zhang2017} & $O\left(\frac{L^2}{\mu}\left(\kappa \left(\pen\right)^2\right) + \frac{1}{n})\right)$ & $\widetilde{O}(dn\kappa)$ & $\pen \leq 1$
\\
\cline{2-5}
& \cref{prop: sc smooth pop loss}, \cref{rem: min smooth sc pop} & $O\left(\frac{L^2}{\mu}(\min\{\kappa \left(\pen\right)^2, \pen\} + \frac{1}{n})\right)$ & $\widetilde{O}\left(d (n + \sqrt{n \kappa})\right)$ & $\pen \leq 1$
\\ [1ex] 
 \hline
\multirow{3}{*}{$\GG$}
 & \cite{bst14} & $\widetilde{O}\left(LR \left(\sqrt{\pen} + \frac{1}{\sqrt{n}}\right)\right)$ & $\widetilde{O}(R^2 d^6 n^3 \max\{d, \varepsilon n R\})$ & $\pen \leq 1$
 \\ 
 \cline{2-5}
 & \cref{prop: convex lip pop loss}, \cref{rem: convex lip pop loss runtime} & $O\left(LR\left(\sqrt{\pen} + \frac{1}{\sqrt{n}}\right)\right)$ & $O(n^2 d \max \{1, \left(\frac{\varepsilon}{d}\right)^2 \})$ & $\pen \leq 1$
 \\ 
 \hline
 $\JJ$
& \cref{prop: convex smooth pop loss}, \cref{rem: min smooth convex pop} & $O\left( \min \left\{\blrt\left(\left(\pen\right)^{2/3} + \frac{1}{\sqrt{n}}\right), LR\left(\sqrt{\pen} + \frac{1}{\sqrt{n}}\right)\right\}\right)$ & $\widetilde{O}\left(nd + \max \left\{n^{5/6} d^{2/3} \varepsilon^{1/3} \left(\frac{\beta R}{L}\right)^{1/3}, n^{3/4} d^{3/4} \varepsilon^{1/4} \sqrt{\frac{\beta R}{L}}  \right\}\right)$
& $\left(\pen\right)^2 \leq \frac{L}{\beta R}$
\\ [1ex] 

 \hline
 
\end{tabular}
}
\caption{Expected Excess Population Loss (SCO), $\delta = 0$. The function classes refer to $f$ and the excess population loss bounds refer to $F(w) = \mathbb{E}_{x \sim \DD} f(w,x).$ Function classes: $\FF$ = $L$-Lipschitz, $\mu$-strongly convex. $\HH$ = $L$-Lipschitz, $\mu$-strongly convex, $\beta$-smooth. $\GG$ = $L$-Lipschitz, convex. $\JJ$ = $L$-Lipschitz, convex, $\beta$-smooth. Note: All excess risk bounds should be read as $\min\{..., LR\}$ by taking the trivial algorithm that outputs $w_0 \in \WW$ for all $X \in \XX^n$ (e.g. if $\varepsilon n < d$). Our method yields the tightest expected excess population loss bounds and smallest runtime for each class. For our algorithm in the $L$-Lipschitz, convex, $\beta$-smooth class, the two runtimes in the $\max\{...\}$ correspond in order to the respective excess risk bounds in the $\min\{...\}$ term in that same row. See \cref{rem: min smooth convex pop} for details. ``N/A" means implementation and/or runtime details not provided for method: this applies to the conceptual objective perturbation method of \cite{chaud2011}. The result of \cite{chaud2011} (Thm 18) is proved only for linear classifiers, but seems it should be extendable to functions satisfying the rank-one Hessian assumption. In addition, \cite[Thm 18]{chaud2011} only gives high probability bounds, not expected excess risk bounds. Implementation details not given for the output perturbation/localization step (first step) of the exponential + localization algorithm of \cite{bst14}. However, the exponential sampling step of their method alone has worst-case runtime of $O(\frac{L^2}{\mu^2} \frac{ d^9}{\varepsilon^2} \max \{1, \frac{L}{\mu} \})$ and thus total runtime for their method, if localization were implemented using the methods of our paper, would exceed this quantity by an additive term that would be no smaller to our total runtime. If $\WW$ is in isotropic position, then runtime for exponential sampling improves by a factor of $O(d^3).$
\label{table:pop loss delta = 0}
}
\end{table}

\begin{table}[t]
\resizebox{\textwidth}{!}{\begin{tabular}{||c | c | c | c | c||}
 \hline 
 Function Class & Reference & Expected Excess Population Loss & Runtime & Assumptions/Restrictions \\ [0.5ex] 
 \hline\hline
 \multirow{4}{*}{$\HH$}
 & \cite{bst14} & $O\left(\frac{L^2}{\mu}\left(\frac{1}{n} + \frac{\sqrt{d}}{\varepsilon n} \log^2\left(\frac{n}{\delta}\right)\right)\right)$ & $O(n^2 d)$ & $\varepsilon \leq 2 \sqrt{\log\left(\frac{1}{\delta}\right)}, \frac{\sqrt{d \log(\frac{1}{\delta})}}{\varepsilon n} \leq 1$
 \\ 
 \cline{2-5}
 & \cite{zhang2017} & $O\left(\frac{L^2}{\mu}\left(\kappa \frac{d}{(\varepsilon n)^2} + \frac{1}{n}\right)\right)$ & $\widetilde{O}(\kappa nd)$ & $\varepsilon \leq 1, \frac{\sqrt{d \log(\frac{1}{\delta})}}{\varepsilon n} \leq 1$
 \\
 \cline{2-5}
 &\cite{fkt20} & $O\left(\frac{L^2}{\mu}\left(\frac{1}{n} + \frac{d}{(\varepsilon n)^2}\right)\right)$ & $\widetilde{O}(nd)$ & $\varepsilon \leq 1, \kappa \lesssim \frac{n}{\log(n)}, \frac{\sqrt{d \log(\frac{1}{\delta})}}{\varepsilon n} \leq 1$
 \\
 \cline{2-5}
& \cref{prop: sc smooth pop loss}, \cref{rem: smooth sc pop loss runtime katyusha} & $O\left(\frac{L^2}{\mu}\left(\min\left\{\kappa \frac{d}{(\varepsilon n)^2}, \frac{\sqrt{d}}{\varepsilon n} \right\} + \frac{1}{n}\right)\right)$ & $\widetilde{O}(d(n + \sqrt{n \kappa)})$ & $\lpen \leq 1$
\\
 \hline
\end{tabular}}
\caption{Expected Excess Population Loss (SCO), $\delta > 0$. 
The function classes refer to $f$ and the excess population loss bounds refer to $F(w) = \mathbb{E}_{x \sim \DD} f(w,x).$ Function class: $\HH$ = $L$-Lipschitz, $\mu$-strongly convex, $\beta$-smooth. All excess risk bounds should be read as $\min \{LR, ... \}$ by taking the trivial algorithm. The assumption $\frac{\sqrt{d}}{\varepsilon n} = \widetilde{O}(1)$ is needed to ensure the excess risk bounds shown are non-trivial. All excess risk bounds ignore the logarithmic factor $\log\left(\frac{1}{\delta}\right)$ or $\cd = \sqrt{\log\left(\frac{2}{\sqrt{16\delta + 1} - 1}\right)}.$ The $\widetilde{O}(...)$ notation is used here for runtime when there the runtime bound involves another logarithmic factor such as $\log(n)$ or $\log(d)$. Our algorithm is the only one that has no restrictions on $\varepsilon > 0.$  Noisy SGD (\cite{bst14}) doesn't require smoothness, but excess population loss bounds do not improve by adding smoothness. 
The algorithm of \cite{fkt20} that we list here is the phased ERM method (Algorithm 3), which requires $\varepsilon \leq 1.$ The authors present other algorithms which attain similar excess population loss and runtime bounds (up to log factors) for which the assumption is less strict, but to the best of our knowledge, arbitrary $\varepsilon > 0$ is not permitted with any of these, due to the suboptimal choice of Gaussian noise (see \cite{zhao2019}). Additionally, $\kappa \lesssim \frac{n}{\log(n)}$ is required for all of the algorithms in \cite{fkt20}. 
\label{table: pop delta > 0}
}
\end{table}

\begin{table}[t]
\resizebox{\textwidth}{!}{\begin{tabular}{||c | c | c | c | c||}
 \hline 
 Function Class & Reference & Excess Empirical Risk & Runtime & Assumptions/restrictions
 \\ [0.5ex]
 \hline\hline
 \multirow{2}{*}{$\FFE$}
 & \cite{bst14} & $\widetilde{O}\left(\frac{L^2}{\mu}\left(\pen\right)^2\right)$ & $O\left(\frac{L^2}{\mu^2} \frac{d^9}{\varepsilon^2} \max \{1, \frac{L}{\mu} \} + \max \{d n^2, \varepsilon n\}\right)$ & $\pen \leq 1$ 
 \\ 
 \cline{2-5}
 & \cref{cor: sc lip SGD} & $O\left(\frac{L^2}{\mu}\pen\right)$ & $O\left(\max \{d n^2, \varepsilon n\}\right)$ & $\pen \leq 1$ 
 \\ 
 \hline
 \multirow{4}{*}{$\HHE$}
 & \cite{chaud2011}
 & $\widetilde{O}\left(\frac{L^2}{\mu} \left(\pen\right)^2\right)$ & N/A & $\rank(\nabla^2 f(w,x)) \leq 1,$ $f$ linear classifier, $\pen \leq 1$ 
 \\
 \cline{2-5}
 & \cite{zhang2017} & $O\left(\frac{L^2}{\mu}\kappa \left(\pen\right)^2)\right)$ & $\widetilde{O}(\kappa nd)$ & $\pen \leq 1$ 
 \\
 \cline{2-5}
& \cref{cor: katyusha sc smooth}, \cref{rem: min smooth sc ERM} & $O\left(\frac{L^2}{\mu} \left(\pen\right) \min \left\{\kappa \left(\pen\right), 1 \right\}\right)$ & $\widetilde{O}\left(d (n + \sqrt{n \kappa})\right)$ & $\pen \leq 1$ 
\\ [1ex] 
 \hline
\multirow{2}{*}{$\GGE$}
 & \cite{bst14} & $\widetilde{O}\left(LR \pen\right)$ & $\widetilde{O}\left(R^2 d^6 n^3 \max\{d, \varepsilon n R\}\right)$ & $\pen \leq 1$ 
 \\ 
 \cline{2-5}
 & \cref{cor: convex lip SGD} & $O\left(LR \sqrt{\pen}\right)$ & $O\left(n^2 d \max \left\{1, (\frac{\varepsilon}{d})^2 \right\}\right)$ & $\pen \leq 1$
 \\ 
 \hline
 \multirow{3}{*}{$\JJE$}
 & \cite{chaud2011}
 & $\widetilde{O}\left(LR \pen\right)$ & N/A & $\rank(\nabla^2 f(w,x)) \leq 1,$ $f$ linear classifier, $\pen \leq 1$ 
\\
 \cline{2-5}
 & \cite{zhang2017} & $O\left(\blrt \left(\pen\right)^{2/3}\right)$ & $O\left(n^{5/3} d^{1/3} \varepsilon^{2/3} \left(\frac{\beta R}{L}\right)^{2/3}\right)$ & $\left(\pen\right)^2 \leq \frac{L}{\beta R}$
 \\
 \cline{2-5}
& \cref{cor: smooth convex katyusha}, \cref{rem: smooth convex ERM Katyusha min} & $O\left(\min \left\{ \blrt \left(\pen\right)^{2/3}, LR \sqrt{\pen} \right\}\right)$ & $\widetilde{O}\left(dn + \max \left\{n^{5/6} d^{2/3} \varepsilon^{1/3} \left(\frac{\beta R}{L}\right)^{1/3}, n^{3/4} d^{3/4} \varepsilon^{1/4} \sqrt{\frac{\beta R}{L}}  \right\}\right)$ & $\left(\pen\right)^2 \leq \frac{L}{\beta R}$
\\ [1ex] 
 \hline
\end{tabular}}
\caption{ERM, $\delta = 0$. ERM function classes: $\FFE$ = $L$-Lipschitz, $\mu$-strongly convex. $\HHE$ = $L$-Lipschitz, $\mu$-strongly convex, $\beta$-smooth. $\GGE$ = $L$-Lipschitz, convex. $\JJE$ = $L$-Lipschitz, convex, $\beta$-smooth. All excess risk bounds should be read as $\min \{LR, ... \}$ by taking the trivial algorithm (e.g. if $d > \varepsilon n$). The assumptions $\pen \leq 1$ and $\left(\pen\right)^2 \leq \frac{L}{\beta R}$ are needed to ensure the respective excess risk bounds shown are non-trivial. Our algorithm is the fastest in each respective function class. In the smooth, strongly convex, Lipschitz class, our excess risk bounds are also near-optimal, making it the clear preferred $\varepsilon$-differentially private algorithm for this class. For our algorithm in the $L$-Lipschitz, convex, $\beta$-smooth class, the two runtimes in the $\max\{...\}$ correspond in order to the respective excess risk bounds in the $\min\{...\}$ term in that same row. See \cref{rem: smooth convex ERM Katyusha min} for details. 
``N/A" means implementation and/or runtime details not provided for method. 
Implementation details not given for the output perturbation/localization step (first step) of the exponential + localization algorithm of \cite{bst14}. However, the second step of their method alone has runtime of $O(\frac{L^2}{\mu^2} \frac{n d^9}{\varepsilon^2} \max \{1, \frac{L}{\mu} \})$ and thus total runtime for their method, if localization were implemented using the methods of this paper, would exceed this quantity by an additive term comparable to our total runtime for that class. If $\WW$ is in isotropic position, then runtime of the exponential mechanism improves by a factor of $O(d^3).$
Objective perturbation bounds of \cite{chaud2011} are proved only for linear classifiers, but seem to be extendable to ERM functions satisfying the rank-one hessian assumption. 
\label{table:ERM delta = 0}
}
\end{table}

\begin{table}[t]
\resizebox{\textwidth}{!}{\begin{tabular}{||c | c | c | c | c||}
 \hline 
 Function Class & Reference & Excess Empirical Risk & Runtime & Assumptions
 \\ [0.5ex] 
 \hline\hline
 \multirow{5}{*}{$\HHE$}
 & \cite{kifer2012} & $O\left(\frac{L^2}{\mu} \frac{d}{(\varepsilon n)^2} + \beta R^2 \frac{1}{\varepsilon n}\right)$ & N/A & $\varepsilon \leq 1, \rank(\nabla^2 f(w,x)) \leq 1 \forall w, x,$ and $n \geq d^2$, $\frac{\sqrt{d \log(\frac{1}{\delta})}}{\varepsilon n} \leq 1$
 \\ [1ex] 
 \cline{2-5}
 & \cite{bst14} & $O\left(\frac{L^2}{\mu} \frac{d}{(\varepsilon n)^2} \log^2\left(\frac{n}{\delta}\right)\right)$ & $\widetilde{O}(d n^2)$ & $\varepsilon \leq 2 \sqrt{\log\left(\frac{1}{\delta}\right)}, \frac{\sqrt{d \log(\frac{1}{\delta})}}{\varepsilon n} \leq 1$
 \\ 
 \cline{2-5}
 & \cite{wang2017} & $O\left(\frac{L^2}{\mu} \frac{d}{(\varepsilon n)^2} \log(n)\right)$ & $\widetilde{O}(d(n + \kappa))$ & $\varepsilon = \widetilde{O}(\frac{\kappa}{n^2})$, $\frac{\sqrt{d \log(\frac{1}{\delta})}}{\varepsilon n} \leq 1$
 \\
 \cline{2-5}
 & \cite{zhang2017} & $O\left(\frac{L^2}{\mu}\kappa \frac{d}{(\varepsilon n)^2}\right)$ & $\widetilde{O}(\kappa nd)$ & $\varepsilon \leq 1$, $\frac{\sqrt{d \log(\frac{1}{\delta})}}{\varepsilon n} \leq 1$
 \\
 \cline{2-5}
& \cref{cor: katyusha sc smooth}, \cref{rem: min smooth sc ERM} & $O\left(\frac{L^2}{\mu}\min \left\{\kappa \frac{d}{(\varepsilon n)^2}, \frac{\sqrt{d}}{\varepsilon n} \right\}\right)$ & $\widetilde{O}(d(n + \sqrt{n \kappa)})$ & $\lpen \leq 1$
\\ 
 \hline
\end{tabular}}
\caption{ERM, $\delta > 0$. ERM function class: $\HHE$ = $L$-Lipschitz, $\mu$-strongly convex, $\beta$-smooth. All excess risk bounds should be read as $\min \{LR, ... \}$ by taking the trivial algorithm. The assumptions $\frac{\sqrt{d} \log\left(\frac{1}{\delta}\right)}{\varepsilon n} \leq 1$ or $\lpen \leq 1$ are needed to ensure the respective excess risk bounds shown are non-trivial. All excess risk bounds shown ignore the $\log\left(\frac{1}{\delta}\right)$ or $\cd + \sqrt{\cd^2 + \varepsilon}$ factor which is present in all of them. Here $\cd = \sqrt{\log(\frac{2}{\sqrt{16\delta + 1} - 1})}$. The $\widetilde{O}(...)$ notation is used here for runtime when there the runtime bound involves another logarithmic factor such as $\log(n)$ or $\log(d)$. Our algorithm is the only one that has no restrictions on $\varepsilon > 0.$  
Noisy SGD \cite{bst14} doesn't require smoothness, but excess risk bounds do not improve by adding smoothness. 
\cite[Remark 4.1]{wang2017} notes that the constraint $\varepsilon = \widetilde{O}(\frac{\kappa}{n^2})$ can be relaxed at the cost of an additional logarithmic factor in their excess risk bound, by amplifying the variance of their noise. However, their algorithm can never ensure $(\varepsilon, \delta)$-differential privacy for arbitrary $\varepsilon > 0$ due to their suboptimal choice of the Gaussian noise (c.f. \cite{zhao2019}). 
\label{table: ERM delta > 0}
}
\end{table}

\newpage
\FloatBarrier
\section{Proofs of Results in \cref{Sec: Conceptual risk bounds}}
\label{app: Sec 2 proofs}
We recall the following basic fact about $\beta$-smooth convex functions, which will be used frequently in many proofs:
\begin{lemma}(Descent Lemma)
Let $f$ be a convex $\beta$-smooth function on some domain $\mathcal{W}$. Then for any $w, w' \in \mathcal{W},$
\[ f(w) - f(w') \leq \langle \nabla f(w'), w - w' \rangle + \frac{\beta}{2}\|w - w'\|_2^2.
\]
\end{lemma}
\subsection{Proofs of Results in \cref{subsection 2.2}}
\subsubsection{Proof of \cref{conceptual sc alg is private}}
\label{pf: conceptual sc alg is private}
Recall the well-known post-processing property of differential privacy, which states that any function (e.g. projecting onto $\WW$) of an $(\varepsilon, \delta)$-differentially private method is itself $(\varepsilon, \delta)$-differentially private \cite[Proposition 2.1]{dwork2014}. By this fact, it suffices to show that the algorithm $\mathcal{A}(X) = \ws(X) + z$ (without projection) is differentially private. Assume first $\delta = 0$. Then by the definition of differential privacy, it suffices to show that for any $s \in \range(\wpr)$ (for which $\widehat{p}'(s) \neq 0$) and any $X, X' \in \XX^n$ such that $|X \Delta X'| \leq 2,$
\[
\frac{\widehat{p}(s)}{\widehat{p}'(s)} \leq e^{\varepsilon},
\]
where $\widehat{p}$ and $\widehat{p}'$ are the probability density functions (pdfs) of $\mathcal{A}(X)$ and $\mathcal{A}(X')$ respectively. Now note that $\widehat{p}(s) = p_{z}(s - \ws(X))$ and $\widehat{p}'(s) = p_{z}(s - \ws(X')),$ where $p_{z}$ is the pdf of the noise, given above. Then \begin{align*}
\frac{\widehat{p}(s)}{\widehat{p}'(s)}  &= \frac{p_{z}(s - \ws(X))}{p_{z}(s - \ws(X'))} = \exp\left({\frac{-\varepsilon\|s - \ws(X)\|_2 + \varepsilon \|s - \ws(X')\|_2}{\Delta_{F}}}\right) \\&\leq \exp\left({\frac{\varepsilon\|\ws(X) - \ws(X')\|_2}{\Delta_{F}}}\right) \leq \exp({\varepsilon}),
\end{align*}
where the second to last line uses the reverse triangle inequality and the last line uses the definition of $\Delta_{F}$. This establishes that $\mathcal{A}$ is $(\varepsilon, 0)$-differentially private. Now, by the post-processing property \cite{dwork2014}, we conclude that the algorithm $\mathcal{A'}(X) = \wprpro(X) = \Pi_{\WW}(\mathcal{A}(X))$ is $(\varepsilon, 0)$- differentially private. \\

For $\delta > 0$, the proof follows directly from the following result of \cite{zhao2019}: 
\begin{theorem} (Theorem 5 in \cite{zhao2019})
For $\delta \in (0, \frac{1}{2}),$ $(\varepsilon, \delta)$- differentially privacy can be achieved by adding Gaussian noise with mean $0$ and standard deviation $\sigma = \frac{(c + \sqrt{c^2 + \varepsilon}) \Delta}{\varepsilon \sqrt{2}}$ to each dimension of a query with $L_2$ sensitivity $\Delta.$ Here $c = \sqrt{\log\left(\frac{2}{\sqrt{16\delta + 1} - 1}\right)}.$
\end{theorem}
Here, our ``query" is $\ws(X),$ which has sensitivity $\Delta_{F},$ and hence $\mathcal{A}$ is $(\varepsilon, \delta)$. Applying the post-processing property again shows that $\mathcal{A'}$ is $(\varepsilon, \delta)$-differentially private. 

\subsection{Proofs of Results in \cref{subsection 2.3: sc Lip ER bounds}}
\subsubsection{Proof of \cref{prop:sc sensitivity}}
The proof relies on the following generalization (to non-differentiable functions on a possibly constrained ($\mathcal{W} \neq \mathbb{R}^d$) domain) of Lemma 7 from \cite{chaud2011}.
\begin{lemma}
\label[lemma]{lemma7}
Let $G(w), g(w)$ be functions on some convex closed set $\mathcal{W} \subseteq \mathbb{R}^d$. Suppose that $G(w)$ is $\mu$-strongly convex and $G(w) + g(w)$ is convex. Assume further that $g$ is $L_{g}$-Lipschitz. Define $w_{1} = \arg\min_{w \in \mathcal{W}} G(w), w_{2} = \arg\min_{w \in \mathcal{W}} [G(w) + g(w)]$. Then $\|w_{1} - w_{2}\|_2 \leq \frac{L_{g}}{\mu}.$
\end{lemma}
\begin{proof}[Proof of \cref{lemma7}]
At a point $w \in \mathcal{W},$ let $h(w)$ and $H(w)$ output any subgradients of $g$ and $G$, respectively. By first-order optimality conditions, for all $w \in \mathcal{W}$, we have \[
    \langle H(w_1), w - w_1 \rangle \geq 0\] and \[
    \langle H(w_2), w - w_2 \rangle + \langle h(w_2), w - w_2 \rangle \geq 0.
\]
Plugging $w_2$ for $w$ in the first inequality and $w_1$ for $w$ in the second inequality, and then subtracting inequalities gives: \begin{equation}
\label{eqn: FOO}
    \langle H(w_1) - H(w_2), w_1 - w_2 \rangle \leq \langle h(w_2), w_1 - w_2  \rangle. 
\end{equation}
Now, by strong convexity of $G$, we have \[
\mu \|w_1 - w_2 \|_2^2 \leq \langle H(w_1) - H(w_2), w_1 - w_2 \rangle. 
\]
Combining this with \cref{eqn: FOO} and using Cauchy-Schwartz yields: 
\[
\mu \|w_1 - w_2 \|_2^2 \leq \langle H(w_1) - H(w_2), w_1 - w_2 \leq \langle h(w_2), w_1 - w_2  \rangle \leq \|h(w_2)\|_2 \|w_1 - w_2\|_2. 
\]
Finally, using $L_{g}$-Lipschitzness of $g$ and dividing the above inequality by $\mu \|w_1 - w_2 \|_2$ gives the result. 
\end{proof}

Now we can prove \cref{prop:sc sensitivity}.
\begin{proof}[Proof of \cref{prop:sc sensitivity}]
Let $X, X' \in \XX^n$ such that (WLOG) $x_{n} \neq x'_{n}$, but all other data points are the same. 
Apply \cref{lemma7} to $G(w) = F(w,X), g(w) = F(w,X') - F(w,X)$ and note that $g$ is $2L$-Lipschitz by the triangle inequality. For $F$ of ERM form, we have $g(w) = \frac{1}{n}[f(w, x'_{n}) - f(w, x_n)]$, which is $2 \frac{L}{n}$-Lipschitz. 
\end{proof}

\subsubsection{Proof of \cref{prop: sc ER}}
We begin with a simple lemma. 
\begin{lemma} 
\label[lemma]{lem: expec of gamma and gauss}
1. Let $z$ be a random vector in $\mathbb{R}^d$ with density $p_z(t) \propto \exp(-\frac{\varepsilon \|t\|_{2}}{\Delta}).$ \\
    Then $\mathbb{E}\|z\|^2 = \frac{d(d+1) \Delta^2}{\varepsilon^2}.$ \\
2. Let $b \sim N(0, \mathbf{I}_{d} \sigma^2)$ (Gaussian random vector in $\mathbb{R}^d$). Then $~\mathbb{E}\|b\|_2^2 = d\sigma^2.$
\end{lemma}

\begin{proof}
\begin{enumerate}
    \item Observe that $\|z\|_2 \sim \Gamma(d, \frac{\Delta}{\varepsilon})$, which has mean $\mathbb{E}\|z\|_2 = d \frac{\Delta}{\varepsilon}$ and variance $\mathbb{E}\|z\|_2^2 - (\mathbb{E}\|z\|_2)^2 = d (\frac{\Delta}{\varepsilon})^2$. Now solve for $\mathbb{E}\|z\|_2^2$. 
    \item Next, $b$ consists of $d$ independent Gaussian random variables (i.i.d) $b_{i} \sim N(0, \sigma^2).$ Hence $\mathbb{E}[\|b\|_2^2] = \sum_{i = 1}^d \mathbb{E}b_{i}^2 = d\sigma^2.$
\end{enumerate}
\end{proof}

We also recall the following basic fact about projections onto closed, convex sets:
\begin{lemma}
\label[lemma]{projection lemma}
Let $\mathcal{W}$ be a closed, convex set in $\mathbb{R}^d,$ and let $a \in \mathbb{R}^d.$
Then $\pi = \Pi_{\mathcal{W}}(a)$ if and only if $\langle a - \pi, w - \pi \rangle \leq 0$ for all $w \in \mathcal{W}.$
\end{lemma}
\begin{proof}
Fix any $a \in \mathbb{R}^{d}.$ By definition, $\Pi_{\mathcal{W}}(a) = \arg\min_{w \in \mathcal{W}} \frac{1}{2}\|w - a\|_2^2 := \arg\min_{w \in \mathcal{W}} h(w).$ Then $\nabla h(w) = w - a$ and by first-order optimality conditions, $\pi = \arg\min_{w \in \mathcal{W}} h(w)$ if and only if for all $w \in \mathcal{W},$  \[
\langle \nabla h(\pi), w - \pi \rangle \geq 0 \Leftrightarrow \langle \pi - a, w - \pi \rangle \geq 0 \Leftrightarrow \langle a - \pi, w - \pi \rangle \leq 0.
\]
\end{proof}

\begin{proof}[Proof of \cref{prop: sc ER}]
By $L$-Lipschitzness of $F,$ $\ERpro \leq L \mathbb{E} \ |\Pi_{\WW}(\ws(X) + z) - \ws(X) \|_2$. 
Applying \cref{projection lemma} to $a = \ws(X) + z$ and $\pi = \Pi_{\WW}(\ws(X) + z)$ and using Cauchy-Schwartz gives \begin{align*}
\|\ws(X) - \Pi_{\WW}(\ws(X) + z)\|_2^2 &\leq \langle \Pi_{\WW}(\ws(X) + z) - \ws(X), z \rangle \\ &\leq \|\Pi_{\WW}(\ws(X) + z) - \ws(X)\|_2 \|z\|_2. 
\end{align*}
Dividing both sides of this inequality by $\|\ws(X) - \Pi_{\WW}(\ws(X) + z)\|$ shows that $|\Pi_{\WW}(\ws(X) + z) - \ws(X) \|_2 \leq \|z\|_2$ for any $z \in \mathbb{R}^{d}.$ Hence $\ERpro \leq L \mathbb{E} \|z\|_2.$ Since $\|z\|_2 \sim \Gamma(d, \frac{\Delta_{F}}{\varepsilon})$ has mean $\mathbb{E} \|z\|_2 = \frac{d \Delta_{F}}{\varepsilon},$ part a) follows. For part b), combine the second part of \cref{lem: expec of gamma and gauss} with Jensen's inequality. This completes the proof. 
\end{proof}

\subsection{Proofs of Results in \cref{subsection 2.4 convex lip}}
\subsubsection{Proof of \cref{prop: convex ER}}
\label[proof]{Proof of Prop 8}
We will prove part 2 (ERM) here; the proof of part 1 is almost identical, but delete the $n$ term everywhere it appears to account for the different sensitivity estimate in \cref{prop:sc sensitivity}. 
To begin, decompose excess risk into 3 terms and then bound each one individually: \begin{align*}
\mathbb{E}_{\mathcal{A'_{\lambda}}}F(\wlapro(X), X) - F(\ws(X), X) = \underbrace{\mathbb{E}_{A'_{\lambda}} [F(\wla(X), X) - \Fl(\wla(X), X)]}_{\textcircled{\small{i}}} \\
+ \underbrace{\mathbb{E}_{A'_{\lambda}} \Fl(\wla(X), X) - \Fl(\wls(X), X)}_{\textcircled{\small{ii}}} 
+ \underbrace{\Fl(\wls(X), X) - F(\ws(X), X)}_{\textcircled{\small{iii}}}. 
\end{align*}
Suppose first $\delta = 0$. Then we bound the terms as follows: \textcircled{\small{i}} = $-\frac{\lambda}{2} \mathbb{E}\|\wla(X) \|_2^2 \leq 0$; \textcircled{\small{ii}} $\leq \frac{(L + \lambda R) \Delta_{\lambda} d}{\varepsilon}$ by \cref{prop: sc ER}; and \textcircled{\small{iii}} $ = [\Fl(\wls(X), X) - \Fl(\ws(X),X)] + [\Fl(\ws(X), X) - F(\ws(X), X)] < 0 + \frac{\lambda}{2} \|\ws(X)\|_2^2 \leq \frac{\lambda R^2}{2}$ since $\wls(X)$ is the unique minimizer of $\Fl(\cdot, X)$ by strong convexity. 
By \cref{prop:sc sensitivity}, we have $\Delta_{\lambda} \leq \frac{2(L + \lambda R)}{\lambda}$ for $F \in \GG$ and $\Delta_{\lambda} \leq \frac{2(L + \lambda R)}{\lambda n}$ for $F \in \GGE$. 
For $\delta = 0$, combining the above gives us: 
\[
\ERLpro \leq \frac{\lambda R^2}{2} + \frac{2(L+\lambda R)^2}{\lambda}\left(\pen\right) \]
for $F \in \GGE,$ and \[
\ERLpro \leq \frac{\lambda R^2}{2} + \frac{2(L+\lambda R)^2}{\lambda}\left(\pe\right) \]
for $F \in \GG$. Minimizing these expressions over $\lambda > 0$ yields (up to constant factors) $\lambda^{*} = \frac{L}{R \sqrt{1 + \frac{\varepsilon n}{d}}}$ for $F \in \GGE$ and $\lambda^{*} = \frac{L}{R \sqrt{1 + \frac{\varepsilon}{d}}}$ for $F \in \GG$. Plugging in these choices of $\lambda^*$ and using elementary algebraic inequalities gives the results for $\delta = 0$.  \\
For $\delta > 0$, we instead bound the three terms in the above decomposition as follows: 
\textcircled{\small{i}} $\leq 0$, as before; \textcircled{\small{ii}} $\leq (L + \lambda R)\sqrt{d} \sigma_{\lambda} =  (L + \lambda R)\sqrt{d} \left(\frac{\cd + \sqrt{\cd^2 + \varepsilon}}{\sqrt{2} \varepsilon}\right) \Delta_{\lambda}  
$; 
and \textcircled{\small{iii}} $\leq \frac{\lambda R^2}{2}$, as before. Using \cref{prop:sc sensitivity}, we get \[
\ERDLpro \leq \frac{\lambda R^2}{2} + \frac{2(L+\lambda R)^2}{\lambda}\left(\lpen\right)
\] 
for the ERM case; delete the $n$ for non-ERM. Minimizing over $\lambda > 0$ yields (up to constants) $\lambda^* = \frac{L}{R \sqrt{1 + \frac{\varepsilon n} {d^{1/2} \left(\ld\right)}}}$ for ERM $F$. Plugging this in and using elementary algebraic inequalities (e.g. note $L + \lambda^* R \leq 2L$) and the assumptions in the statement of the proposition finishes the proof. Delete the $n$ everywhere and follow line for line for non-ERM.   
\subsubsection{Proof of \cref{lem: smooth ER}}
By the descent lemma, \begin{align*}
\ER &= \mathbb{E}_{\mathcal{A}}[F(\ws(X) + z), X) - F(\ws(X), X)] \\
&\leq \mathbb{E}_{\mathcal{A}}\langle \nabla_{w} F(\ws(X), X), \ws(X) + z - \ws(X)\rangle \\
&+ \frac{\beta}{2} \mathbb{E}_{\mathcal{A}}\|\ws(X) + z - \ws(X)\|_{2}^2 \\
&= \frac{\beta}{2}\mathbb{E}_{\mathcal{A}} \|z\|_2^2.
\end{align*}
 Now recall (\cref{lem: expec of gamma and gauss}) that for $\delta = 0$, $\mathbb{E}_{\mathcal{A}} \|z\|_2^2 = \frac{d(d+1) \Delta_{F}^2}{\varepsilon^2}$, whereas for $\delta > 0$, $\mathbb{E}_{\mathcal{A}} \|z_{\delta}\|_2^2 = d \sigma^2 = 
\frac{d}{2} \left(\frac{\cd + \sqrt{\cd^2 + \varepsilon}}{\varepsilon}\right)^2 \Delta_{F}^2.$
Combining the above gives the result. 

\subsubsection{Proof of \cref{prop: lower bound for HHE}}
The proof is an extension of the argument used in \cite{bst14}. Define the following unscaled hard instance: \[
f(w,x) := \frac{1}{2}\|Mw - x\|_2^2, 
\]
for $w, x \in B(0,1) \subset \mathbb{R}^d,$ where $M$ is a symmetric positive definite $d \times d$ matrix satisfying $\frac{\mu}{\beta} \mathbf{I} \preccurlyeq M^2 \preccurlyeq \mathbf{I}.$ Note that this implies $\sqrt{\frac{\mu}{\beta}} \mathbf{I} \preccurlyeq M \preccurlyeq \mathbf{I}.$ Then for any $w, x \in B(0,1) \subset \mathbb{R}^d$, \[
\nabla f(w,x) = M^T(Mw - x) \implies \|\nabla f(w,x)\|_2 \leq \|M^2\|_2 \|w\|_2 + \|M^T\|_2 \|x\|_2 \leq 2,
\]
so $f(\cdot, x)$ is 2-Lipschitz,
and \[
\frac{\mu}{\beta} \mathbf{I} \preccurlyeq \nabla^2 f(w,x) = M^2 \preccurlyeq \mathbf{I},
\]
so $f$ is $\frac{\mu}{\beta}$-strongly convex and $1$-smooth. 
Hence $F(w,X):= \frac{1}{n} \sum_{i=1}^n f(w,x_i) \in \mathcal{H}^{ERM}_{1, \frac{\mu}{\beta}, 2, 1}.$ 
We will require the following lemma from \cite{bst14}:
\begin{lemma} \cite[Lemma 5.1]{bst14}
\label[lemma]{lemma 5.1 bst}
\begin{enumerate}
    \item Let $n, d \in \mathbb{N}, \varepsilon > 0.$ For every $\varepsilon$-differentially private algorithm $\mathcal{A},$ there is a data set $X = (x_1, \cdots, x_n) \in B(0,1)^n \subset \mathbb{R}^{d \times n}$ such that with probability at least $1/2$ (over the randomness in $\mathcal{A}$), we have \[
    \|\mathcal{A}(X) - \Bar{X}\|_2 = \Omega\left(\min\left\{1, \pen \right\}\right),
    \]
    where $\Bar{X}:= \frac{1}{n} \sum_{i=1}^n x_{i}.$
    \item Let $n, d \in \mathbb{N}, \varepsilon > 0, \delta = o(1/n).$ For every $(\varepsilon, \delta)$-differentially private algorithm $\mathcal{A},$ there is a data set $X = (x_1, \cdots, x_n) \in B(0,1)^n \subset \mathbb{R}^{d \times n}$ such that with probability at least $1/3$ (over the randomness in $\mathcal{A}$), we have \[
    \|\mathcal{A}(X) - \Bar{X}\|_2 = \Omega\left(\min\left\{1, \frac{\sqrt{d}}{\varepsilon n} \right\}\right),
    \]
    where $\Bar{X}:= \frac{1}{n} \sum_{i=1}^n x_{i}.$
\end{enumerate}
\end{lemma}
Note that for any $(\varepsilon, \delta)$-differentially private algorithm $\mathcal{A},$ the algorithm $M^{-1} \mathcal{A}$ is also $(\varepsilon, \delta)$-differentially private by post-processing \cite{dwork2014}. Moreover, since $M$ is a bijection, every $(\varepsilon, \delta)$-differentially private algorithm $\mathcal{A}$ can be written as $M \mathcal{A'}$ for $(\varepsilon, \delta)$-differentially private algorithm $\mathcal{A'}.$ Thus, we can replace $\|\mathcal{A}(X) - \Bar{X}\|_2$ by $\|M \mathcal{A}(X) - \Bar{X}\|_2$ in the statement of \cref{lemma 5.1 bst}.

Now observe that for any $w$ and $X = (x_1, \cdots, x_n),$ \[
\nabla F(w,X) = \frac{1}{n} \sum_{i=1}^n M^T(Mw - x_i) \implies \ws(X) = (M^T M)^{-1} M^T \bar{X} = M^{-1} \Bar{X}, \] and
\begin{align*}
F(w,X) - F(\ws(X), X)  &= \frac{1}{2n}\left[n\|Mw\|_2^2 - 2\sum_{i=1}^n\left(\langle Mw, x_i \rangle - \langle \Bar{X} , x_i\rangle\right) - n\|\Bar{X}\|_2^2 \right]\\
&= \frac{1}{2} \left[\|Mw\|_2^2 -2 \langle Mw, \Bar{X}\rangle + \frac{1}{n^2}\|\sum_{i=1}^n x_i\|_2^2 \right] \\
&= \frac{1}{2}\|Mw - \Bar{X}\|_2^2.
\end{align*}

Combining this with \cref{lemma 5.1 bst} and the discussion that followed, and applying Markov's inequality implies that for any $\varepsilon$-differentially private algorithm $\mathcal{A},$ there exists a data set $X \in B(0,1)^n$ such that \begin{equation}
\label{eq: unscaled LB delta = 0}
    \ER = \Omega\left(\min\left\{1, \left(\pen\right)^2 \right\}\right).
\end{equation}
Likewise, for any $(\varepsilon, \delta)$-differentially private algorithm $\mathcal{A},$ there exists a data set $X \in B(0,1)^n$ such that \begin{equation}
\label{eq: unscaled LB delta > 0}
     \ER = \Omega\left(\min\left\{1, \left(\frac{\sqrt{d}}{\varepsilon n}\right)^2 \right\}\right).
\end{equation}

It remains to properly scale $F.$ To this end, define $\mathcal{W} = \XX = B(0,R)$ and $\widetilde{f}: \WW \times \XX \to \mathbb{R}$ by $\widetilde{f}(\widetilde{w},\widetilde{x}) := \frac{\beta}{2}\|M \widetilde{w} - \widetilde{x}\|_2^2,$ where $\widetilde{w} = Rw$ and $\widetilde{x} = Rx.$ Then $\widetilde{F}(\widetilde{w}, \widetilde{X}):= \frac{1}{n}\sum_{i=1}^n \widetilde{f}(\widetilde{w},\widetilde{x_i}) \in \HHE$ with $L = 2 \beta R$ and \[
\widetilde{F}(\widetilde{w}, \widetilde{X}) - \widetilde{F}(\widetilde{w}^*(\widetilde{X}), \widetilde{X}) = \frac{\beta}{2}\|M\widetilde{w} - \Bar{\widetilde{X}}\|_2^2 = \frac{\beta R^2}{2}\|Mw - \Bar{X}\|_2^2
,
\]
where $\widetilde{w}^*(\widetilde{X}) = \argmin_{\widetilde{w} \in \WW} \widetilde{F}(\widetilde{w}, \widetilde{X})$ and $\Bar{\widetilde{X}} = \frac{1}{n}\sum_{i=1}^n \widetilde{x_{i}}.$ Therefore, applying the unscaled lower bounds above from \cref{eq: unscaled LB delta = 0} and \cref{eq: unscaled LB delta > 0} completes the proof.

\subsubsection{Proof of \cref{prop: convex smooth ER}}
As in the proof of \cref{prop: convex ER}, we decompose excess risk into three terms and then bound each one individually: \begin{align*}
\mathbb{E}_{\mathcal{A}} F(\wla(X), X) - F(\ws(X), X) = \underbrace{\mathbb{E}_{A} [F(\wla(X), X) - \Fl(\wla(X), X)]}_{\textcircled{\small{a}}} \\+ 
\underbrace{\mathbb{E}_{A} F(\wla(X), X) - \Fl(\wls(X), X)}_{\textcircled{\small{b}}} 
+ \underbrace{\Fl(\wls(X), X) - F(\ws(X), X)}_{\textcircled{\small{c}}}. 
\end{align*}

We prove part 2 (ERM) first. Assume first that $\delta = 0$. Then we bound the terms as follows: \textcircled{\small{a}} = $-\frac{\lambda}{2} \mathbb{E}\|\wla(X)\|_2^2 \leq 0$; \textcircled{\small{b}} $\leq (\beta + \lambda) (2\frac{L + \lambda R}{\lambda})^2 \left(\pen\right)^2 $ by $(\beta + \lambda)$-smoothness and $(L + \lambda R)$-Lipschitzness of $\Fl$, and \cref{lem: smooth ER} combined with \cref{prop:sc sensitivity}; and \textcircled{\small{c}} $ = [\Fl(\wls(X), X) - \Fl(\ws(X),X)] + [\Fl(\ws(X), X) - F(\ws(X), X)] < 0 + \frac{\lambda}{2} \|\ws(X)\|_2^2 \leq \frac{\lambda R^2}{2}$ since $\wls(X)$ is the unique minimizer of $\Fl(\cdot, X)$ by strong convexity. Putting these estimates together, setting $\lambda = \left(\frac{\beta L^2}{R^2}\right)^{1/3} \left(\pen\right)^{2/3} 
$, noting $\lambda R \leq L$ by the parameter assumptions and hence $\lambda \leq 2 \beta$ (by the inequality $L \leq 2 \beta R$ which holds for all $F \in \JJ$) yields 
\begin{align*}
    \ER &\leq 
     \frac{\lambda R^2}{2} + 4 (\beta + \lambda) \left(\pen\right)^2 \frac{(L+\lambda R)^2}{\lambda^2}  \\
     &\leq \frac{\lambda R^2}{2} + \frac{4 (3 \beta) (2L)^2}{\lambda^2}\left(\pen\right)^2 \\
     &\leq 48.5 \blrt\left(\pen\right)^{2/3}.
\end{align*}
Now suppose $\delta > 0$. Then the \textcircled{\small{b}} term in the decomposition is bounded as \[
\textcircled{\small{b}} \leq 4(\beta + \lambda) \left(\frac{L + \lambda R}{\lambda}\right)^2 \left(\lpen\right)^2 
\] by $(\beta + \lambda)$-smoothness and $(L + \lambda R)$-Lipschitzness of $\Fl$, and \cref{lem: smooth ER} combined with \cref{prop:sc sensitivity}. Then the other two terms are bounded as above for $\delta = 0$. Set $\lambda = \left(\frac{\beta L^2}{R^2}\right)^{1/3} \left(\lpen\right)^{2/3} 
$. 
Again note that $\lambda R \leq L$ and $\lambda \leq 2\beta$ by the parameter assumptions and the inequality $L \leq 2 \beta R$ for $F \in \JJ$. Then, by similar steps as above, we have \begin{align*}
  \ERD &\leq \frac{\lambda R^2}{2} + 4\left(\beta + \lambda\right)\left(\frac{L + \lambda R}{\lambda}\right)^2 \left(\lpen\right)^2 \\
  &\leq \frac{\lambda R^2}{2} + 48 \frac{\beta L^2}{\lambda^2}\left(\lpen\right)^2 \\
  &\leq 48.5 \blrt\left(\lpen\right)^{2/3}.
\end{align*}

To prove part 1, set $\lambda = \left(\frac{\beta L^2}{R^2}\right)^{1/3} \left(\pe\right)^{2/3}$ when $\delta = 0$ and $\lambda = \left(\frac{\beta L^2}{R^2}\right)^{1/3} \left(\lpe\right)^{2/3}$ for $\delta > 0$. Then the proof of part 1 follows line for line as the proof of part 2, except with ``$n$" deleted everywhere.

\section{Proofs for \cref{section 3: implment} and Additional Results for Non-Strongly-Convex}
\subsection{Proofs of Results in \cref{subsection: 3.1}}
\label{app: proofs of sec 3.1}
\subsubsection{Proof of \cref{prop: sc imp priv}}
Define $\Delta_{T} := \sup_{|X \Delta X'|\leq 2} \|w_{T}(X) - w_{T}(X')\|_{2}.$ By standard arguments (see proof of \cref{conceptual sc alg is private}) and our choice of noise, it suffices to show that $\Delta_{T} \leq \Delta_{F} + 2\sqrt{\frac{2 \alpha}{\mu}}$. 
Now, since $F(w_{T}(X), X) - F(\ws, X) \leq \alpha$ by the choice of $T = T(\alpha)$, and since $F(\cdot, X)$ is $\mu$-strongly convex, we have 
\[ \|w_{T} - \ws\|_2^2 \leq \frac{2}{\mu} [F(w_{T}, X) - F(\ws, X) - \langle \nabla F(\ws, X), w_{T} - \ws \rangle] \leq \frac{2}{\mu} \alpha. 
\]
Hence for any data sets $X, X' \in \XX^n$ such that $|X \Delta X'| \leq 2$, we have \begin{align*}
  \|w_{T}(X) - w_{T}(X') \| &\leq \|w_{T}(X) - \ws(X) \| + \|\ws(X) - \ws(X')\| + \|\ws(X') - w_{T}(X')\| \\
  &\leq \sqrt{\frac{2}{\mu} \alpha} + \Delta_{F} + \sqrt{\frac{2}{\mu} \alpha}.  
\end{align*}

\subsubsection{Proof of \cref{thm: black box sc lip}}
Notice \begin{align*}
\ERpro = \mathbb{E}F(\Pi_{\WW}(w_{T}(X) + \widehat{z}), X) - F(\ws(X), X) \\
= \mathbb{E}[F(\Pi_{\WW}(w_{T}(X) + \widehat{z}), X) - F(w_{T}(X), X)]
+ \mathbb{E}[F(w_{T}(X), X) - F(\ws(X), X)] \\
\leq L \mathbb{E}\|\widehat{z}\|_2 + \alpha,
\end{align*}
where we used \cref{projection lemma} and the non-expansiveness property of projection in the last inequality.
Now for $\delta = 0,$ $\ERpro \leq \sqrt{2} L\left(\Delta_{F} + 2 \sqrt{\frac{2 \alpha}{\mu}}\right)\left(\pe\right) + \alpha$ by \cref{lem: expec of gamma and gauss}. Then using \cref{prop:sc sensitivity} to bound $\Delta_{F}$ proves the first statement in each of 1a) and 2a). Similarly, for $\delta > 0$, we have \[
\ERDpro \leq L\left(\Delta_{F} + 2 \sqrt{\frac{2 \alpha}{\mu}}\right) \left(\lpe\right) + \alpha
\] 
by \cref{lem: expec of gamma and gauss}. Appealing again to \cref{prop:sc sensitivity} completes the proof. The verification that the prescribed choices of $\alpha$ achieve the respective conceptual upper bounds is routine, using the assumptions stated in the theorem.

\subsubsection{Proof of of \cref{cor: sc lip SGD}}
In line 4 of \cref{alg: stochastic subgrad}, note that $\mathbb{E}_{i \sim unif[n]}\|g(w,x_{i})\|_2^2 \leq L^2$ since $f(\cdot,x)$ is $L$-Lipschitz for all $x \in \XX.$ Also, $\mathbb{E}_{i \sim unif[n]}[g(w,x_{i})] = G(w, x_{i}) \in \partial F(w,X).$ Therefore, the result follows from \cite[Theorem 6.2]{bubeck}, which says that $F(\widehat{w_{T}}, X) - F(\ws(X), X) \leq \frac{2L^2}{\mu (T+1)},$ where $\widehat{w_{T}}$ is the weighted average of the first $T$ iterates $w_{t}, t \in [T]$ returned by \cref{alg: stochastic subgrad}. 

\subsubsection{Proof of \cref{th: smooth sc implement}}
Write \begin{align*}
    \ER &= \mathbb{E}F(w_{T}(X) + \widehat{z}, X) - F(\ws(X), X) \\
&= \mathbb{E}F(w_{T}(X) + \widehat{z}, X) - F(w_{T}(X), X) 
+ F(w_{T}(X), X) - F(\ws(X), X) \\ 
&\leq \mathbb{E} \langle \nabla F (w_T(X), X), \widehat{z} \rangle + \mathbb{E}\|\widehat{z}\|_2^2 + \alpha \\
&= \frac{\beta}{2} \mathbb{E}\|\widehat{z}\|_2^2 + \alpha,
\end{align*}
where we used the descent lemma to get the inequality. Now for $\delta = 0$, \[
\ER \leq \beta \left(\Delta_{F} + 2 \sqrt{\frac{2 \alpha}{\mu}}\right)^2 \left(\pe\right)^2 + \alpha.
\] by \cref{lem: expec of gamma and gauss}. 
~Substituting the bounds on $\Delta_{F}$ from \cref{prop:sc sensitivity} proves the first statement in each of 1a) and 2a). 
Similarly, for $\delta > 0$, \[
~\ER \leq \frac{\beta}{2} \left[\frac{d \left(\ld\right)^2 \left(\Delta_{F} + 2 \sqrt{\frac{2 \alpha}{\mu}}\right)^2}{\varepsilon^2}\right] + \alpha.
\]
Again appealing to \cref{prop:sc sensitivity} establishes the first statements in 1b) and 2b). Verification that the prescribed choices of $\alpha$ achieve the claimed bounds is straightforward. 

\subsection{Efficient Implementation for Convex, Lipschitz functions}
\label{subsection 3.3}
In this subsection, we proceed similarly as in \cref{subsection 2.4 convex lip}, where we used regularization to obtain a strongly convex objective $\Fl(w,X) = F(w,X) + \frac{\lambda}{2} \|w\|_2^2$, optimized $\Fl$ using our techniques for strongly convex functions, and then add noise (calibrated to the sensitivity of $\Fl$) to ensure privacy. The key difference here is that, as in the previous subsection, we do not obtain $\wls(X)$ exactly, but rather get an approximate minimizer. Our Black Box implementation procedure is described in \cref{alg: black box convex non-smooth}. Our first result is that this method is differentially private. 

\FloatBarrier
\begin{algorithm}[ht!]
\caption{Black Box Output Perturbation Algorithm with Regularization for $\GG$ and $\JJ$}
\label{alg: black box convex non-smooth}
\begin{algorithmic}[1]
\Require~ Number of data points $n \in \mathbb{N},$ dimension $d \in \mathbb{N}$ of data, non-private (possibly randomized) optimization method $\mathcal{M}$, privacy parameters $\varepsilon > 0, \delta \geq 0$, data universe $\XX,$ data set $X \in \XX^{n}$, function $F(w,X) \in \GG$, accuracy and regularization parameters $\alpha > 0, \lambda > 0$ with corresponding iteration number $T = T(\alpha, \lambda) \in \mathbb{N}$ (such that $\mathbb{E}[\Fl(w_{T}(X), X) - \Fl(\ws(X), X)] \leq \alpha$).
\State Run $\MM$ on $\Fl(w,X) = F(w,X) + \frac{\lambda}{2}\|w\|_2^2$ for $T = T(\alpha)$ iterations to ensure $\mathbb{E}\Fl(w_{T}(X), X) - \Fl(\wls, X) \leq \alpha$. 
\State Add noise to ensure privacy: $\wprl:= w_{T} + \hat{\zl}$, where the density $p(\hat{\zl})$ of $\hat{\zl}$ is proportional to 
$\begin{cases} 
\exp\left\{-\frac{\varepsilon \|\hat{\zl}\|_{2}}{\Delta_{\lambda} + 2\sqrt{\frac{2 \alpha}{\lambda}}}\right\} &\mbox{if } \delta = 0 \\
\exp\left\{-\frac{2 \varepsilon^2 \|\hat{\zl}\|_{2}^2}{(\Delta_{\lambda} + 2\sqrt{\frac{2 \alpha}{\lambda}})^2 (\cd + \sqrt{\cd^2 + \varepsilon})^2}\right\} &\mbox{if } \delta > 0,
\end{cases}$ 
where $\Delta_{\lambda}:= \frac{2(L + \lambda R)}{\lambda n}$
is an upper bound on the $L_2$ sensitivity of $\Fl.$ \\
\Return $\wprl.$
\end{algorithmic}
\end{algorithm}

\begin{proposition}
\label[proposition]{prop: convex imp priv}
\cref{alg: black box convex non-smooth} is $(\varepsilon, \delta)$-differentially private. Moreover, $\mathcal{A'}_{\lambda}(X) := \Pi_{\WW}\left(w_{\mathcal{A}_{\lambda}}(X)\right)$ is $(\varepsilon, \delta)$-differentially private, where $w_{\mathcal{A}_{\lambda}(X)}$ denotes the output of \cref{alg: black box convex non-smooth}.
\end{proposition}

Next, we establish excess risk bounds for the generic \cref{alg: black box convex non-smooth} that depend on $\alpha,$ and show that the conceptual excess risk bounds from \cref{subsection 2.4 convex lip} can be obtained by choosing $\alpha$ appropriately. 

\begin{theorem}
\label[theorem]{thm: black box imp convex lip}
Suppose $F \in \GGE$. Run \cref{alg: black box convex non-smooth} on $F \in \GGE$ with $\lambda > 0$ chosen as follows: If $\delta = 0,$ set $\lambda = \frac{L}{R \sqrt{1 + \frac{\varepsilon n}{d}}}$.
If $\delta \in \left(0, \frac{1}{2}\right),$ set $\lambda =
\frac{L}{R \sqrt{1 + \frac{\varepsilon n}{d^{1/2} (\cd + \sqrt{\cd^2 + \varepsilon})}}}.$
Denote $\mathcal{A'}_{\lambda}(X) = w_{\mathcal{A'}_{\lambda}}(X):=  \Pi_{\WW}\left(w_{\mathcal{A}_{\lambda}}(X)\right),$ where $\wprl(X)$ is the output of \cref{alg: black box convex non-smooth}. \\
a) Let $\delta = 0, \pen \leq 1.$  
Then \[
\ERpro \leq 16.5 LR \sqrt{\pen} + 32\sqrt{\alpha} \sqrt{LR} (\frac{\varepsilon n}{d})^{1/4}(1 + \pe).\] In particular, setting $\alpha = \frac{LR \left(\pen\right)^{3/2}}{(1 + \pe)^2}$ implies \[
\ERLpro \leq 49 LR \sqrt{\pen}.\] \\
b) Let $\delta \in (0,\frac{1}{2}), \lpen \leq 1$. 
Then \[
\ERDpro \leq 13 LR \left(\lpen\right)^{1/2} +  12 \sqrt{\alpha} \sqrt{LR} \left(\frac{\varepsilon n}{\sqrt{d} (\cd + \sqrt{\cd^2 + \varepsilon})}\right)^{1/4}.\]
In particular, setting $\alpha = LR \left(\lpen\right)^{3/2}$ implies \[
\ERDLpro \leq 25 LR \left(\lpen\right)^{1/2}.\]
\end{theorem}

Instantiating \cref{alg: black box convex non-smooth} with fast stochastic optimization algorithms allow us to achieve the above excess risk bounds efficiently. Note that $\Fl(\cdot, X)$ is $(L + \lambda R)$-Lipschitz and $\lambda$-strongly convex. Recall that in this case, $T = \ceil*{\frac{2 (L+\lambda R)^2}{\lambda \alpha}}$ iterations of stochastic subgradient descent are sufficient for finding a point ${w_{T}}$ such that $\mathbb{E}\Fl(w_{T}, X) - \Fl(\wls(X), X) \leq \alpha$. Plugging in the value of $\lambda$ prescribed in \cref{thm: black box imp convex lip}, we have the following:

\begin{corollary}
\label[corollary]{cor: convex lip SGD}
Let $F \in \GGE.$ Run \cref{alg: black box convex non-smooth} with $\MM$ as the stochastic subgradient method and set $\lambda$ as specified in \cref{thm: black box imp convex lip}, $\alpha$ and $T$ as prescribed below. Denote $\mathcal{A'}_{\lambda}(X) = w_{\mathcal{A'}_{\lambda}}(X):=  \Pi_{\WW}\left(w_{\mathcal{A}_{\lambda}}(X)\right),$ where $\wprl(X)$ is the output of \cref{alg: black box convex non-smooth}. \\
a) Let $\delta = 0$ and $\pen \leq 1.$ Then setting $\alpha = \frac{LR \left(\pen\right)^{3/2}}{(1 + \pe)^2}$ and $T = \ceil*{12 n^2 (1 + \left(\ep\right)^2)}$ implies \[
\ERLpro \leq 49 LR \sqrt{\pen}\] in runtime $
    O\left(d n^2 \max\left\{1, \left(\frac{\varepsilon}{d}\right)^2 \right\}\right).$
b) Let $\delta \in \left(0, \frac{1}{2}\right)$ and  $\lpen \leq 1.$ Then setting $\alpha = LR \left(\lpen\right)^{3/2}$ and $T = \ceil*{12 \frac{n^2 \varepsilon}{d \ld})}$  implies \[
\ERDLpro \leq 25 LR \left(\lpen\right)^{1/2}\] in runtime $O(n^2 \varepsilon).$
\end{corollary}

There are two existing algorithms for $\GGE$ that we are aware of that provide excess risk and runtime bounds. For $\delta = 0,$ the exponential mechanism of \cite{bst14} has excess risk $O\left(LR \left(\pen\right)\right)$ and runtime $\widetilde{O}\left(R^2 d^6 n^3 \max\{d, \varepsilon n R\}\right),$ which is clearly not very practical for large-scale problems. For $\delta > 0,$ the noisy SGD method of \cite{bst14} has excess risk $O\left(LR \left(\frac{\sqrt{d} \log(n) \log\left(\frac{1}{\delta}\right)}{\varepsilon n}\right)\right)$ and runtime $O(n^2 d).$ This method requires $\varepsilon \leq 2 \sqrt{\log\left(\frac{1}{\delta}\right)}$ to be private, however, which restricts how small excess risk can get. Although both of these excess empirical risk bounds are tighter than the ones we provide via output perturbation in most parameter regimes, in \cref{Section 4: pop loss}, we prove tighter \textit{population loss} bounds (which is arguably more important than empirical risk) than the two competing algorithms of \cite{bst14} noted above.

Next, we assume additionally that $F$ is $\beta$-smooth. 
\subsubsection{Smooth, convex, Lipschitz functions.}
\label{subsection 3.4: smooth convex lip imp}
When we add the smoothness assumption, we need to change the regularization parameter in \cref{alg: black box convex non-smooth} in order to obtain tighter excess risk bounds. 

\begin{theorem}
\label[theorem]{thm: black box convex smooth}
$F \in \JJE.$ Run \cref{alg: black box convex non-smooth} with accuracy $\alpha$ and set $\lambda$ as follows: if $\delta = 0$, set $\lambda = 
\left(\frac{\beta L^2}{R^2}\right)^{1/3} \left(\pen\right)^{2/3}.$ If $\delta \in (0,\frac{1}{2})$, set $\lambda = 
\left(\frac{\beta L^2}{R^2}\right)^{1/3} \left(\lpen\right)^{2/3}.$ \\
a) Let $\delta = 0$ and $~\left(\pen\right)^2 \leq \lrb.$ Then, \[
\ER \leq 48.5 \blrt\left(\pen\right)^{2/3} + \alpha\left[3\left(\frac{R \beta}{L}\right)^{2/3} \frac{d^{4/3} n^{2/3}}{\varepsilon^{4/3}} + 1\right] + 12\sqrt{\alpha}R \beta^{1/2} \frac{d}{\varepsilon}.\]
In particular, setting $\alpha = \min \left\{\frac{L^{4/3} R^{2/3}}{\beta ^{1/3}} (\frac{\varepsilon}{d})^{2/3} \frac{1}{n^{4/3}}, \blrt\left(\pen\right)^{2/3}\right\} $ implies \[
\ER \leq 65 \blrt \left(\pen\right)^{2/3}.
\]
b) Let $\delta \in (0,\frac{1}{2})$ and $~\left(\lpen\right)^2 \leq \lrb.$ Then, \begin{align*}
\ERD &\leq 48.5 \blrt \left(\lpen\right)^{2/3} \\
&\;\;\; + \alpha\left[1 + 8n^{2/3}\left(\frac{\beta R}{L}\right)^{2/3} \frac{(d(\ld)^2)^{2/3}}{\varepsilon^{4/3}}\right] + 69 \sqrt{\alpha} R \sqrt{\beta} \frac{\sqrt{d} (\ld)}{\varepsilon}.\end{align*} 
In particular, setting $\alpha = \min\left\{\frac{L^{4/3} R^{2/3}}{\beta ^{1/3}} \left(\lnep\right)^{2/3} \frac{1}{n^2}, \blrt \left(\lpen\right)^{2/3} \right\}$ implies \[
\ERD \leq 127 \blrt \left(\lpen\right)^{2/3}.\]
\end{theorem}

For efficient private optimization, we can instantiate our Black Box Algorithm with an accelerated stochastic method, such Katyusha \cite{az16}, on the regularized objective. Recall that Katyusha finds an $\alpha$-suboptimal point of $\Fl(\cdot, X)$ in runtime  $\widetilde{O}\left(d\left(n + \sqrt{n} \sqrt{\frac{\beta + \lambda}{\lambda}}\right)\right)$ (where $\widetilde{O}$ hides a log factor depending on $\frac{1}{\alpha}$) since $\Fl$ is $(\beta + \lambda)$-smooth and $\lambda$-strongly convex. Putting in the choices of $\lambda$ and $\alpha$ prescribed by \cref{thm: black box convex smooth} gives the following:

\begin{corollary}
\label[corollary]{cor: smooth convex katyusha}
For $F \in \JJE$, take $\MM$ to be \cref{alg: Katyusha} (Katyusha) in \cref{alg: black box convex non-smooth} and $\lambda$ as prescribed in \cref{thm: black box convex smooth}, $\alpha$ and $T$ as given below. \\
a) If $\delta = 0$ and $\left(\pen\right)^2 \leq \frac{L}{R \beta}$ then taking $\alpha = \frac{L^{4/3} R^{2/3}}{\beta ^{1/3}} (\frac{\varepsilon}{d})^{2/3} \frac{1}{n^{4/3}}$ and \[
T = O\left(n + n^{5/6} \left(\ep\right)^{1/3} \left(\frac{\beta R}{L}\right)^{1/3} \log\left(\left(\frac{\beta R}{L}\right)^{1/3} \left(\pe\right)^{2/3} n^{4/3}\right)\right)\] implies \[
\ER \leq 65 \blrt \left(\frac{d}{\varepsilon n}\right)^{2/3}\]
in runtime \[
\widetilde{O}\left(nd + n^{5/6} d^{2/3} \varepsilon^{1/3}\left(\frac{\beta R}{L}\right)^{1/3}\right).\]
b) If $\delta \in \left(0, \frac{1}{2}\right)$ and $\left(\lpen\right)^2 \leq \frac{L}{R \beta},$ then taking $\alpha = \frac{L^{4/3} R^{2/3}}{\beta ^{1/3}} \left(\lnep\right)^{2/3} \frac{1}{n^2},$ and $~T = O\left(n + n^{5/6} \left(\lep\right)^{1/3} \left(\frac{\beta R}{L}\right)^{1/3} \log\left(\left(\frac{\beta R}{L}\right)^{1/3} n^2 \left(\lpen\right)^{2/3}\right) \right)$ implies \[
\ERD \leq 127 \blrt \left(\lpen\right)^{2/3}\] in runtime \[
\widetilde{O}\left(nd + n^{5/6} d^{5/6} \varepsilon^{1/6}\left(\frac{\beta R}{L}\right)^{1/3}\right).\]
\end{corollary}

\begin{remark}
\label[remark]{rem: smooth convex ERM Katyusha min}
Since $\GGE \subseteq \JJE$, by \cref{cor: convex lip SGD}, we obtain the following upper bounds on the excess risk of functions in $\JJE: \ER \lesssim
\min \left\{LR, LR \sqrt{\pen}, \blrt \left(\pen\right)^{2/3} \right\}$ for $\delta = 0$;
and $~\ERD \lesssim \min \left\{LR, LR \left(\lpen\right)^{1/2}, \blrt \left(\lpen\right)^{2/3} \right\}$ for $\delta > 0.$ 

Now, to derive upper bounds on the runtimes for obtaining the second (middle) excess risk bounds in the above minima (i.e. the excess risk bounds for $\GGE$) under the additional assumption of $\beta$-smoothness, we use the regularization term $\lambda$ and the corresponding $\alpha$ that are prescribed in \cref{thm: black box imp convex lip}. For $\delta = 0,$ plugging these choices of $\lambda$ and $\alpha$ into the runtime expression for Katyusha gives the following runtime for achieving excess risk $O\left(LR \sqrt{\pen}\right)$: $\widetilde{O}\left(d\left(n + \sqrt{n\left(\frac{\beta + \lambda}{\lambda}\right)}\right)\right) = \widetilde{O}\left(d\left(n + n^{3/4} \left(\frac{\varepsilon}{d}\right)^{1/4} \sqrt{\frac{\beta R}{L}}\right)\right) = \widetilde{O}\left(nd + n^{3/4}d^{3/4}\varepsilon^{1/4}\sqrt{\frac{\beta R}{L}}\right).$ Likewise, for $\delta > 0,$ the runtime for obtaining excess risk $O\left(LR \left(\lpen\right)^{1/2}\right)$ by Katyusha is $\widetilde{O}\left(d\left(n + n^{3/4} \left(\frac{\varepsilon}{\ld}\right)^{1/4} \sqrt{\frac{\beta R}{L}}\right)\right) = \\
\widetilde{O}\left(nd + n^{3/4} \varepsilon^{1/8} d^{7/8} \sqrt{\frac{\beta R}{L}}\right).$ 
\end{remark}

There are several alternative efficient differentially private algorithms for smooth convex ERM that have been proposed previously, but \textit{our method is faster and arguably simpler} than all of these. Here we discuss only the methods that have practical implementation schemes and runtime guarantees (see \cref{subsection 2.5: smooth sc} and \cref{subsection 2.6: convex smooth} for discussion of objective perturbation) and those that require smoothness (see \cref{subsection 3.3} for discussion of exponential mechanism and noisy SGD), to avoid redundancy. The gradient descent-based output perturbation method of \cite{zhang2017} achieves excess risk of $O\left(\blrt \left(\pen\right)^{2/3}\right)$ in runtime $O\left(n^{5/3} d^{1/3} \varepsilon^{2/3} \left(\lrb\right)^{2/3}\right)$ when $\delta = 0$. When $\delta > 0$, they get excess empirical risk $\widetilde{O}\left(\blrt \left(\frac{\sqrt{d}}{\varepsilon n}\right)^{2/3}\right)$ in runtime $\widetilde{O}\left(n^{5/3} d^{2/3} \varepsilon^{2/3} \left(\frac{\beta R}{L}\right)^{2/3} \right),$ but they require $\varepsilon \leq 1$ in order to be differentially private (by their choice of Gaussian noise and \cite{zhao2019}). Thus, by \cref{rem: smooth convex ERM Katyusha min}, our excess risk bounds match or beat (when $\beta$ or $R$ is very large) theirs in every parameter regime (see Section 2.5 for details). In addition, our runtime is significantly faster by way of using Katyusha instead of gradient descent. Indeed, when our risk bounds match, we improve runtime by a factor of $O\left(n^{5/6} \frac{\beta R}{L} \frac{\varepsilon}{d}\right)^{1/3}$ for $\delta = 0$ and $\widetilde{O}\left(n^{5/6} \frac{\beta R}{L} \lep\right)^{1/3}$ for $\delta > 0,$ while also permitting any $\varepsilon > 0.$ 
For $\delta > 0,$ the other notable method for $\JJE$ is noisy SVRG of \cite{wang2017}. This method achieves nearly optimal excess empirical risk of $\widetilde{O}\left(LR \frac{\sqrt{d}}{\varepsilon n}\right)$ in runtime $\widetilde{O}\left(n\left(\left(\frac{\beta R}{L}\right) \varepsilon \sqrt{d} + d\right)\right),$ provided $\beta$ is ``sufficiently large'' \cite[Theorem 4.4]{wang2017}. By contrast, our method works for arbitrarily small $\beta$ and (ignoring log terms) is faster whenever $d^2 \leq n \varepsilon^5 \left(\frac{\beta R}{L}\right)^4$ due to acceleration.

\vspace{0.2cm}
In the next section, we use the ERM results from this Section (and Section 2) to derive strong differentially private guarantees on the excess \textit{population loss} of Lipschitz convex (and strongly convex) loss functions. The population loss bounds are realized in as little runtime as was needed to obtain the ERM risk bounds in this Section, making them the fastest known runtimes for differentially private optimization when $\delta = 0$ and nearly the fastest for $\delta > 0$. 

\subsection{Proofs of Results in \cref{subsection 3.3}}
\subsubsection{Proof of \cref{prop: convex imp priv}}
By standard arguments, it suffices to show that $\sup_{|X \Delta X'| \leq 2} \|w_{T}(X) - w_{T}(X') \|_2 \leq \Delta_{\lambda} + 2\sqrt{\frac{2 \alpha}{\mu}}.$ Let $|X \Delta X'| \leq 2$. Then \begin{align*}
\|w_{T}(X) - w_{T}(X') \|_2  &\leq \|w_{T}(X) - \wls(X)\| + \|\wls(X) - \wls(X') \| + \|\wls(X') - w_{T}(X')\| \\
&\leq \sqrt{\frac{2 \alpha}{\lambda}} + \Delta_{\lambda} + \sqrt{\frac{2 \alpha}{\lambda}}. 
\end{align*}
The last line follows because: $\lambda$-strong convexity of $\Fl$ implies the middle term is bounded according to \cref{prop:sc sensitivity} and the other two terms are both bounded by a similar argument as in \cref{prop: sc imp priv}, again using $\lambda$-strong convexity of $\Fl$. Finally, the post-processing property of differential privacy \cite[Proposition 2.1]{dwork2014} implies that $\mathcal{A}'$ is also differentially private. 

\subsubsection{Proof of \cref{thm: black box imp convex lip}}
To begin, decompose excess risk into 3 terms: 
\begin{align*}
\begin{split}
&\ERLpro = \\
&\underbrace{\mathbb{E}_{\mathcal{A'}_{\lambda}} [F(\Pi_{\WW}(w_{T}(X) + \zl), X) - \Fl(\Pi_{\WW}(w_{T}(X) + \zl), X)]}_{\textcircled{\small{a}}} \\
&+ \underbrace{\mathbb{E}_{\mathcal{A'}_{\lambda}} \Fl(\Pi_{\WW}(w_{T}(X) + \zl), X) - \Fl(\wls(X), X)}_{\textcircled{\small{b}}}  
+ \underbrace{\Fl(\wls(X), X) - F(\ws(X), X)}_{\textcircled{\small{c}}}. 
\end{split}
\end{align*}

Now bound the terms as follows: \textcircled{\small{a}} = $-\frac{\lambda}{2} \mathbb{E}\|\wlapro(X)\|_2^2 \leq 0$; \[\textcircled{\small{b}} \leq (L + \lambda R) \mathbb{E}\|\Pi_{\WW}(w_{T}(X) + \zl) - \wls(X)\| \leq (L + \lambda R)\left[\sqrt{\frac{2 \alpha}{\lambda}} + \mathbb{E}\|\zl\|_{2}\right],\]
using contractivity of the projection map, $\lambda$-strong convexity of $\Fl$ and our choice of $T = T(\alpha).$ Also, \begin{align*}
    \textcircled{\small{c}} = [\Fl(\wls(X), X) - \Fl(\ws(X),X)] + [\Fl(\ws(X), X) - F(\ws(X), X)] &< \\
    0 + \frac{\lambda}{2} \|\ws(X)\|_2^2 \leq \frac{\lambda R^2}{2} \end{align*}
since $\wls(X)$ is the unique minimizer of $\Fl(\cdot, X)$ by strong convexity. Hence $\ERLpro \leq (L + \lambda R)\left[\sqrt{\frac{2 \alpha}{\lambda}} + \mathbb{E}\|\zl\|_{2}\right] + \frac{\lambda R^2}{2}$. \\
Assume that $F \in \GGE$: the proof of the non-ERM case is very similar (but simpler since $n$ disappears).
First suppose $\delta = 0$, so $\mathbb{E}\|\zl\|_{2} \leq \frac{\sqrt{2} d (\Delta_{\lambda} + 2\sqrt{\frac{2 \alpha}{\lambda}})}{\varepsilon}.$  Then $\Dl \leq  \frac{2(L+ \lambda R)}{\lambda n} \leq \frac{4L}{\lambda n} \leq 4\sqrt{2} R \sqrt{\frac{\varepsilon}{n d}}$ by our choice of $\lambda$ and assumption that $\pen \leq 1$. Also, $\frac{1}{\lambda} < \sqrt{2}\frac{R}{L} \sqrt{\frac{d}{\varepsilon n}}.$ Using these estimates and the above estimate for excess risk gives 
\begin{multline}
    \ERLpro \leq \\
    2L\left[2\sqrt{\alpha} \sqrt{\frac{R}{L}} (\nep)^{1/4} + \sqrt{2}\frac{d}{\varepsilon}\left(4 \sqrt{2} R \sqrt{\frac{\varepsilon}{n d}} + 4 \sqrt{2}\sqrt{\alpha} \sqrt{\frac{R}{L}} \left(\nep\right)^{1/4}\right)\right] + LR \frac{\sqrt{\pen}}{2}\\
    \leq 16 \sqrt{\alpha} \sqrt{LR} (\nep)^{1/4}(1 + \pe) + 16.5 LR \sqrt{\pen},
\end{multline}
as claimed. Then plug in $\alpha.$\\
Next, suppose  $\delta \in (0,1/2).$ Then $\mathbb{E}\|\zl\|_2 \leq \lpe \left(\Dl + 2 \sqrt{\frac{2 \alpha}{\lambda}}\right).$ Now \[
\lambda = \frac{L}{R \sqrt{1 + \frac{\varepsilon n}{d^{1/2} (\cd + \sqrt{\cd^2 + \varepsilon})}}} < \frac{L}{R \sqrt{\frac{\varepsilon n}{d^{1/2} (\cd + \sqrt{\cd^2 + \varepsilon})}}} \leq \frac{L}{R}\sqrt{\lpen}\]
and $~\Dl \leq \frac{4L}{\lambda n} \leq 4 \sqrt{2}\frac{R}{n} (\frac{\varepsilon n}{\sqrt{d} (\cd + \sqrt{\cd^2 + \varepsilon})^2})^{1/2}.$ Using these estimates and the estimates for each component of excess risk, as above, gives 
the first statement of part 2b). Using the assumptions and estimates above, it is easy to verify that the prescribed choices of $\alpha$ yield the second statements in each part of the theorem.   
\subsubsection{Proof of \cref{thm: black box convex smooth}}
We decompose excess risk into 3 terms: \begin{multline*}
\ER = \underbrace{\mathbb{E}_{A} [F(w_{T}(X) + \zl, X) - \Fl(w_{T}(X) + \zl, X)]}_{\textcircled{\small{a}}} + \\
\underbrace{\mathbb{E}_{A} \Fl(w_{T}(X) + \zl, X) - \Fl(\wls(X), X)}_{\textcircled{\small{b}}}
+ \underbrace{\Fl(\wls(X), X) - F(\ws(X), X)}_{\textcircled{\small{c}}}. 
\end{multline*}

Now bound the terms as follows: \textcircled{\small{a}} = $-\frac{\lambda}{2} \mathbb{E}\|\wla(X)\|_2^2 \leq 0$. By the descent lemma (note $\Fl(\cdot, X)$ is $(\beta + \lambda)$-smooth) and mean-zero noise $\zl$, we have
\begin{align*}
\textcircled{\small{b}} &=  \mathbb{E}[\Fl(w_{T} + \zl, X) - \Fl(w_{T}, X)] + \mathbb{E}[\Fl(w_{T}, X) - \Fl(\wls, X)]   \\
&\leq \mathbb{E}[\langle \nabla \Fl(w_{T}(X), X)), \zl \rangle + \frac{\beta + \lambda}{2} \|\zl\|_2^2] + \alpha \\
&= \frac{\beta + \lambda}{2} \mathbb{E} \|\zl\|_2^2 + \alpha. 
\end{align*}
 
Lastly, \textcircled{\small{c}} $ = [\Fl(\wls(X), X) - \Fl(\ws(X),X)] + [\Fl(\ws(X), X) - F(\ws(X), X)] < 0 + \frac{\lambda}{2} \|\ws(X)\|_2^2 \leq \frac{\lambda R^2}{2}$ since $\wls(X)$ is the unique minimizer of $\Fl(\cdot, X)$ by strong convexity. Hence $\ER \leq \frac{\lambda R^2}{2} + (\beta + \lambda) \mathbb{E} \|\zl\|_2^2 + \alpha.$ \\
Assume that $F \in \JJE$: the proof of the non-ERM case is very similar, but simpler since $n$ disappears.
First suppose $\delta = 0$, so $\mathbb{E}\|\zl\|_2^2 \leq \frac{2 d^2 (2 \sqrt{\frac{2 \alpha}{\lambda}} + \Dl)^2}{\varepsilon^2}.$ Note that  for our choice of $\lambda$ given in 
\cref{thm: black box convex smooth}, $~\lambda \leq \frac{L}{R} \leq 2\beta$ and $(L + \lambda R) \leq 2L$ by the assumption $~\left(\pen\right)^2 \leq \lrb$ and since $~L \leq 2 \beta R$ always holds for $~F \in \JJ.$ Also, $~\Dl \leq \frac{2(L+ \lambda R)}{\lambda n} \leq \frac{4L}{\lambda n} \leq 4(\frac{L R^2}{\beta})^{1/3} (\nep)^{2/3}$. Therefore,
\[
    \ER \leq \frac{\lambda R^2}{2} + 3\beta \left(\pe\right)^2 \left[16\frac{L^2}{\lambda ^2 n^2} + \frac{\alpha}{\lambda} + 4\sqrt{\frac{\alpha}{\lambda}} \frac{L}{\lambda n}\right] + \alpha.
\]
Plugging in $\lambda = \left(\frac{\beta L^2}{R^2}\right)^{1/3} \left(\pen\right)^{2/3}$, using the estimates above and the fact that $\pen \leq 1$ by assumption gives \[
\lambda R^2, ~\beta \left(\pe\right)^2 \frac{L^2}{\lambda ^2 n^2} \leq \blrt \left(\pen\right)^{2/3}.
\]
Also, by the choice of $\lambda,$ \[
\beta \left(\pe\right)^2\left[\frac{\alpha}{\lambda} + 4\sqrt{\frac{\alpha}{\lambda}} \frac{L}{\lambda n}\right] \leq \alpha\left[\left(\frac{R \beta}{L}\right)^{2/3} \frac{d^{4/3} n^{2/3}}{\varepsilon^{4/3}}\right] + 4\sqrt{\alpha}R \beta^{1/2} \frac{d}{\varepsilon}.
\]
Putting these pieces together proves the first statement in part 2a). Verifying the second statement is routine, using the assumptions stated in the theorem and the estimates obtained above. \\

Now suppose $\delta \in (0, \frac{1}{2})$. Then $\mathbb{E}\|\zl\|_2^2 \leq \frac{d\left(\ld\right)\left(2\sqrt{\frac{2 \alpha}{\lambda}} + \Dl\right)^2}{\varepsilon^2}.$ Again, the assumption $\left(\lpen\right)^2 \leq \lrb$ implies $\lambda \leq 2\beta$ and $L + \lambda R \leq 2L$. Thus, \begin{multline*}
    \ERD \leq \frac{\lambda R^2}{2} + 3\beta \frac{d\left(\ld\right)}{\varepsilon^2}\left[\Dl^2 + 4\sqrt{2}\Dl \sqrt{\frac{\alpha}{\lambda}} + 8\frac{\alpha}{\lambda}\right] + \alpha \\
    \leq 48.5 \blrt \left(\lpen\right)^{2/3} + 3\beta \frac{d \left(\ld\right)}{\varepsilon^2} \left[\frac{16 \sqrt{2} L}{\lambda n} \sqrt{\frac{\alpha}{\lambda}} + \frac{8\alpha}{\lambda}\right] + \alpha,
\end{multline*} 
where in the last line we plugged in estimates for $\lambda$ and $\Dl$ to bound the first two terms in the sum. Then plugging in $\lambda = \left(\frac{\beta L^2}{R^2}\right)^{1/3} \left(\lpen\right)^{2/3}$ and using the estimates and assumptions completes the proof of the first statement of the theorem. The second statement is easy to verify by plugging in the prescribed $\alpha.$

\section{Proofs for \cref{Section 4: pop loss} and Additional Results for Non-Strongly-Convex}
\label{app: Sec 4 proofs}
\subsection{Proofs of Results in \cref{subsection 4.1}}
\subsubsection{Proof of \cref{lem: stability of sc lip f}}
Let $x \in \XX, |X \Delta X'| \leq 2.$ Then \begin{align*}
\mathbb{E}_{\mathcal{A'}}[f(\mathcal{A'}(X), x) - f(\mathcal{A'}(X'), x)] &= \mathbb{E}_{\mathcal{A'}}[f(\Pi_{\WW}(\widehat{w}(X) + z), x) - f(\Pi_{\WW}(\widehat{w}(X') + z), x)] \\
&\leq L \|\widehat{w}(X) - \widehat{w}(X')\| \\
&\leq L \Delta_{\hf} \leq \frac{2L^2}{\mu n},
\end{align*}
using contractivity of the projection map and \cref{prop:sc sensitivity}. Clearly, the first inequality above also holds without projection, so that the same uniform stability bound holds for $\mathcal{A}$ as well. 
\subsubsection{Proof of \cref{prop: sc lip pop loss}}
Write
\begin{multline*}
\mathbb{E}_{X \sim \DD^{n}, \mathcal{A'}}[F(\wprpro(X), \DD) - F(\ws(\DD), \DD)] = \underbrace{\mathbb{E}_{X \sim \DD^{n}, \mathcal{A'}}[F(\wprpro(X), \DD) - \hf(\wprpro(X), X)]}  _{\textcircled{\small{a}}} + \\
\underbrace{\mathbb{E}_{X \sim \DD^{n}, \mathcal{A'}}[\hf(\wprpro(X), X) - \hf(\hw(X), X)]}_{ \textcircled{\small{b}}} + \underbrace{\mathbb{E}_{X \sim \DD^{n}, \mathcal{A'}}\hf(\hw(X), X) - F(\ws(\DD), \DD)}_{\textcircled{\small{c}}}. \\
\end{multline*}
Now, \textcircled{\small{a}} $\leq \alpha = \frac{2L^2}{\mu n}$ by \cref{lem: stability} and \cref{lem: stability of sc lip f}. 
If $~\delta = 0,$ then ~\textcircled{\small{b}} $\leq \frac{2L^2}{\mu}\left(\pen\right);$ \\
and if $\delta \in (0, \frac{1}{2}),$ then ~\textcircled{\small{b}} $\leq \frac{\sqrt{2}L^2 \sqrt{d}\left(\ld\right)}{\mu \varepsilon n}$ by \cref{cor: sc upper bound}. 
Finally, since 
\[\expec[\min_{w \in \WW} \hf(w, X)] \leq \min_{w \in \WW} \expec[\hf(w, X)] = F(\ws(\DD), \DD),\]
we have that  ~\textcircled{\small{c}} $\leq 0.$ Summing these inequalities up completes the proof. 

\subsubsection{Proof of \cref{rem: sc lip pop loss runtime}}
Decompose \begin{multline*}
\ERP = \underbrace{\mathbb{E}_{X \sim \DD^{n}, \mathcal{A}}[F(\wpr(X), \DD) - \hf(\wpr(X), X)]}  _{\textcircled{\small{a}}} + \\
\underbrace{\mathbb{E}_{X \sim \DD^{n}, \mathcal{A}}[\hf(\wpr(X), X) - \hf(\hw(X), X)]}_{ \textcircled{\small{b}}} + \underbrace{\mathbb{E}_{X \sim \DD^{n}, \mathcal{A}}\hf(\hw(X), X) - F(\ws(\DD), \DD)}_{\textcircled{\small{c}}}, \\
\end{multline*}
as before. 
Observe that $\Delta_{T} := \sup_{|X \Delta X'| \leq 2}\|w_T(X) - w_T(X')\|_2 \leq \Delta_{\widehat{F}} + 2 \sqrt{\frac{2 \alpha}{\mu}} \leq 2(\frac{L}{\mu n} + \sqrt{\frac{2 \alpha}{\mu}}),$ and $\mathcal{A}$ is $L \Delta_{T}$-uniformly stable with respect to $f$. Hence $\textcircled{\small{a}} \leq 2(\frac{L^2}{\mu n} + L \sqrt{\frac{2 \alpha}{\mu}}),$ by \cref{lem: stability}. Plugging in the choices of $\alpha$ given in the Corollary shows that $\textcircled{\small{a}} \leq 5 \lmu.$ 
Next, by \cref{cor: sc lip SGD}, \[ \textcircled{\small{b}} = \mathbb{E}_{\mathcal{A}}[\hf(\wpr(X), X) - \hf(\hw(X), X)] \leq 9\lmu \pen\] if $\delta = 0$ and 
$\mathbb{E}_{\mathcal{A}}[\hf(\wpr(X), X) - \hf(\hw(X), X)] \leq 6\lmu \lpen$ if $\delta \in (0, \frac{1}{2})$ with the given choices of $\alpha$ and $T.$ Finally, $\textcircled{\small{c}} \leq 0,$ as in the proof of \cref{prop: sc lip pop loss}. Putting these estimates together completes the proof. 

\subsection{Population Loss Bounds for Convex, Lipschitz functions}
\label{subsection 4.3}
For convex functions, we use a regularization scheme, as described in \cref{subsection 2.4 convex lip} and \cref{subsection 3.3}, and output \[
\wlapro(X) = \Pi_{\WW}(\hw_{\lambda}(X) + \zl),\] where $\hw_{\lambda}(X) = \argmin_{w \in \WW} \widehat{\Fl}(w,X)$ (c.f. \cref{regularized output pert - conceptual alg}). 
For this mechanism, we have the following expected excess population loss bounds:
\begin{proposition}
\label[proposition]{prop: convex lip pop loss}
Let $f(\cdot, x)$ be $L$-Lipschitz on $\WW$ and convex for all $x \in \XX,$ let $\varepsilon > 0$ and $\delta \in [0, \frac{1}{2}).$ Let $X$ be a data set of size $n$ drawn i.i.d. according to $\mathcal{D}.$ Then with the choices of $\lambda > 0$ below, running the conceptual regularized output perturbation algorithm (\cref{regularized output pert - conceptual alg}) on $\hf(w,X)$ achieves the following expected excess population loss bounds: \\
a) Assume $\delta = 0$ and $\pen \leq 1$. Set $\lambda = \left(\frac{L}{R}\right)\left(\left(\pen\right)^{1/2} + \frac{1}{\sqrt{n}}\right).$ \\Then, $\ERPLpro \leq LR\left(127\left(\pen\right)^{1/2} + \frac{1}{2\sqrt{n}}\right).$ \\
b) Assume $\delta \in \left(0, \frac{1}{2}\right)$ and $\lpen \leq 1.$ Set $\lambda = \left(\frac{L}{R}\right)\left(\left(\lpen\right)^{1/2} + \frac{1}{\sqrt{n}}\right).$\\
Then, $\ERDPLpro \leq LR\left(19\left(\lpen\right)^{1/2} + \frac{1}{2\sqrt{n}}\right).$
\end{proposition}

When SGD (\cref{alg: stochastic subgrad}) is used as the non-private input to the Black Box \cref{alg: black box convex non-smooth}, we achieve the above loss bounds with the following runtimes: 
\begin{corollary}
\label[corollary]{rem: convex lip pop loss runtime}
Let $f(\cdot, x)$ be $L$-Lipschitz on $\WW$ and convex for all $x \in \XX,$ let $\varepsilon > 0$ and $\delta \in [0, \frac{1}{2}).$ Let $X$ be a data set of size $n$ drawn i.i.d. according to $\mathcal{D}.$ Run $\mathcal{A}=$~\cref{alg: black box convex non-smooth} on $\widehat{\Fl}(w,X)$ with $\mathcal{M}$ as the stochastic subgradient method (\cref{alg: stochastic subgrad}) with step sizes $\eta_{t} = \frac{2}{\lambda (t + 1)}$ and $\lambda, ~T, \alpha$ as prescribed below. \\
a) Suppose $\delta = 0, ~\pen \leq 1.$ Set $\lambda = \frac{L}{R}\left(\left(\pen\right)^{1/2} + \frac{1}{\sqrt{n}}\right),$ $~\alpha = \frac{LR \left(\pen\right)^{3/2}}{(1 + \pe)^2},$ and $T = 2n^2 \max\left\{1, (\frac{\varepsilon}{d})^2\right\}.$ Then \[
\ERPLpro \leq 32 LR\left(\left(\pen\right)^{1/2} + \frac{1}{\sqrt{n}}\right)\]
in runtime $O\left(dn^2 \max\left\{1, (\frac{\varepsilon}{d})^2\right\}\right).$\\
b) Suppose $\delta \in (0, \frac{1}{2}).$ Set $\lambda = \left(\frac{L}{R}\right)\left(\left(\lpen\right)^{1/2} + \frac{1}{\sqrt{n}}\right), ~\alpha = LR\left(\lpen\right)^{3/2},~$ and $~T = 2\left(\frac{n^2 \varepsilon}{d}\right).$ 
Then \[
\ERDPLpro \leq 32 LR\left(\left(\lpen\right)^{1/2} + \frac{1}{\sqrt{n}}\right)\]
in runtime $O(n^2 \varepsilon).$
\end{corollary}

For $\delta = 0,$ the only competing expected excess population loss bound for this class that we are aware of is $\widetilde{O}\left(LR \left(\sqrt{\pen} + \frac{1}{\sqrt{n}}\right)\right),$ obtained by the exponential mechanism in \cite{bst14}. This bound is the same as ours. However, its worst-case runtime $\widetilde{O}(R^2 d^9 n^3 \max \{d, \varepsilon n R\}),$ is generally much larger than ours. 
For $\delta > 0,$ on the other hand, there are several efficient algorithms that nearly attain the optimal (by \cite{bft19}) private expected excess population loss bound of $O\left(LR \left(\frac{1}{\sqrt{n}} + \frac{\sqrt{d}}{\varepsilon n}\right)\right)$ \cite{bft19, fkt20, arora20}. The fastest of these is the private FTRL method of \cite{arora20}, which has runtime $O(nd)$ (and is suboptimal in excess loss by a $\log(n)$ factor). However, the method does not guarantee $(\varepsilon, \delta)$-differential privacy for arbitrary $\varepsilon$. In fact, $\varepsilon$ must be bounded by $\sqrt{\log(\frac{1}{\delta})}$ in order for their results to be valid \footnote{Their results are stated in terms of $(\alpha, \gamma)$-Renyi-privacy for $\alpha > 1$, but they argue on p.9, using \cite[Lemma 2.6]{fkt20}, that their Algorithm 1 is  $(\varepsilon, \delta)$-differentially private for $\varepsilon = \frac{L \sqrt{\log(\frac{1}{\delta})}}{\sigma \sqrt{n}},$ where they put $\sigma = \frac{2 \alpha L}{\sqrt{n}\sqrt{\gamma}}.$ Some algebra reveals that $\varepsilon = \sqrt{\log(\frac{1}{\delta})}$ is necessary for their risk and privacy guarantees to both hold.}.

\paragraph{Smooth, convex, Lipschitz functions.}
If we assume additionally that $f$ is $\beta$-smooth, then the conceptual regularized output perturbation algorithm results in the following expected excess population loss guarantees:  
\begin{proposition}
\label[proposition]{prop: convex smooth pop loss}
Let $f(\cdot, x)$ be Let $f(\cdot, x)$ be $L$-Lipschitz on $\WW$, convex, and $\beta$-smooth for all $x \in \XX,$ let $\varepsilon > 0$ and $\delta \in [0, \frac{1}{2}).$ Let $X$ be a data set of size $n$ drawn i.i.d. according to $\mathcal{D}.$ Then with the choices of $\lambda > 0$ below, running the conceptual regularized output perturbation algorithm (\cref{regularized output pert - no proj}) on $\hf(w,X)$ achieves the following expected excess population loss bounds: \\
a) If $\delta = 0$ and $\left(\pen\right)^2 \leq \lrb,$ then setting $\lambda = \left(\frac{\beta L^2}{R^2}\right)^{1/3}\left(\left(\pen\right)^{2/3} + \frac{1}{\sqrt{n}}\right)$ implies \[
\ERPL \leq \blrt \left(\frac{11}{\sqrt{n}} + 83\left(\pen\right)^{2/3}\right).\]
b) If $\delta \in \left(0, \frac{1}{2}\right)$ and $\left(\lpen\right)^2 \leq \lrb,$ then 
setting $\lambda = \left(\frac{\beta L^2}{R^2}\right)^{1/3}\left(\left(\lpen\right)^{2/3} + \frac{1}{\sqrt{n}}\right)$ implies 
\[
\ERDPL \leq \blrt \left(\frac{11}{\sqrt{n}} + 43\left(\lpen\right)^{2/3}\right).\]
\end{proposition}

Next, we show that we can obtain these population loss bounds with a practical, (nearly) linear-time algorithm. 

\begin{corollary}
\label[corollary]{rem: convex smooth pop loss runtime}
Let $f(\cdot, x)$ be Let $f(\cdot, x)$ be $L$-Lipschitz on $\WW$, convex, and $\beta$-smooth for all $x \in \XX$. Let $\varepsilon > 0$ and $\delta \in [0, \frac{1}{2})$ and let $X$ be a data set of size $n$ drawn i.i.d. according to $\mathcal{D}.$  Run $\mathcal{A}=$~\cref{alg: black box convex non-smooth} on $\widehat{\Fl}(w,X)$ with $\mathcal{M}$ as \cref{alg: Katyusha} (Katyusha) with $\lambda, ~\alpha, ~T$ as prescribed below. \\
1. Suppose $\delta = 0$ and $\left(\pen\right)^2 \leq \frac{L R}{\beta}.$ Set $\lambda = \left(\frac{\beta L^2}{R^2}\right)^{1/3}\left(\left(\pen\right)^{2/3} + \frac{1}{\sqrt{n}}\right),$ $~\alpha = \frac{L^{4/3} R^{2/3}}{\beta^{1/3}},$ and $~T = O\left(n + n^{5/6} \left(\ep\right)^{1/3} \left(\frac{\beta R}{L}\right)^{1/3} \log\left(\left(\frac{\beta R}{L}\right)^{1/3} \left(\pe\right)^{2/3} n^{4/3}\right)\right).$ \\ Then, \[
\ERPL \leq 128 \blrt \left(\frac{1}{\sqrt{n}} + \left(\pen\right)^{2/3}\right)\]
in runtime $O\left(nd + n^{5/6}d^{2/3} \varepsilon^{1/3} \left(\frac{\beta R}{L}\right)^{1/3}\right).$\\ 
2. Suppose $\delta \in (0, \frac{1}{2})$ and $\left(\lpen\right)^2 \leq \frac{LR}{\beta}.$ \\
Set $\lambda = \left(\frac{\beta L^2}{R^2}\right)^{1/3}\left(\left(\lpen\right)^{2/3} + \frac{1}{\sqrt{n}}\right), \alpha = \frac{L^{4/3} R^{2/3}}{\beta^{1/3}}\left(\frac{\varepsilon n}{\sqrt{d}\left(\ld\right)}\right)^{2/3}\frac{1}{n^2},$ and \\
$~T = O\left(n + n^{5/6} \left(\lep\right)^{1/3} \left(\frac{\beta R}{L}\right)^{1/3} \log\left(\left(\frac{\beta R}{L}\right)^{1/3} n^2 \left(\lpen\right)^{2/3}\right) \right).$ \\
Then, \[
\ERDPL \leq 128 \blrt \left(\frac{1}{\sqrt{n}} + \left(\lpen\right)^{2/3}\right)\]
in runtime $O\left(nd + n^{5/6}d^{5/6} \varepsilon^{1/6} \left(\frac{\beta R}{L}\right)^{1/3}\right).$
\end{corollary}
\begin{remark}
\label[remark]{rem: min smooth convex pop}
Since the smooth, convex, Lipschitz loss function class is a subset of the convex, Lipschitz loss function class, we obtain by combining \cref{prop: convex lip pop loss} with \cref{rem: convex smooth pop loss runtime} the following upper bounds on the expected excess population loss of smooth, strongly convex, Lipschitz functions: 
$\ERPL \lesssim 
\min \left\{LR,  LR\left(\left(\pen\right)^{1/2} + \frac{1}{\sqrt{n}}\right), \blrt\left(\left(\pen\right)^{2/3} + \frac{1}{\sqrt{n}}\right)\right\}$ for $\delta = 0$;  and $\ERDPL \lesssim 
\min \left\{LR,  LR\left(\left(\lpen\right)^{1/2} + \frac{1}{\sqrt{n}}\right), \blrt\left(\left(\lpen\right)^{2/3} + \frac{1}{\sqrt{n}}\right)\right\}$ for $\delta > 0.$ The second terms in each of the above minima are the non-smooth expected excess population losses from \cref{prop: convex lip pop loss}. For smooth convex loss functions, by utilizing Katyusha, these expected excess population losses can be realized faster than the runtime promised in \cref{rem: convex lip pop loss runtime}. To derive these enhanced upper bounds on the runtimes for obtaining the non-smooth convex expected excess population loss bounds, recall the procedure described in \cref{rem: smooth convex ERM Katyusha min}: use the regularization term $\lambda$ (and the corresponding $\alpha$) prescribed in \cref{thm: black box imp convex lip}. Plugging this into the runtime expression for Katyusha gives the following runtime for achieving excess empirical risk $O(LR \sqrt{\pen})$ for $\delta = 0$: $\widetilde{O}(d(n + \sqrt{n\left(\frac{\beta + \lambda}{\lambda}\right)})) = \widetilde{O}\left(nd + n^{3/4} d^{3/4} \varepsilon^{1/4} \sqrt{\frac{\beta R}{L}}\right).$ By the uniform stability argument used in this section, this excess empirical loss bound implies expected excess population loss of $O\left(LR \left(\frac{1}{\sqrt{n}} + \sqrt{\pen}\right)\right)$. This explains where the ``$\max \{...\}$'' term in the last row (corresponding to $\JJ$) of \cref{table:pop loss delta = 0} comes from. Likewise, for $\delta > 0,$ the runtime for obtaining excess empirical risk $O\left(LR \left(\lpen\right)^{1/2}\right)$, which yields the expected excess population loss $O\left(LR\left(\left(\pen\right)^{1/2} + \frac{1}{\sqrt{n}}\right)\right)$, by Katyusha is $\widetilde{O}\left(nd + n^{3/4} d^{7/8} \varepsilon^{1/8} \sqrt{\frac{\beta R}{L}}\right).$ 
\end{remark}

Recall that the optimal non-private population loss for convex Lipschitz is $O(\frac{LR}{\sqrt{n}}).$ 
Hence, if we assume $\beta R = L$, so that $\blrt = LR$ for simplicity, then when $\delta = 0,$ we get privacy for free if $\left(\pe\right)^2 \lesssim \sqrt{n}.$ For $\delta > 0,$ under the same simplifying assumptions on $\beta, L, R$, we get privacy for free if $\frac{d \cd^2}{\varepsilon^2} \lesssim \sqrt{n}.$ For general $\beta, L,$ and $R,$ the ``privacy for free'' parameter ranges may be narrower than these by a factor of $\left(\frac{L}{\beta R}\right)^{1/3}$.

\vspace{0.2cm}
For $\delta = 0$, there are no other private expected population loss bounds we are aware of besides the one of \cite{bst14} mentioned above, which, to the best of our knowledge, does not improve under additional smoothness assumption and is therefore inferior to our bounds in both excess loss and runtime. The gradient descent-based output perturbation method of \cite{zhang2017} achieves with probability $~1 - \gamma$, excess population loss of $\widetilde{O}\left(\frac{R L^{1/2} (\beta d)^{1/4}}{(n \varepsilon \gamma)^{1/2}}\right),$ but they do not provide a guarantee for \textit{expected} excess population loss. For $\delta > 0,$ \cite{bft19}, \cite{fkt20}, and \cite{arora20} all nearly achieve the optimal (by \cite{bft19}) expected excess population loss bound $O\left(LR\left(\frac{\sqrt{d}}{\varepsilon n} + \frac{1}{\sqrt{n}}\right)\right)$ for this function class efficiently. In fact, \cite{fkt20} achieves the bound up to a $\log(\frac{1}{\delta})$ factor in runtime $O(nd)$, provided $\frac{\beta R}{L} \lesssim \min\left\{\frac{1}{\sqrt{n}}, \frac{\sqrt{\log\left(\frac{1}{\delta}\right) n}}{\sqrt{d}}\right\}$ and $\varepsilon \leq 2 \log\left(\frac{1}{\delta}\right).$ The first condition is not very restrictive; indeed, essentially every method (including ours) requires some such restriction to guarantee non-trivial excess risk for $\JJ.$ The condition on $\varepsilon,$ however does restrict the range of $n, d$ in which strong risk bounds can be realized privately. Our method has no such restriction. 

\vspace{0.2cm}
The last noteworthy result from recent DP SCO literature is that conceptual objective perturbation can attain obtain the optimal $(\varepsilon, \delta)$ private population excess loss in the convex smooth case under strict assumptions such as the loss function $f$ having rank-one hessian and $\beta \leq \frac{2 \varepsilon L}{R} \sqrt{2n + 4d \log\left(\frac{1}{\delta}\right)}$ \cite{bft19}. By contrast, here we have shown that output perturbation, which does not require any of the above assumptions and can be implemented quite quickly and easily, nearly attains the optimal rate and also applies (and performs quite well) for $\delta = 0$.

\subsubsection{Proof of \cref{prop: sc smooth pop loss}}
The proof is nearly identical to the proof of \cref{prop: sc lip pop loss}, except that we use the excess empirical risk bounds for $\HHE$ (\cref{cor: smooth sc fund ER bounds}) to bound the \textcircled{\small{b}} term instead of \cref{cor: sc upper bound}.

\subsubsection{Proof of \cref{rem: smooth sc pop loss runtime katyusha}}
The proof follows exactly as the proof of \cref{rem: sc lip pop loss runtime}, but using \cref{cor: katyusha sc smooth} instead of \cref{cor: sc lip SGD} to bound \textcircled{\small{b}}.

\subsection{Proofs of Results in \cref{subsection 4.3}}
\subsubsection{Proof of \cref{prop: convex lip pop loss}}
As in the proof of \cref{prop: sc lip pop loss}, write  
\begin{multline*}
\ERPLpro = \underbrace{\mathbb{E}_{X \sim \DD^{n}, \mathcal{A'}_{\lambda}}[F(\wlapro(X), \DD) - \hf(\wlapro(X), X)]}  _{\textcircled{\small{a}}} + \\
\underbrace{\mathbb{E}_{X \sim \DD^{n}, \mathcal{A'}_{\lambda}}[\hf(\wlapro(X), X) - \hf(\hw(X), X)]}_{ \textcircled{\small{b}}} + \underbrace{\mathbb{E}_{X \sim \DD^{n}, \mathcal{A'}_{\lambda}}\hf(\hw(X), X) - F(\ws(\DD), \DD)}_{\textcircled{\small{c}}}. \\
\end{multline*}

Now decompose \textcircled{\small{a}} into three parts: \begin{align*}
\textcircled{\small{a}} &= \underbrace{\mathbb{E}[F(\wlapro(X), \mathcal{D}) - \Fl(\wlapro(X), X)]}_{\textcircled{\small{i}}} 
+ \underbrace{\mathbb{E}[\Fl(\wlapro(X), \mathcal{D}) - \widehat{\Fl}(\wlapro(X), X)]}_{\textcircled{\small{ii}}} 
\\&+ \underbrace{\mathbb{E}[\widehat{\Fl}(\wlapro(X), X) - \widehat{F}(\wlapro(X), X)]}_{\textcircled{\small{iii}}},
\end{align*}
where $\Fl(w,\mathcal{D}) := F(w, \mathcal{D}) + \frac{\lambda}{2}\|w\|_2^2.$
Then \textcircled{\small{i}} + \textcircled{\small{iii}} = 0, whereas \textcircled{\small{ii}} $\leq \frac{2(L+\lambda R)^2}{\lambda n}$ by \cref{lem: stability} and \cref{lem: stability of sc lip f}. Next, decompose \textcircled{\small{b}} into three parts as in the proof of \cref{prop: convex ER}: 
\begin{align*}
\textcircled{\small{b}} &= \underbrace{\mathbb{E}[\hf(\wlapro(X), X) - \widehat{\Fl}(\wlapro(X), X)]}_{\textcircled{\small{1}}} + \underbrace{\mathbb{E}[\widehat{\Fl}(\wlapro(X), X) - \widehat{\Fl}(\widehat{w_{\lambda}}(X), X)]}_{\textcircled{{\small{2}}}} \\ 
&+ \underbrace{\mathbb{E}[\widehat{\Fl}(\widehat{w_{\lambda}})(X), X) - \hf(\hw(X), X)]}_{\textcircled{{\small{3}}}}.
\end{align*}
Then $\textcircled{{\small{1}}} \leq 0.$ Also, \[
\textcircled{{\small{2}}} \leq (L + \lambda R) \mathbb{E}\|\zl\|_2.
\]
By the definition of $\zl$ (\cref{subsection 2.4 convex lip}), \cref{lem: expec of gamma and gauss}, and \cref{prop:sc sensitivity}, we have $\mathbb{E}\|\zl\|_2 \leq \frac{2 \sqrt{2}(L + \lambda R)}{\lambda}\left(\pen\right)$ if $\delta = 0;$ if $\delta \in (0, \frac{1}{2}),$ then $\mathbb{E}\|\zl\|_2 \leq \frac{2 (L + \lambda R)}{\lambda} \left(\lpen\right).$
The third term $\textcircled{\small{3}} \leq \frac{\lambda R^2}{2}.$ Lastly, $\textcircled{\small{c}} \leq 0$ by the same argument used in \cref{prop: sc lip pop loss}. Hence if $~\delta = 0,$ then \[
\ERPLpro \leq 2 \sqrt{2}\left[\frac{(L + \lambda R)^2}{\lambda}\left(\pen\right)\right] + \frac{\lambda R^2}{2},\]
and if $~\delta \in (0, \frac{1}{2}),$ then \[
\ERPLpro \leq 2\left[\frac{(L + \lambda R)^2}{\lambda}\left(\lpen\right)\right] + \frac{\lambda R^2}{2}.\]
For $\delta = 0,$ plugging in $\lambda = \left(\frac{L}{R}\right)\left(\left(\pen\right)^{1/2} + \frac{1}{\sqrt{n}}\right)$ gives the first result. \\
For $\delta \in \left(0, \frac{1}{2}\right),$ instead choose $\lambda = \left(\frac{L}{R}\right)\left(\left(\lpen\right)^{1/2} + \frac{1}{\sqrt{n}}\right)$ to get the last result. 

\subsubsection{Proof of \cref{rem: convex lip pop loss runtime}}
Decompose
\begin{multline*}
\ERPLpro = \underbrace{\mathbb{E}_{X \sim \DD^{n}, \mathcal{A'}_{\lambda}}[F(\wpr(X), \DD) - \hf(\wpr(X), X)]}  _{\textcircled{\small{a}}} + \\
\underbrace{\mathbb{E}_{X \sim \DD^{n}, \mathcal{A'}_{\lambda}}[\hf(\wpr(X), X) - \hf(\hw(X), X)]}_{ \textcircled{\small{b}}} + \underbrace{\mathbb{E}_{X \sim \DD^{n}, \mathcal{A'}_{\lambda}}\hf(\hw(X), X) - F(\ws(\DD), \DD)}_{\textcircled{\small{c}}}, 
\end{multline*}
as before. First, $\mathcal{A}$ is $L \Delta_{T}$-uniformly stable with respect to $f$ (see proof of \cref{rem: sc lip pop loss runtime}), and $\Delta_{T} = \sup_{|X \Delta X'| \leq 2} \|w_T(X) - w_T(X')\|_2 \leq \Delta_{\lambda} + 2 \sqrt{\frac{2 \alpha}{\lambda}} \leq 2\left[\frac{L + \lambda R}{\lambda n} + \sqrt{\frac{2 \alpha}{\lambda}}\right],$ using \cref{cor: convex lip SGD} and strong convexity of $\Fl$. Hence $\textcircled{\small{a}} \leq 2L[\frac{L + \lambda R}{\lambda n} + \sqrt{\frac{2 \alpha}{\lambda}}]$ by \cref{lem: stability}. Now \[\textcircled{\small{b}} \leq (L + \lambda R)\left[\sqrt{\frac{2 \alpha}{\lambda}} + \mathbb{E}\|\zl\|_2\right] + \frac{\lambda R^2}{2},
\]
by decomposing and bounding as done in the proof of \cref{thm: black box imp convex lip}. Also, \textcircled{\small{c}} $\leq 0$ as usual. Then combining the above with our choices of $\lambda, \alpha, T$ and using \cref{lem: expec of gamma and gauss} completes the proof. 
\subsubsection{Proof of of \cref{prop: convex smooth pop loss}}
To begin, we follow along the same lines as in the proof of \cref{prop: convex lip pop loss}: 
\begin{multline*}
\ERPL = \underbrace{\mathbb{E}_{X \sim \DD^{n}, \mathcal{A}_{\lambda}}[F(\wpr(X), \DD) - \hf(\wpr(X), X)]}  _{\textcircled{\small{a}}} + \\
\underbrace{\mathbb{E}_{X \sim \DD^{n}, \mathcal{A}_{\lambda}}[\hf(\wpr(X), X) - \hf(\hw(X), X)]}_{ \textcircled{\small{b}}} + \underbrace{\mathbb{E}_{X \sim \DD^{n}, \mathcal{A}_{\lambda}}\hf(\hw(X), X) - F(\ws(\DD), \DD)}_{\textcircled{\small{c}}}. \\
\end{multline*}

Then $\textcircled{\small{a}} \leq \frac{2(L+\lambda R)^2}{\lambda n}$ as in the proof of \cref{prop: convex lip pop loss}. Next, decompose \textcircled{\small{b}} into three parts as in the proof of \cref{prop: convex ER}: 
\begin{align*}
\textcircled{\small{b}} &= \underbrace{\mathbb{E}[\hf(\wpr(X), X) - \widehat{\Fl}(\wpr(X), X)]}_{\textcircled{\small{1}}} + \underbrace{\mathbb{E}[\widehat{\Fl}(\wpr(X), X) - \widehat{\Fl}(\widehat{w_{\lambda}}(X), X)]}_{\textcircled{\small{2}}} 
\\&= \underbrace{\mathbb{E}[\widehat{\Fl}(\widehat{w_{\lambda}})(X), X) - \hf(\hw(X), X)]}_{\textcircled{\small{3}}}.
\end{align*}
Then $\textcircled{1} \leq 0$ and $\textcircled{\small{3}} \leq \frac{\lambda R^2}{2},$ as in the proof of \cref{prop: convex lip pop loss}, but to bound \textcircled{\small{2}}, we now use the descent lemma to obtain \[
\textcircled{\small{2}} \leq \frac{(\beta + \lambda)}{2} \mathbb{E}\|\zl\|_2^2.
\]
By the definition of $\zl$ (\cref{subsection 2.4 convex lip}), \cref{lem: expec of gamma and gauss}, and \cref{prop:sc sensitivity}, we have $\mathbb{E}\|\zl\|_2^2 \leq 8\frac{(L + \lambda R)^2}{\lambda^2}\left(\pen\right)^2$ if $\delta = 0;$ if $\delta \in (0, \frac{1}{2}),$ then $\mathbb{E}\|\zl\|_2^2 \leq \frac{4 (L + \lambda R)^2}{\lambda^2} \left(\lpen\right)^2.$
The third term $\textcircled{\small{3}} \leq \frac{\lambda R^2}{2}.$ Lastly, $\textcircled{\small{c}} \leq 0$ by the same argument used in \cref{prop: sc lip pop loss}. Hence if $~\delta = 0,$ then \[
\ERPL \leq \frac{2(L + \lambda R)^2}{\lambda n} + 4 \frac{(\beta + \lambda)(L + \lambda R)^2}{\lambda^2}\left(\pen\right)^2 + \frac{\lambda R^2}{2}.\] 
If $\delta \in (0, \frac{1}{2}),$ then \begin{align*}
\ERDPL &\leq \frac{2(L + \lambda R)^2}{\lambda n} 
\\&+ 2 \frac{(\beta + \lambda)(L + \lambda R)^2}{\lambda^2}\left(\lpen\right)^2 
+ \frac{\lambda R^2}{2}. 
\end{align*}
Set $\lambda = \left(\frac{\beta L^2}{R^2}\right)^{1/3}\left(\left(\pen\right)^{2/3} + \frac{1}{\sqrt{n}}\right)$ and use the assumptions stated in the Proposition to get $\lambda \leq 5 \beta.$ Then bound $\frac{2(L + \lambda R)^2}{\lambda n} \leq 2 \blrt\left(\left(\pen\right)^{2/3} + \frac{5}{\sqrt{n}}\right);$ $4 \frac{(\beta + \lambda)(L + \lambda R)^2}{\lambda^2}\left(\pen\right)^2 \leq 80 \blrt\left( \pen\right)^{2/3};$ and $\frac{\lambda R^2}{2} \leq \frac{1}{2}\blrt\left(\left(\pen\right)^{2/3} + \frac{1}{\sqrt{n}}\right)$ to botain the first part of the Proposition. 
\\
For $\delta \in \left(0, \frac{1}{2}\right),$ instead choose $\lambda = \left(\frac{\beta L^2}{R^2}\right)^{1/3}\left(\left(\lpen\right)^{2/3} + \frac{1}{\sqrt{n}}\right)$ and proceed similarly to get the second bound. 

\subsubsection{Proof of \cref{rem: convex smooth pop loss runtime}}
The proof follows almost exactly as the proof of \cref{rem: convex lip pop loss runtime}, but uses the descent lemma to bound \textcircled{\small{b}} with a $\frac{(\beta + \lambda)}{2} \mathbb{E}\|\zl\|^2$ term (instead of $(L + \lambda R)\mathbb{E}\|\zl\|_2$) and uses
\cref{cor: smooth convex katyusha} instead of \cref{cor: convex lip SGD} (along with the alternative choices of $\alpha$ and $\lambda$) to obtain the stated bounds in faster runtime. 

\section{Proofs of Results in \cref{section 5: applications}}
\subsection{Proofs of Results in \cref{sec: TERM}}
\label[section]{sec: TERM proofs}

\subsubsection{Proof of \cref{lem: sensitivity of ftau}}
First note that $\Delta_{\ftau} \leq \frac{2L}{\mu}$ follows immediately from \cref{prop:sc sensitivity}. 
To prove $\Delta_{\ftau} \leq \frac{2L C_{\tau}}{\mu n},$ we follow the same approach used in proving \cref{prop:sc sensitivity}. Let $X, X' \in \XX^n$ such that $|X \Delta X'| \leq 2$ and assume WLOG that $x_n \neq x_n'.$ We apply \cref{lemma7} with $g_{\tau}(w):= \ftau(w,X) - \ftau(w,X')$
and $G_{\tau}(w) := \ftau(w,X').$ Denote $v_{i}(\tau, w) = \frac{e^{\tau f(w,x_i)}}{\sum_{j=1}^n e^{\tau f(w,x_j)}}$ and $v_{i}'(\tau, w) = \frac{e^{\tau f(w,x_i')}}{\sum_{j=1}^n e^{\tau f(w,x_j')}}.$ Then observe that 
\begin{align*}
\nabla g_{\tau}(w) &= \nabla \ftau(w,X) - \nabla \ftau(w,X') \\
&= \sum_{i=1}^n v_i(\tau, w) \nabla f(w,x_i) - v_i'(\tau, w) \nabla f(w, x_i')\\
&= v_n(\tau, w) \nabla f(w, x_n) - v_n'(\tau, w) \nabla f(w, x_n').
\end{align*}
Now convexity and $L$-Lipschitzness of $f$ imply \[
\|\nabla g_{\tau}(w)\| \leq 2L \max\{v_n(\tau, w), v_n'(\tau, w)\} \leq 2L C_{\tau}.
\]
This proof is completed by noticing that $G_{\tau}$ is $\mu$-strongly convex and appealing to \cref{lemma7}.

\subsubsection{Proof of \cref{lem: TERM beta smoothness}}
First, we have $\nabla \ftau(w, X) = \sum_{i=1}^n v_i (w, \tau) \nabla f_i(w),$ where $v_i(w, \tau):= \frac{e^{\tau f(w, x_{i})}}{\sum_{j=1}^n e^{\tau f(w, x_{j})}}$ and we denote $f_i(w):= f(w, x_i),$ by \cref{tilted ERM lemma 1}. Now for any $w_1, w_2 \in \WW$, we have 
\begin{align*}
    \left\|\nabla \ftau(w_1, X) - \nabla \ftau(w_2, X)\right\|_2 &= \left\|\sum_{i=1}^n \left(v_i(w_1, \tau) \nabla f_i(w_1) - v_i(w_2, \tau) f_i(w_2) \right)\right\|_2 \\
    &\leq \sum_{i=1}^n \left(v_i(w_1, \tau) \beta \|w_1 - w_2 \|_2 + L_{v_{i}}\|w_1 - w_2 \|_2 L\right) \\
    &\leq \beta \|w_1 - w_2\|_2 + L\|w_1 - w_2\|_2 \sum_{i=1}^n \left(L_{v_{i}}\right),
\end{align*}
where $L_{v_{i}}$ denotes the Lipschitz constant of $v_i(w, \tau)$ as a function of $w.$ 
Next, we compute $L_{v_{i}}$ by bounding $\| \nabla v_i(w, \tau)\|$: 
\begin{align*}
    \nabla v_i(w, \tau) &= \frac{\tau e^{\tau f_i(w)} \nabla f_i(w) \left(\sum_{j=1}^n e^{\tau f_j(w)}\right) - \tau \sum_{j=1}^n e^{\tau f_j(w)} \nabla f_j(w) e^{\tau f_i(w)}}{\left(\sum_{j=1}^n e^{\tau f_j(w)}\right)^2} \\
    & = \frac{\tau \sum_{j=1}^n e^{\tau f_j(w)} e^{\tau f_i(w)} \left(\nabla f_i(w) - \nabla f_j(w) \right)}{\left(\sum_{j=1}^n e^{\tau f_j(w)}\right)^2}. 
\end{align*}
Taking the norm of both sides and using $L$-Lipschitzness of $f(\cdot, x_i)$ implies 
\[
L_{v_{i}} \leq 2L \tau \frac{e^{\tau f_i(w)}}{\sum_{j=1}^n e^{\tau f_j(w)}},\] and hence \[\sum_{i=1}^n L_{v_{i}} \leq 2 L \tau.
\]
Therefore, \begin{align*}
    \|\nabla \ftau(w_1, X) - \nabla \ftau(w_2, X)\|_2 &\leq \beta \|w_1 - w_2\|_2 + L\|w_1 - w_2\|_2 \sum_{i=1}^n L_{v_{i}} \\
    &\leq \left(\beta + 2L^2 \tau \right)\|w_1 - w_2\|_2,
\end{align*}
which proves \cref{lem: TERM beta smoothness}.

\subsection{Proofs of Results in \cref{subsection: DP adversarial training}}
\subsubsection{Proof of \cref{lem: max is strongly convex and lipschitz}}
First of all, $g$ is $L$-Lipschitz on $\WW$ if and only if $\|q\|_2 \leq L$ for all subgradients $q \in \partial g(w)$ for all $w \in \WW.$ Moreover, for any $w \in \WW,$ \[
\partial g(w) = conv\left(\bigcup \{u \in \partial_{w} h(w, \widehat{v}) \colon \widehat{v} \in \argmax_{v} h(w,v)\} \right),
\]
where \textit{conv} denotes the convex hull of a set. Then since $\|u\| \leq L$ for all $u \in  \partial_{w} h(w, \widehat{v})$ (as $h(\cdot, \widehat{v})$ is $L$-Lipschitz by assumption), we have $\|q\|_2 \leq L$ for all $q \in \partial g(w),$ which establishes Lipschitzness of $g.$ 
Next, consider the convex (in $w$ for all $v \in S$) function $\widetilde{h}(w,v) := h(w,v) - \frac{\mu}{2}\|w\|_2^2.$ Then since the maximum of a (possibly uncountable) family of convex functions is convex, we have that $\max_{v \in S}\widetilde{h}(w,v) = g(w) - \frac{\mu}{2}\|w\|_2^2$ is convex, which means that $g$ is $\mu$-strongly convex.

\section{$L_2$ Sensitivity Lower Bound for $\FF$}
\label{app: sens lower bound}
In this section, we establish that the sensitivity analysis, which is essential to most of our results, is tight. We do so by exhibiting a data universe $\XX := B(0,1)$ and a function $F \in \FFE$ with $\Delta_{F} \geq \frac{L}{\mu n}.$ Our hard instance is $f(w,x) := \frac{\mu}{4} \|w\|_2^2  + \frac{L}{4}x^{T}w,$ on $\mathbb{R}^d \times B(0, 1).$ Then \[
\nabla f(w,x) = \frac{1}{2}\left(\mu w + \frac{L}{2} x\right) \implies \| \nabla f(w,x) \|_2 \leq \frac{1}{2}\left(\mu R + \frac{L}{2}\right) \leq L,
\]
for any $L \geq \mu R$. So, $f(\cdot, x)$ is $L$-Lipschitz for any $L \geq \mu R$ (for all $x \in \XX$). Furthermore $\nabla^2 f(w,x) = \mu \mathbf{I}.$ Define $F(w,X) = \frac{1}{n}\sum_{i=1}^n f(w,x_{i}).$ Observe that for any $X = (x_1, \cdots, x_n) \in B(0,1),$ we have 
\[
\nabla F(w,X) = \frac{\mu}{2} w + \frac{L}{4n}\sum_{i=1}^N x_i \implies \ws(X) = \frac{-L}{2\mu n} \sum_{i=1}^N x_i. 
\]
Hence $\|\ws(X)\|_2 \leq R$ for any $R \geq \frac{L}{\mu}$ (for all $X \in \XX$), so $F \in \FFE.$ Now consider the data sets $X = \{0, ... , 0, e_1 \}$ and $X' = \{0, ... ,0, -e_1\},$ where $e_1 = (1, 0, ..., 0)^{T} \in \XX \subseteq \mathbb{R}^{d}.$ We will show that $\|\ws(X) - \ws(X') \| \geq \frac{L}{\mu n},$ which implies our lower bound $\Delta_{F} \geq \frac{L}{\mu n}.$

Consider \[
\nabla F(w,X) = \frac{\mu}{2}w + \frac{L}{4}e_1 = 0 \implies \ws(X) = -\frac{L}{2 \mu n} e_1 
\]
and likewise, $\ws(X') = \frac{L}{2 \mu n} e_1.$ Thus, $\|\ws(X) - \ws(X') \|_2 = \frac{L}{\mu n}.$

\end{document}